\documentclass[manuscript,screen]{acmart}
\usepackage{multirow}   %导言区
\usepackage{booktabs}    %导言区	
\usepackage{algorithm}
\usepackage{algorithmic}
\usepackage{amsmath}
\usepackage{bm}
\usepackage{graphicx}
\usepackage{soul}
\usepackage{subfigure}
\usepackage{todonotes}
\usepackage{xspace}
\usepackage{adjustbox}
\usepackage{colortbl}
\usepackage{bigstrut}
\usepackage{makecell}

\newcommand{\oomit}[1]{}

\DeclareMathOperator*{\rank}{rank}
\AtBeginDocument{%
  \providecommand\BibTeX{{%
    \normalfont B\kern-0.5em{\scshape i\kern-0.25em b}\kern-0.8em\TeX}}}

\setcopyright{acmcopyright}
\copyrightyear{2018}
\acmYear{2018}
\acmDOI{XXXXXXX.XXXXXXX}

\acmJournal{JACM}
\acmVolume{37}
\acmNumber{4}
\acmArticle{111}
\acmMonth{8}

\acmPrice{15.00}
\acmISBN{978-1-4503-XXXX-X/18/06}

\begin{document}

\title{Verifying Safety of Neural Networks from Topological Perspectives}

\author{ZHEN LIANG}
\authornote{Both authors contributed equally to this research.}
\email{liangzhen@nudt.edu.cn}
\orcid{0000-0002-1171-7061}
\affiliation{%
  \institution{National University of Defense Technology}
  \streetaddress{Institute for Quantum Information \& State Key Laboratory of High Performance Computing}
  \city{Changsha}
  \state{Hunan}
  \country{China}
  \postcode{410073}
}

\author{DEJIN REN}
\authornotemark[1]
\email{rendj@ios.ac.cn}
\orcid{0000-0001-7779-0096}
\affiliation{%
 \institution{Chinese Academy of Sciences}
 \streetaddress{State Key Lab. of Computer Science}
 \city{Beijing}
 \country{China}}

\author{BAI XUE}
\email{xuebai@ios.ac.cn}
\orcid{0000-0001-9717-846X}
\affiliation{%
 \institution{Chinese Academy of Sciences}
 \streetaddress{State Key Lab. of Computer Science}
 \city{Beijing}
 \country{China}}

\author{JI WANG}
\email{wj@nudt.edu.cn}
\orcid{0000-0003-0637-8744}
\affiliation{%
  \institution{National University of Defense Technology}
  \streetaddress{Institute for Quantum Information \& State Key Laboratory of High Performance Computing}
  \city{Changsha}
  \state{Hunan}
  \country{China}
  \postcode{410073}
}

\author{WENJING YANG}
\email{wenjing.yang@nudt.edu.cn}
\orcid{0000-0002-6997-0406}
\affiliation{%
  \institution{National University of Defense Technology}
  \streetaddress{Institute for Quantum Information \& State Key Laboratory of High Performance Computing}
  \city{Changsha}
  \state{Hunan}
  \country{China}
  \postcode{410073}
}

 \author{WANWEI LIU}
\authornote{Corresponding author.}
\email{wwliu@nudt.edu.cn}
\orcid{0000-0002-2315-1704}
\affiliation{%
  \institution{National University of Defense Technology}
  \streetaddress{College of Computer Science and Technology}
  \city{Changsha}
  \state{Hunan}
  \country{China}
  \postcode{410073}}

\renewcommand{\shortauthors}{Z. Liang and D. Ren, et al.}

\begin{abstract}
Neural networks (NNs) are increasingly applied in safety-critical systems such as autonomous vehicles. However, they are fragile and are often ill-behaved. Consequently, their behaviors should undergo rigorous guarantees before deployment in practice. In this paper, we propose a set-boundary reachability method to investigate the safety verification problem of NNs from a topological perspective. Given an NN with an input set and a safe set, the safety verification problem is to determine whether all outputs of the NN resulting from the input set fall within the safe set. In our method, the homeomorphism property and the open map property of NNs are mainly exploited, which establish rigorous guarantees between the boundaries of the input set and the boundaries of the output set. The exploitation of these two properties facilitates reachability computations via extracting subsets of the input set rather than the entire input set, thus controlling the wrapping effect in reachability analysis and facilitating the reduction of computation burdens for safety verification. The  homeomorphism property exists in some widely used NNs such as invertible residual networks (i-ResNets) and Neural ordinary differential equations (Neural ODEs), and the open map is a less strict property and easier to satisfy compared with the homeomorphism property. For NNs establishing either of these properties, our set-boundary reachability method only needs to perform reachability analysis on the boundary of the input set. Moreover, for NNs that do not feature these properties with respect to the input set, we explore subsets of the input set for establishing the local homeomorphism property and then abandon these subsets for reachability computations. Finally, some examples demonstrate the performance of the proposed method.
\end{abstract}




\received{XX XX 2023}
\received[revised]{XX XX 2023}
\received[accepted]{XX XX 2023}

\maketitle

\section{Introduction}
Machine learning has witnessed rapid growth due to the high amount of data produced in many industries and the increase in computation power. NNs have emerged as a leading candidate computation model for machine learning, which promotes the prosperity of artificial intelligence in various fields, such as computer vision \cite{tian2021image,dahnert2021panoptic}, natural language processing \cite{yuan2021bartscore,karch2021grounding} and so on. Recently, NNs are increasingly applied in safety-critical systems. %For example, an NN has been implemented in the ACAS Xu airborne collision avoidance system for unmanned aircraft \cite{AcasXu2019}, which is a highly safety-critical system and is currently being developed by the Federal Aviation Administration. 
Consequently, to gain users’ trust and ease their concerns, it is of vital importance to ensure that NNs are able to produce safe outputs and satisfy the essential safety requirements before the deployment. 

Safety verification of NNs, which determines whether all outputs of an NN satisfy specified safety requirements via computing output reachable sets, has attracted a huge attention from different communities such as machine learning \cite{lomuscio2017approach,akintunde2019verification}, formal methods \cite{huang2017safety,tran2019star,liuww2020article}, and security \cite{wang2018efficient,gehr2018ai2}.  Because NNs are generally large, nonlinear, and non-convex, exact computation of output reachable sets is challenging. Although there are some methods on exact reachability analysis such as  SMT-based \cite{katz2017reluplex} and polyhedron-based approaches \cite{xiang2017reachable,tran2019parallelizable}, they are usually time-consuming and do not scale well. Moreover, these methods are limited to NNs with ReLU activation functions. Consequently, over-approximate reachability analysis, which mainly involves the computation of super sets of output reachable sets, is often resorted to in practice. The over-approximate analysis is usually more efficient and can be applied to more general NNs beyond ReLU ones. Due to these advantages, an increasing attention has been attracted and thus a large amount of computational techniques have been developed for over-approximate reachability analysis \cite{liu2021algorithms}.%e.g., mixed-integer linear program \cite{dutta2017output}, interval arithmetic- \cite{wang2018efficient}, zonotope- \cite{singh2018fast}, star- \cite{tran2019star}, and abstract-domain-\cite{singh2019abstract,yang2021improving} based approaches. 

Overly conservative over-approximations, however, often render many
safety properties unverifiable in practice. This conservatism
mainly results from the wrapping effect, which is the accumulation of over-approximation errors through layer-by-layer propagation. As the extent of the wrapping effect correlates strongly with the size of the input set \cite{xiang2018output}, techniques that partition the input set and independently compute output reachable sets of the resulting subsets are often adopted to reduce the wrapping effect, especially for the cases of large input sets. Such partitioning may, however, produce a 
 great number of subsets, which is generally exponential in the dimensionality. This will induce extensive demand on computation time and memory, often rendering existing reachability analysis techniques not suitable for safety verification of complex NNs in real applications. Therefore, exploring subsets of the input set rather than the entire input set could help reduce computation burdens and thus accelerate the safety verification tremendously. %This is the objective of this work, which explores means of addressing the safety verification problem of NNs based on reachability analysis of just a small subset of the initial input set.%namely a set enclosing its boundary. 

%Neural networks (NN, for short) have emerged as a leading candidate computation model for deep learning (DL) in recent years, which promote the prosperity of artificial intelligence (AI) in various fields, such as computer vision, natural language processing, speech recognition and so on. 

In this work, we investigate the safety verification problem of NNs from the topological perspective, mainly focusing on the homeomorphism and open map properties. For one thing, we extend the set-boundary reachability method, which is originally proposed for verifying safety properties of systems modeled by ODEs in \cite{xue2016reach}, to safety verification of NNs. In \cite{xue2016reach}, the set-boundary reachability method only performs over-approximate reachability analysis on the initial set's boundary rather than the entire initial set to address safety verification problems. It was built upon the homeomorphism property of ODEs. This nice property also widely exists in NNs, and representative NNs are invertible NNs such as neural ODEs \cite{chen2018neural} and invertible residual networks \cite{behrmann2019invertible}. Consequently, it is straightforward to extend the set-boundary reachability method to safety verification of these NNs, just using the boundary of the input set for reachability analysis which does not involve reachability computations of interior points and thus reducing computation burdens in safety verification. Furthermore, we extend the set-boundary reachability method to general NNs (feedforward NNs) via exploiting the local homeomorphism property with respect to  the input set. This exploitation is instrumental for constructing a subset of the input set for reachability computations, which is gained via removing a set of points in the input set such that the NN is a homeomorphism with respect to them. The above methods of extracting subsets for performing reachability computations can also be applied to intermediate layers of NNs rather than just between the input and output layers. For another thing, since the homeomorphism property has strong constraints on the NN structures, that is, the dimensions of the input and output layers must be the same, which greatly limits the types and scale of applied NNs. Therefore, we consider the more general trapezoidal NN structures, which are characterized by non-strictly monotonically decreasing layer dimensions, and the open mapping property widely exists in these NNs. Relaxing the homeomorphism property to an open mapping does not guarantee that the input set's boundary is mapped to the output set's boundary, however, it ensures that the boundary of the output set must come from the boundary of the input set (there is a redundancy situation where the input set's boundary is mapped to the interior points of the output set). Subsequently, similar to the homeomorphism property, the safety verification can be carried out with the set boundaries.  Finally, we demonstrate the performance of the proposed method on several examples.

The main contributions of this paper are listed as follows.

\begin{itemize}
    \item We investigate the safety verification problem of NNs from the topological perspective. More concretely, we exploit the homeomorphism and the open map properties, and aim at extracting a subset of the input set rather than the entire input set for reachability computations. To the best of our knowledge, this is the first work on the utilization of the topological property to address the safety verification problems of NNs. This might on its own open research directions on digging into insightful topological properties of facilitating reachability computations for NNs. 
    
    \item  The proposed method is able to enhance the capabilities and performances of existing reachability computation methods for the safety verification of NNs via reducing computation burdens. Based on the homeomorphism and the open map properties, the computation burdens of solving the safety verification problems can be reduced for the NNs featuring these topological properties. We further show that the computation burdens can also be reduced for more general NNs by exploiting the local homeomorphism property  established on the subsets of the input set. 
%    \item Moreover, we propose a general framework of reachable set computation (RSCBA) integrated with existing verification tools, not limited to specific abstract domains. Experiments on various neural networks illustrate  the performance on the tightness improvement of reachable set computation of neural networks.
\end{itemize}

The remainder of this paper is structured as follows. First, an overview of the closely relevant research is introduced in Section \ref{sec:rel}. Afterward, we formulate the safety verification problem of interest on NNs in Section \ref{Sec:pre} and then elucidate our set-boundary reachability method for addressing the safety verification problem with either the homeomorphism property or the open map property in Section \ref{Sec:method}. Following this, we demonstrate the performance of our set-boundary reachability method and compare it with existing methods on several examples in Section \ref{Sec:exp}. Finally, we  summarize the paper and discuss potential future work in Section \ref{Sec:conl}.

\section{Related Work}
\label{sec:rel}
%\subsection{Reachability Verification of NNs.}
There has been a dozen of works on safety verification of NNs. The first work on DNN verification was published in \cite{pulina2010abstraction}, which focuses on DNNs with Sigmoid activation functions via a partition-refinement approach. Later, Katz
et al. \cite{katz2017reluplex} and Ehlers \cite{ehlers2017formal} independently implemented Reluplex and Planet, two SMT solvers to verify DNNs with ReLU activation function on properties expressible with SMT constraints. 

%\cite{katz2017reluplex} proposed the Reluplex algorithm, which is sound and complete to verify fully-connected NNs with ReLU activation functions, and then the Reluplex was incorporated into the Marabou framework \cite{katz2019marabou}, which removes the limitation of ReLU function.  \cite{ehlers2017formal} presented a similar method, which however, utilizes linear approximation to over-approximate the behaviours of NNs. \cite{lomuscio2017approach}  encoded the readability analysis of fully-connected NNs into mixed integer linear programs. However, these methods are not efficient. Consequently,  \cite{cheng2017maximum,dutta2017output,bunel2018piecewise} proposed heuristics, Sherlock algorithm and brand and bound to accelerate the computations, respectively.

In recent years, methods based on abstract interpretation attracts much more attention, which is to propagate sets layer by layer in a sound (i.e., over-approximate) way \cite{cousot1977abstract} and is more efficient. There are many widely used abstract domains, such as intervals \cite{wang2018efficient}, and star-sets \cite{tran2019star}. A method based on zonotope abstract domains is proposed in \cite{gehr2018ai2}, which works for any piece linear activation function with great scalability. Then, it is further improved \cite{singh2018fast} for obtaining tighter results via imposing abstract transformation on  ReLU, Tanh and Sigmoid activation functions. \cite{singh2018fast} proposed specialized abstract zonotope transformers for handling NNs with ReLU, Sigmoid and Tanh functions. \cite{singh2019abstract} proposes an abstract domain that combines floating point polyhedra with intervals to over-approximate output reachable sets. Subsequently, a spurious region guided approach is proposed to infer tighter output reachable sets \cite{yang2021improving} based on the method in \cite{singh2019abstract}. \cite{dutta2017output} abstracts an NN by a polynomial, which has the advantage that dependencies can in principle be preserved. This approach can be precise in practice for small input sets.  \cite{kochdumper2022open} completes NN verification with non-convex polynomial zonotope domains, obtaining tighter reachable sets.  Afterwards, \cite{huang2019reachnn} approximates Lipschitz-continuous neural networks with Bernstein polynomials.  \cite{ivanov2020verifying} transforms a neural network with Sigmoid activation functions into a hybrid automaton and then uses existing reachability analysis methods for the hybrid automaton to perform reachability computations. \cite{xiang2018output} proposed a maximum sensitivity based approach for  solving safety verification problems for multi-layer perceptrons with monotonic activation functions. In this approach, an exhaustive search of the input set is enabled by  discretizing input space to compute the output reachable set which consists of a union of reachtubes.

%CROWN-IBP \cite{zhang2019towards} combines the tight linear relaxation on top of verification bound (CROWN) \cite{gowal2018effectiveness} and interval bound propagation (IBP) method \cite{zhang2018efficient} to certify the existence of adversarial examples. ExactReach \cite{xiang2017reachable} computes the reachable set via retaining a set of polytopes on the ReLU activated NNs. Whereas, the number of the polytopes would grow exponentially, so it does not scale even though it computes an exact reachability result.

%It trades efficiency with accuracy. 

%is a balance between the effectiveness and efficiency, that is to say, the finer the partition, the tighter the reachable set, the slower the computation.

% Due to the accuracy loss caused by the over approximation of the real output set, it is not possible to ensure that all the properties satisfied in NNs can be verified, so it is necessary to consider the refinement of the computation of the reachable set. The existing computation methods takes the refinement from the view of an entire set or its partitioned subsets, whereas, 
 
Neural ODEs were first introduced in 2018, which exhibit considerable computational efficiency on time-series modeling tasks \cite{chen2018neural}. Recent years have witnessed an increased use of them on real-world applications \cite{lechner2020neural,hasani2020natural}. However, the verification techniques for Neural ODEs are rare and still in infancy. The first reachability technique for Neural ODEs 
 appeared in \cite{grunbacher2021verification}, which proposed Stochastic Lagrangian reachability, an abstraction-based technique for constructing an over-approximation of the output reachable set with probabilistic guarantees. Later, this method was improved and implemented in a tool GoTube \cite{gruenbacher2021gotube}, which is able to perform reachability analysis for long time horizons. Since these methods only provide stochastic bounds on the computed over-approximation and thus cannot provide formal guarantees on the satisfaction of safety properties, \cite{lopez2022reachability} presented a deterministic verification framework for a general class of Neural ODEs with multiple continuous- and discrete-time layers. 

 Based on entire input sets, all the aforementioned works focus on developing computational techniques for reachability analysis and safety verification of appropriate NNs. In contrast, the present work shifts this focus to topological analysis of NNs and guides reachability computations on subsets of the input set rather than the entire input set, reducing computation burdens and thus increasing the power of existing safety verification methods for NNs. Although there are studies on topological properties of NNs \cite{behrmann2019invertible,dupont2019augmented,naitzat2020topology}, there is no work on the utilization of topological properties to analyze their reachability and safety verification problems, to the best of our knowledge.

%\subsection{Homeomorphism Property in NNs}
%The main research on homeomorphism in NNs can be categorized into two classes, on invertible residual networks (i-ResNet) \cite{behrmann2019invertible} and neural ordinary differential equations (neural ODEs) \cite{chen2018neural}. The former is usually utilized in flow models to reconstruct the data under another distribution \cite{jacobsen2018revnet, behrmann2019invertible}, or to design lightweight neural network models, in terms of storage and training \cite{gomez2017reversible}. As for the latter, homeomorphism is mainly brought to characterize the expression ability and improvement of neural ODEs \cite{zhang2020invariance}. In this paper, our purpose of working on the homeomorphism property is to provide an rigorous guarantee for boundary analysis. 

\section{Preliminaries}
\label{Sec:pre}
In this section, we give an introduction on the safety verification problem of interest for NNs, homeomorphisms and open maps. Throughout this paper, given a set $\Delta$, $\Delta^{\circ}$, $\partial \Delta$ and $\overline{\Delta}$ respectively denotes its interior, boundary and the closure. 

\subsection{Neural Networks}
NNs, also known as artificial NNs, are a subset of machine learning and are at the heart of deep learning algorithms. They work by using interconnected nodes or neurons in a layered structure that resembles a human brain, and are generally composed of three layers: an input layer, hidden layers and an output layer. Mathematically, an NN is a mathematical function $\bm{N}(\cdot): \mathbb{R}^n\rightarrow \mathbb{R}^m$, where $n$ and $m$ respectively denote the dimension of the input and output of the NN. Taking a general neural network $\bm{N}$ with $l+1$ layers as an example, the  corresponding function $\bm{N}(\cdot): \mathbb{R}^n\rightarrow \mathbb{R}^m$  can be represented by the composition of the transformation of each layer, i.e.,
\begin{equation}
    \bm{N}(\bm{x}) = \bm{N}_l\circ \bm{N}_{l-1} \circ \cdots \circ \bm{N}_{2} \circ \bm{N}_{1}(\bm{x})=\sigma_{l}(\bm{W}_{l}\cdots \sigma_{2}(\bm{W}_2\sigma_1(\bm{W}_1 \bm{x} + \bm{b}_1)+\bm{b}_2)+\cdots +\bm{b}_{l}),
     \label{nn}
\end{equation}
where $\bm{W}_i$ and $\bm{b}_i$, $1 \leq i \leq l$, stand for the weight matrix and bias vector between the adjacent $i$-th and $i+1$-th layers. $\sigma_i$ represents the corresponding activation function, such as $\mathtt{Sigmoid}$, $\mathtt{Tanh}$, $\mathtt{ReLU}$ and so on. Moreover, $\sigma_i$ is an element-wise operator, and can be treated as a lifted operator
\begin{equation}
\sigma_i = \sigma_{i}^{d_i} \circ \sigma_{i}^{d_i-1} \circ \cdots \circ \sigma_{i}^{2} \circ \sigma_{i}^{1},
     \label{actfunc}
\end{equation}
 where $d_i$ is the dimension of the $i$-th layer and $\sigma_{i}^{j}$ imposes the activation function only on the $j$-th dimension of the $i$-th layer.

\subsection{Problem Statement}
Given an input set $\mathcal{X}_{in}$, the output reachable set of an NN $\bm{N}(\cdot): \mathbb{R}^n \rightarrow \mathbb{R}^m$ is stated by the following definition.

\begin{definition}
\label{safety}
For a given neural network $\bm{N}(\cdot): \mathbb{R}^n \rightarrow \mathbb{R}^m$, with an input set $\mathcal{X}_{in} \subseteq \mathbb{R}^n$, the output reachable  set $\mathcal{R}(\mathcal{X}_{in})$ is defined as
\[\mathcal{R}(\mathcal{X}_{in})=\{\bm{y}\in \mathbb{R}^m \mid \bm{y}=\bm{N}(\bm{x}), \ \bm{x}\in \mathcal{X}_{in}\}.\]
\end{definition}

The safety verification problem is formulated in Definition \ref{safety1}.
\begin{definition}[Safety Verification Problem]
\label{safety1}
Given a neural network $\bm{N}(\cdot): \mathbb{R}^n \rightarrow \mathbb{R}^m$, an input set $\mathcal{X}_{in}\subseteq \mathbb{R}^n$ which is compact, and a safe set $\mathcal{X}_s\subseteq \mathbb{R}^m$ which is simply connected, the safety verification problem is to verify that 
 \[\forall \bm{x}_0\in \mathcal{X}_{in}. \ \bm{N}(\bm{x}_0) \in \mathcal{X}_s.\]
\end{definition}

In topology, a simply connected set is a path-connected set where one can continuously shrink any simple closed curve into a point while remaining in it. The requirement that the safe set $\mathcal{X}_{s}$ is a simply  connected set is not strict, since many widely used sets such as intervals, ellipsoids, convex polyhedra and zonotopes are simply connected. 

Obviously, the safety property that $\forall \bm{x}_0\in \mathcal{X}_{in}. \  \bm{N}(\bm{x}_0) \in \mathcal{X}_s$ holds if and only if $\mathcal{R}(\mathcal{X}_{in})\subseteq \mathcal{X}_s$. However, it is challenging to compute the exact output reachable set $\mathcal{R}(\mathcal{X}_{in})$ and thus  an over-approximation $\Omega(\mathcal{X}_{in})$, which is a super set of the set $\mathcal{R}(\mathcal{X}_{in})$ (i.e., $\mathcal{R}(\mathcal{X}_{in})\subseteq \Omega(\mathcal{X}_{in})$), is commonly resorted to in existing literature for formally reasoning about the safety property. If $\Omega(\mathcal{X}_{in}) \subseteq \mathcal{X}_s$, the safety property that $\forall \bm{x}_0\in \mathcal{X}_{in}.\  \bm{N}(\bm{x}_0) \in \mathcal{X}_s$ holds.

%Different from existing methods of using entire sets (input set or intermediate set) to compute an over-approximation of the output set for safety verification, we in this work carefully explore boundaries of sets to perform over-approximate reachability analysis of subsets of the input set for safety verification. 

\subsection{Homeomorphisms}
%As studied in \cite{xue2016reach}, if the given neural network $\bm{N}(\cdot): \mathcal{X}_{in} \rightarrow \mathcal{R}(\mathcal{X}_{in})$ is a homeomorphism, the boundary $\partial \mathcal{R}(\mathcal{X}_{in})$ of the output reachable set can be obtained via propogating the input set's boundary. 
In this subsection, we will recall the definition of a homeomorphism, which is a map between spaces that preserves all topological properties.
\begin{definition}
A map $h: \mathcal{X}\rightarrow \mathcal{Y}$ with $\mathcal{X},\mathcal{Y}\subseteq \mathbb{R}^n$ is a homeomorphism with respect to $\mathcal{X}$ if it is a continuous bijection and its inverse $h^{-1}(\cdot): \mathcal{Y}\rightarrow \mathcal{X}$ is also continuous. 
\end{definition}

Homeomorphisms are continuous functions that preserve topological properties, which map boundaries to boundaries and interiors to interiors \cite{massey2019basic}, as illustrated in Fig. \ref{home}.
\begin{proposition}
 Suppose sets $\mathcal{X},\mathcal{Y}\subseteq \mathbb{R}^n$ are compact. If a map $h(\cdot): \mathcal{X}\rightarrow \mathcal{Y}$ is a homeomorphism, then  $h$ maps the boundary of the set $\mathcal{X}$ onto the boundary of the set $\mathcal{Y}$, and the interior of  the set $\mathcal{X}$ onto the interior of the set $\mathcal{Y}$.
\end{proposition}

\begin{figure}[htbp]
\centering
\subfigure[A non-homeomorphic map]{
\begin{minipage}[t]{0.32\linewidth}
\centering
\includegraphics[width=1.6in]{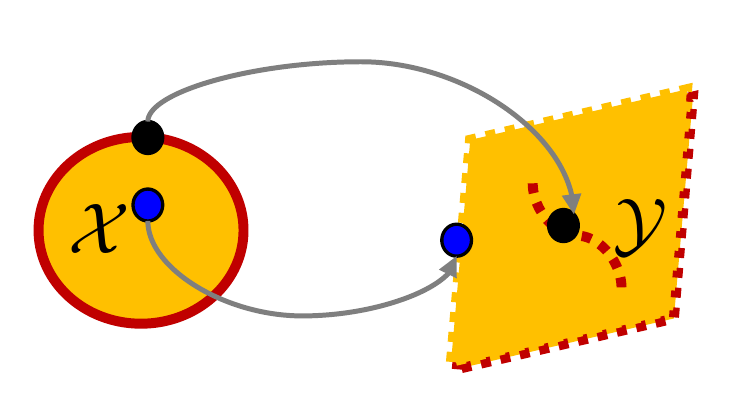}
%\caption{fig2}
\label{non-homeo}
\end{minipage}%
}%
\subfigure[A homeomorphic map]{
\begin{minipage}[t]{0.32\linewidth}
\centering
\includegraphics[width=1.6in]{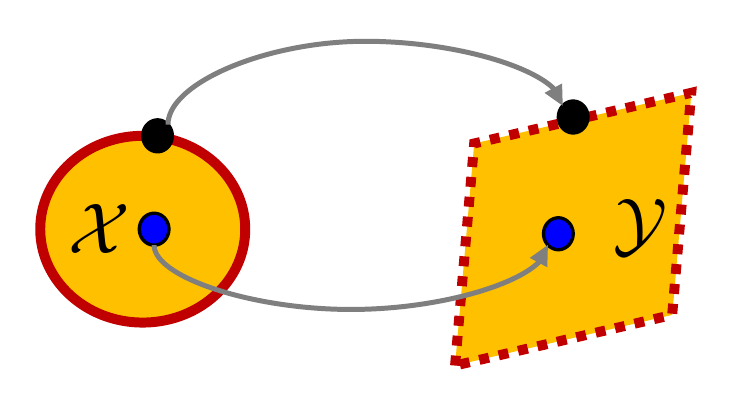}
%\caption{fig1}
\label{home}
\end{minipage}%
}%
\subfigure[An open map]{
\begin{minipage}[t]{0.32\linewidth}
\centering
\includegraphics[width=1.6in]{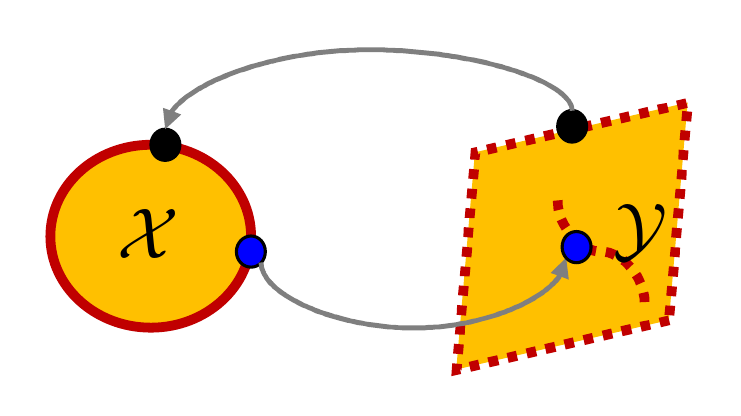}
%\caption{fig2}
\label{openmap}
\end{minipage}%
}%
\centering
\caption{Non-homeomorphic, homeomorphic and open maps}
\label{ill-home}

\end{figure}

Based on this property, \cite{xue2016reach} proposed a set-boundary reachability method for safety verification of ODEs, via only propagating the initial set's boundary. Later, this method was extended to a class of delay differential equations \cite{xue2020over}.

\subsection{Open maps}
Besides homeomorphisms,  open maps also provide insight into the set-boundary analysis, whose definitions are given in the following.

\begin{definition}
A map $f$ that maps open sets to open sets is an  open map \cite{stein2011princeton}, and a subset $\mathcal{O}$ of a metric space is said to be open if  $\mathcal{O}$ is a neighborhood of each of its points \cite{mendelson1990introduction}.
\end{definition}

Different from preserving all
topological properties with homeomorphisms, open maps guarantee that the preimage of the boundary of the output set lies in the boundary of the input set, as shown in Fig. \ref{openmap} and clarified in Proposition \ref{boundary analysis}. 
\begin{proposition}
 \label{boundary analysis}
Let $\mathcal{A}$ and $\mathcal{B}$ be topological spaces, $f:\mathcal{A} \rightarrow \mathcal{B}$ is an open map, $\mathcal{X} \subset \mathcal{A}$ and $\mathcal{Y} \subset \mathcal{B}$, $\mathcal{Y}=f(\mathcal{X})$, then an inverse
image $\bm{x} \in \{\bm{x}|f(\bm{x})=\bm{y}, \bm{x} \in \mathcal{X}\}$ of a boundary point $\bm{y} \in \mathcal{Y}$ is a boundary point in $\mathcal{X}$.   
\end{proposition}

In this subsection, we also supplement some necessary preliminaries about open maps for their identification in neural networks in later section.
\begin{lemma}
(Open Map Theorem.) Suppose $\mathcal{X}$ and $\mathcal{Y}$ are Banach spaces, and \ $f: \mathcal{X} \rightarrow \mathcal{Y}$ is a continuous linear transformation. If \   $f$ is surjective, then \  $f$ is an open map.
\label{open map}
\end{lemma}

\begin{lemma}\label{f_open}
Suppose $f:\mathbb{R}\rightarrow \mathbb{R}$ is a strictly monotonic and continuous map, then $f$ is an open map.
\end{lemma}
\begin{proof}
Without loss of generality, suppose $f$ is strictly increasing map, i.e., $\forall x_1> x_2 \in \mathbb{R}, f(x_1) > f(x_2)$, then, obviously,  $f$ is injective. Let $\mathcal{Y}=f(\mathbb{R})$, then $f:\mathbb{R}\rightarrow \mathcal{Y}$ is bijective and the inverse $f^{-1}: \mathcal{Y} \rightarrow \mathbb{R}$ exists. It is easy to prove $f^{-1}$~ is continuous with the `epsilon-delta' definition of continuity \cite{rudin1976principles}, thus $f$ is an open map.
\end{proof}

\begin{lemma}
    Suppose $f_1: \mathcal{X} \rightarrow \mathcal{Y}$ and \ $f_2: \mathcal{Y} \rightarrow \mathcal{Z}$ are both open maps, then $f = f_2 \circ f_1$ is also an open map, where $\circ$ denotes the composition operator.
\label{comp}
\end{lemma}

\begin{proof}
Using the counterfactual method, let $\bm{y}$ be a boundary point in $\mathcal{Y}$, $\bm{x}$ is an inverse image of $\bm{y}$. Assuming that $\bm{x}$ is an interior point of $\mathcal{X}$, then exist an open set $\mathcal{O}_{\mathcal{X}} \subset \mathcal{X}$ such that $\bm{x} \in \mathcal{O}_{\mathcal{X}}$, because $f$ is an open map, $\mathcal{O}_{\mathcal{Y}}=f(\mathcal{O}_{\mathcal{X}})$ is also an open set in $\mathcal{Y}$, $\bm{y} \in \mathcal{O}_{\mathcal{Y}} \subset \mathcal{Y}$, it infers that $\bm{y}$ is an interior point in $\mathcal{Y}$, which is contradicted that $\bm{y}$ is a boundary point in $\mathcal{Y}$.
\end{proof}

\section{Safety Verification Based on Boundary Analysis}
\label{Sec:method} In this section, we introduce our set-boundary reachability method for addressing the safety verification problem in the sense of Definition \ref{safety} from a topological perspective. We respectively in Subsection  \ref{Sec-inn} and Subsection \ref{safety-om} introduce our set-boundary reachability method for safety verification from the aspects of  homeomorphisms and open maps established in NNs.

\subsection{Safety Verification with Homeomorphisms}
In this subsection, we introduce our set-boundary reachability method with homeomorphisms. We first consider invertible NNs in Subsection \ref{Sec-inn}, and then extend the method to more general NNs in Subsection \ref{Sec-noninn}.

\subsubsection{Safety verification on invertible NNs}\label{Sec-inn}
Generally, the NNs featuring the homeomorphism property are termed invertible NNs. Invertible NNs, such as i-RevNets \cite{jacobsen2018revnet}, RevNets \cite{gomez2017reversible}, i-ResNets \cite{behrmann2019invertible} and Neural ODEs \cite{chen2018neural}, are NNs with invertibility by designed architectures, which are  extensively used in flow model and can reconstruct inputs from their outputs. These NNs are continuous bijective maps. Based on the facts that $\mathcal{X}_{in}$ is compact, they are homeomorphisms [Corollary 2.4, \cite{joshi1983introduction}]\footnote{A continuous bijection from a compact space onto a Hausdorff space is a homeomorphism. (Euclidean space and any subset of Euclidean space is Hausdorff.)}.

In the existing literature, many invertible NNs are constructed by requiring their Jacobian determinants to be non-zero \cite{ardizzone2018analyzing}. Consequently, based on the inverse function theorem \cite{krantz2002implicit}, these NNs are homeomorphisms. In the present work, we also use Jacobian determinants to justify the invertibility of some NNs. It is noteworthy that Jacobian determinants being non-zero is a sufficient but not necessary condition for homeomorphisms and the reason that we resort to this requirement lies in the simple and efficient computations of Jacobian determinants with interval arithmetic. However, this demands the differentiability of NNs. Thus, this  technique of computing Jacobian determinants to determine homeomorphisms is not applicable to NNs with ReLU activation functions.

%As for NNs with non-differential ReLU activation functions, one may find that the computation of Jacobian determinants excludes the applicability of our method, nevertheless, it is difficult to establish homeomorphisms in fact since the mappings on inactive states (i.e., $(-\infty, 0] \rightarrow 0$) is not injective.} 

Based on the homeomorphism property of mapping the input set's boundary onto the output reachable set's boundary, we propose a set-boundary reachability method for safety verification of invertible NNs, which just performs the over-approximate reachability analysis on the input set's boundary. Its computation procedure is presented in Algorithm \ref{alg: iNNs}. 
\begin{algorithm}[htbp]

\caption{Safety Verification Framework for Invertible NNs Based on Boundary Analysis}
\label{alg: iNNs}
\begin{algorithmic}[1] %[1] enables line numbers
\REQUIRE an invertible NN $\bm{N}(\cdot): \mathbb{R}^n \rightarrow \mathbb{R}^n$, an input set $\mathcal{X}_{in}$ and a safe set $\mathcal{X}_s$.
\ENSURE  \textbf{Safe} or \textbf{Unknown}.
\STATE extract the boundary $\partial \mathcal{X}_{in}$ of the input set $\mathcal{X}_{in}$;
\STATE apply existing methods to compute an over-approximation $\Omega(\partial \mathcal{X}_{in})$;
\IF {$\Omega(\partial \mathcal{X}_{in})\subseteq \mathcal{X}_s$} 
\STATE return \textbf{Safe}
%\STATE $\mathcal{B}_{out}=T(\mathbb{N}, \mathcal{B}_{in})$.\label{Line 2}
%\STATE $\widehat{\mathcal{X}_{out}}=$CVX$(\mathcal{B}_{out})$.\label{Line 3}
\ELSE
\STATE return \textbf{Unknown}
\ENDIF
%\STATE $\mathcal{X}_{in}=\widehat{\mathcal{X}_{out}}$.
%\STATE \textbf{return} $\widehat{\mathcal{X}_{out}}$
\end{algorithmic}
\end{algorithm}

\begin{remark}
   In the second step  of Algorithm \ref{alg: iNNs}, we may take partition operator on the input set's boundary to refine the computed over-approximation for addressing the safety verification problem. The computations can be accelerated via parallel techniques.
\end{remark}

\begin{theorem}[Soundness]
\label{sound}
If Algorithm \ref{alg: iNNs} returns \textbf{\em Safe},  the safety property in the sense of Definition \ref{safety} holds.    
\end{theorem}
\begin{proof}
It is equivalent to show that if $\mathcal{R}(\partial \mathcal{X}_{in})\subseteq \mathcal{X}_s$, 
\[\forall \bm{x}_0\in \mathcal{X}_{in}. \ \bm{N}(\bm{x}_0)\in \mathcal{X}_s.\]
The conclusion holds by Lemma 3 in \cite{xue2016reach}.
\end{proof}

%\textcolor{red}{The property that the boundary of the output reachable set is a subset of the output reachable set of the input set's boundary (i.e., $\partial \mathcal{R}(\mathcal{X}_{in})\subseteq \mathcal{R}(\partial \mathcal{X}_{in})$) exists in some NNs.  }

In order to enhance the understanding of Algorithm \ref{alg: iNNs} and its benefits, we use a sample example  to illustrate it. %An illustration of using Algorithm \ref{alg: iNNs} for safety verification of invertible NNs is presented in Example \ref{ex1}.

\begin{figure}[htbp]
\centering
\subfigure[\textcolor{blue}{$\mathcal{R}(\mathcal{X}_{in})$}; \textcolor{red}{$\mathcal{R}(\partial \mathcal{X}_{in})$}]{
\begin{minipage}[t]{0.32\linewidth}
\centering
\includegraphics[width =1.7in, height=1.1in]{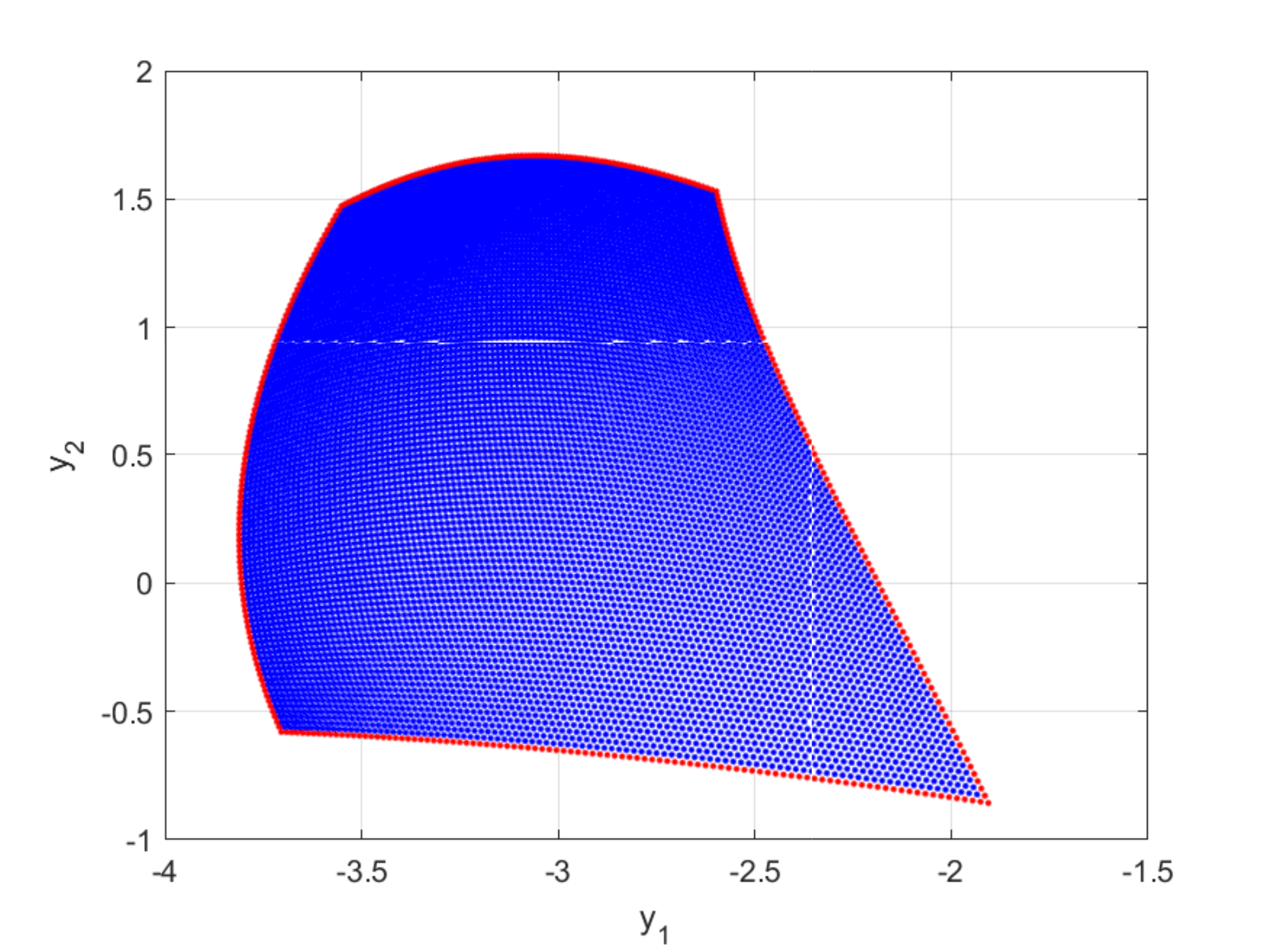}
%\caption{fig1}
\label{illu_eps1}
\end{minipage}%
}%
\subfigure[\textcolor{blue}{$\partial \Omega(\mathcal{X}_{in})$}; \textcolor{red}{$\Omega(\partial \mathcal{X}_{in})$}; \textcolor{green}{$\partial \mathcal{X}_s$}]{
\begin{minipage}[t]{0.32\linewidth}
\centering
\includegraphics[width =1.7in, height=1.1in]{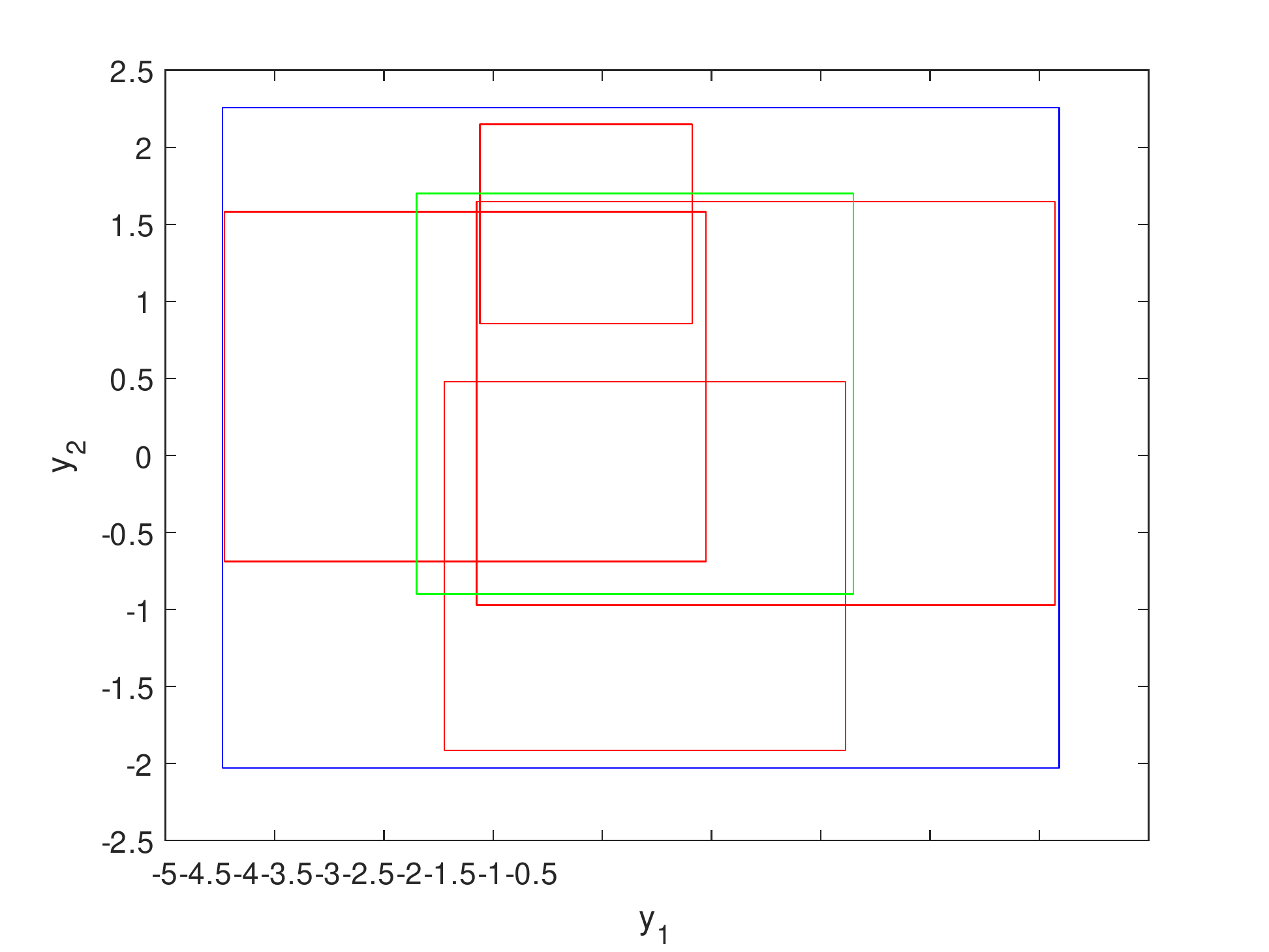}
%\caption{fig1}
\label{illu_eps10}
\end{minipage}%
}%
\subfigure[\textcolor{blue}{$\Omega(\mathcal{X}_{in})$}; \textcolor{red}{$\Omega(\partial \mathcal{X}_{in})$}; \textcolor{green}{$\partial \mathcal{X}_s$}]{
\begin{minipage}[t]{0.32\linewidth}
\centering
\includegraphics[width =1.7in, height=1.1in]{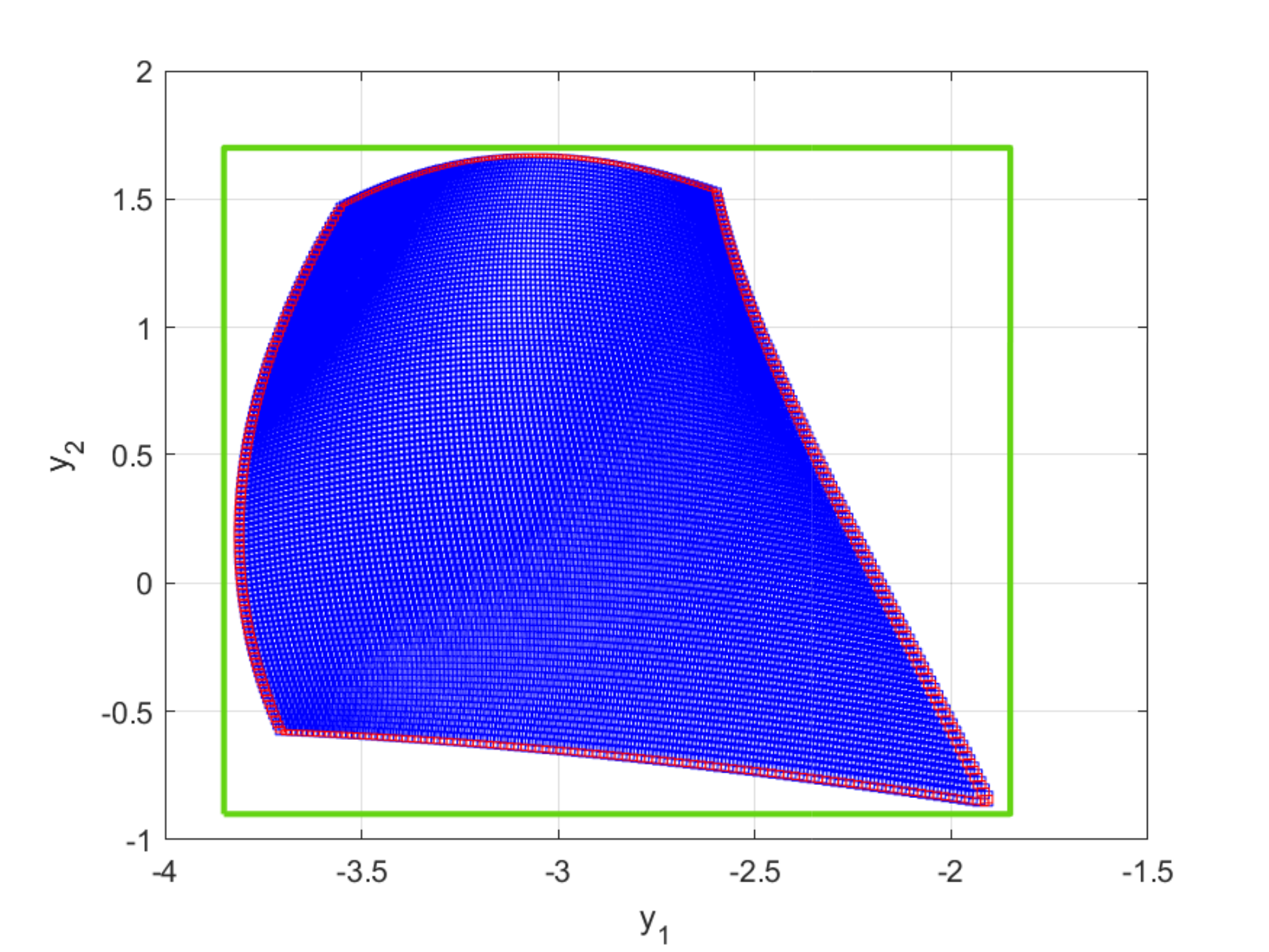}
%\caption{fig2}
\label{illu_eps11}
\end{minipage}%
}%
\centering
\caption{Illustrations on Example \ref{ex1}}
\label{ex1_figure}
\end{figure}

\begin{example}
\label{ex1}
Consider an {\em NN} from \cite{xiang2018output}, which has 2 inputs, 2 outputs and 1 hidden layer consisting of 5 neurons. The input set is $\mathcal{X}_{in}=[0,1]^2$. Its boundary is $\partial \mathcal{X}_{in}=\cup_{i=1}^4 \mathcal{B}_i$, where $\mathcal{B}_1=[0,0]\times [0,1]$, $\mathcal{B}_2=[1,1]\times [0,1]$, $\mathcal{B}_3=[0,1]\times [0,0]$ and $\mathcal{B}_4=[0,1]\times [1,1]$. The activation functions for the hidden layer and the output layer are $\mathtt{Tanh}$ and $\mathtt{Purelin}$ functions, respectively, whose weight matrices and bias vectors can be found in Example 1 in \cite{xiang2018output}. For this neural network, based on interval arithmetic, we can show that the determinant of the Jacobian matrix $\frac{\partial \bm{y}}{\partial \bm{x}_0}=\frac{\partial \bm{N}(\bm{x}_0)}{\partial \bm{x}_0}$ is non-zero for any $\bm{x}_0\in \mathcal{X}_{in}$. Therefore, this NN is invertible and the map $\bm{N}(\cdot):\mathcal{X}_{in} \rightarrow \mathcal{R}(\mathcal{X}_{in})$ is a homeomorphism with respect to the input set $\mathcal{X}_{in}$, leading to $\mathcal{R}(\partial \mathcal{X}_{in})=\partial \mathcal{R}(\mathcal{X}_{in})$. This statement is also verified via the visualized results in Fig. \ref{illu_eps1}. 

The homeomorphism property facilitates the reduction of the wrapping effect in over-approximate reachability analysis and thus reduces computation burdens in addressing the safety verification problem in the sense of Definition \ref{safety}. For this example, with the safe set $\mathcal{X}_s=[-3.85, -1.85]\times[-0.9,1.7]$, we first %do not take partition operator on%
take the input set and its boundary for reachability computations. Based on interval arithmetic, we respectively compute over-approximations $\Omega(\mathcal{X}_{in})$ and $\Omega(\partial \mathcal{X}_{in})$, which are illustrated in Fig. \ref{illu_eps10}. Although the approximation $\Omega(\partial \mathcal{X}_{in})$ is indeed smaller than $\Omega(\mathcal{X}_{in})$, it still renders the safety property unverifiable. We next take partition operator for more accurate reachability computations. If the entire input set is used, we can successfully verify the safety property when the entire input set is divided into  $10^4$ small intervals of equal size. In contrast, our set-boundary reachability method just needs $400$ equal partitions on the input set's boundary, significantly reducing the computation burdens. The reachability results, i.e., the computation of $\Omega(\partial \mathcal{X}_{in})$, are illustrated in Fig. \ref{illu_eps11}.

\end{example}

\subsubsection{Safety Verification on non-invertible NNs}\label{Sec-noninn}

When an NN has the homeomorphism property, we can use Algorithm \ref{alg: iNNs} to address the safety verification problem in the sense of Definition \ref{safety}.  However, not all of NNs have such a nice property. In this subsection we extend the set-boundary reachability method to safety verification of non-invertible NNs, via analyzing the homeomorphism property of NNs with respect to subsets of the input set $\mathcal{X}_{in}$, named local homeomorphism property.
\begin{example}
\label{ex2}
Consider an NN from \cite{xiang2018reachability}, which has 2 inputs, 2 outputs and 1 hidden layer consisting of 7 neurons. The input set is $\mathcal{X}_{in}=[-1,1]^2$. The activation functions for the hidden layer and the output layer are $\mathtt{Tanh}$ and $\mathtt{Purelin}$ functions, respectively, whose  weight matrices and bias vectors can be found in Example 4.3 in \cite{xiang2018reachability}. For this neural network, the boundary of the output reachable set, i.e., $\partial \mathcal{R}(\mathcal{X}_{in})$, is not included in the output reachable set of the input set's boundary $\mathcal{R}(\partial \mathcal{X}_{in})$. This statement is visualized in Fig. \ref{illu_eps2}.

\iffalse
\begin{figure}[htb!]
\center
\includegraphics[width=1.7in]{Fig/illu_2.eps} 
%and  each component of $\bm{c}$ is in $[-10^2,10^2]$ 
\caption{blue region-- $\mathcal{R}(\mathcal{X}_{in})$; red region--$\mathcal{R}(\partial \mathcal{X}_{in})$.}
\label{illu_eps2}
\end{figure}
\fi
\end{example}

Example \ref{ex2} presents us an NN, whose mapping does not admit the homeomorphism property with respect to the input set and the output reachable set. However, the NN may feature the homeomorphism property with respect to a subset of the input set. This is illustrated in Example \ref{ex3}.
\begin{example}
    \label{ex3}
    Consider the NN in Example \ref{ex2} again. We divide the input set $\mathcal{X}_{in}$ into $4\times 10^4$ small intervals of equal size and verify whether the NN is a homeomorphism with respect to each of them based on the use of interval arithmetic to determine the determinant of the  corresponding Jacobian matrix $\frac{\partial \bm{y}}{\partial \bm{x}_0}=\frac{\partial \bm{N}(\bm{x}_0)}{\partial \bm{x}_0}$. The blue region in Fig. \ref{illu_eps3} is the set of intervals, which features the NN with the homeomorphism property. The number of these intervals is 31473. For simplicity, we denote these intervals by set $\mathcal{A}$. 

\iffalse
    \begin{figure}[htb!]
\center
\includegraphics[width=1.8in]{Fig/illu_3.eps} 
%and  each component of $\bm{c}$ is in $[-10^2,10^2]$ 
\caption{blue region-- the set $\mathcal{A}$}
\label{illu_eps3}
\end{figure}
\fi
\end{example}

\begin{figure}[htbp]
\centering
\subfigure[\textcolor{blue}{$\mathcal{R}(\mathcal{X}_{in})$}; \textcolor{red}{$\mathcal{R}(\partial \mathcal{X}_{in})$}]{
\begin{minipage}[t]{0.32\linewidth}
\centering
\includegraphics[width =1.7in, height=1.1in]{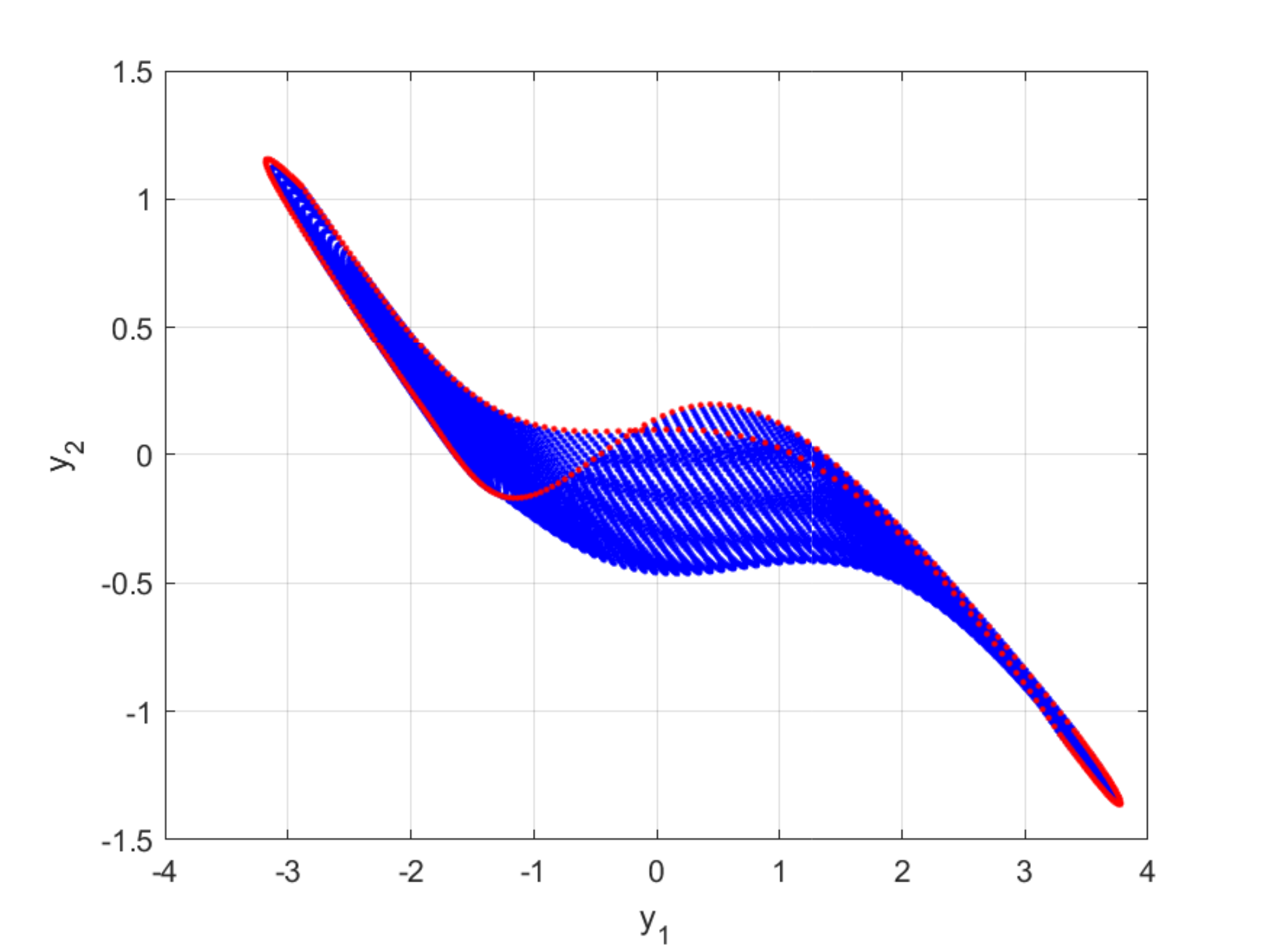}
%\caption{fig1}
\label{illu_eps2}
\end{minipage}%
}%
\subfigure[Set \textcolor{blue}{$\mathcal{A}$}]{
\begin{minipage}[t]{0.32\linewidth}
\centering
\includegraphics[width = 1.7in, height=1.1in]{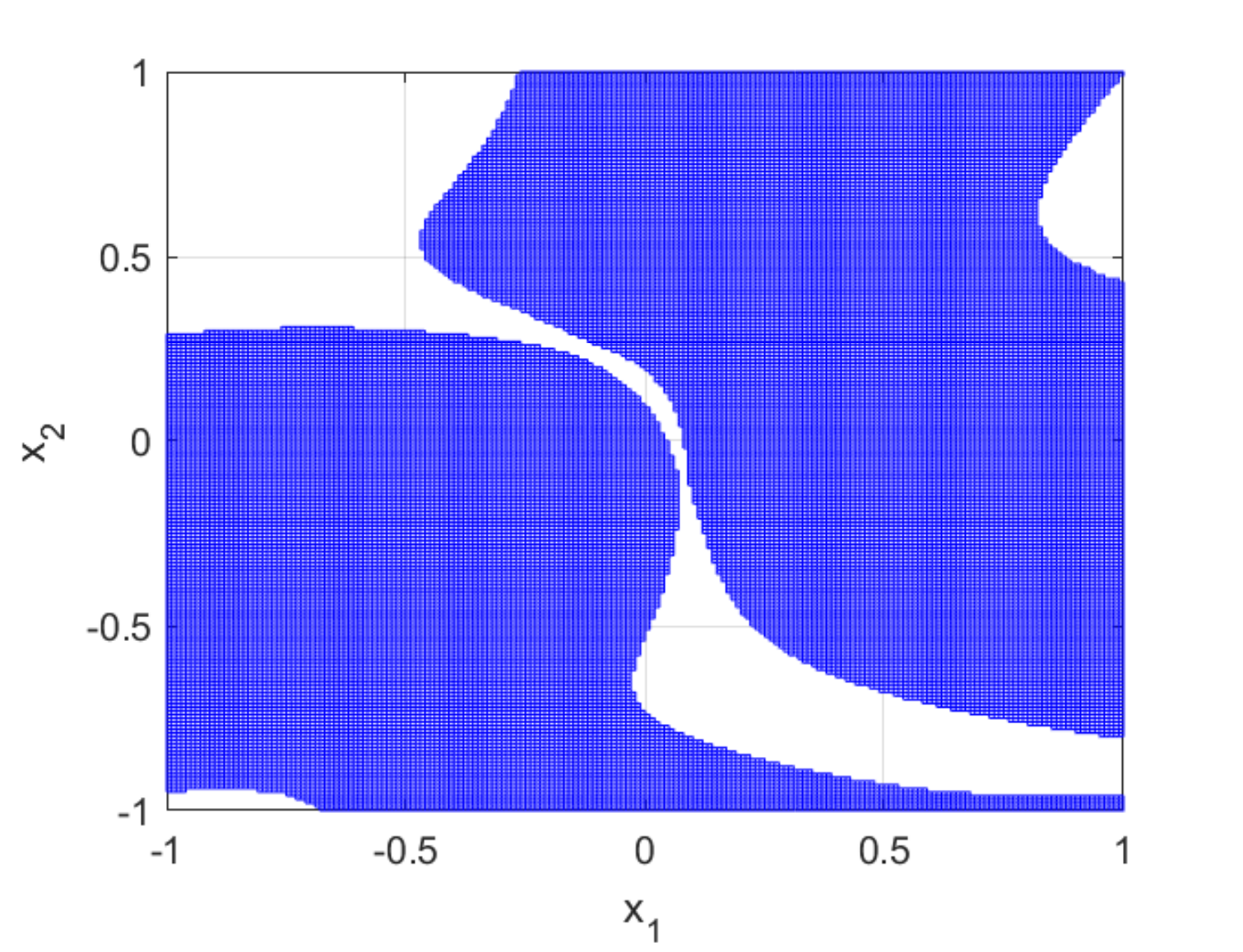}
%\caption{fig1}
\label{illu_eps3}
\end{minipage}%
}%
\subfigure[\textcolor{blue}{$\Omega(\mathcal{X}_{in})$}; \textcolor{red}{$\Omega(\overline{\mathcal{X}_{in}\setminus \mathcal{A}})$}]{
\begin{minipage}[t]{0.32\linewidth}
\centering
\includegraphics[width = 1.7in, height=1.1in]{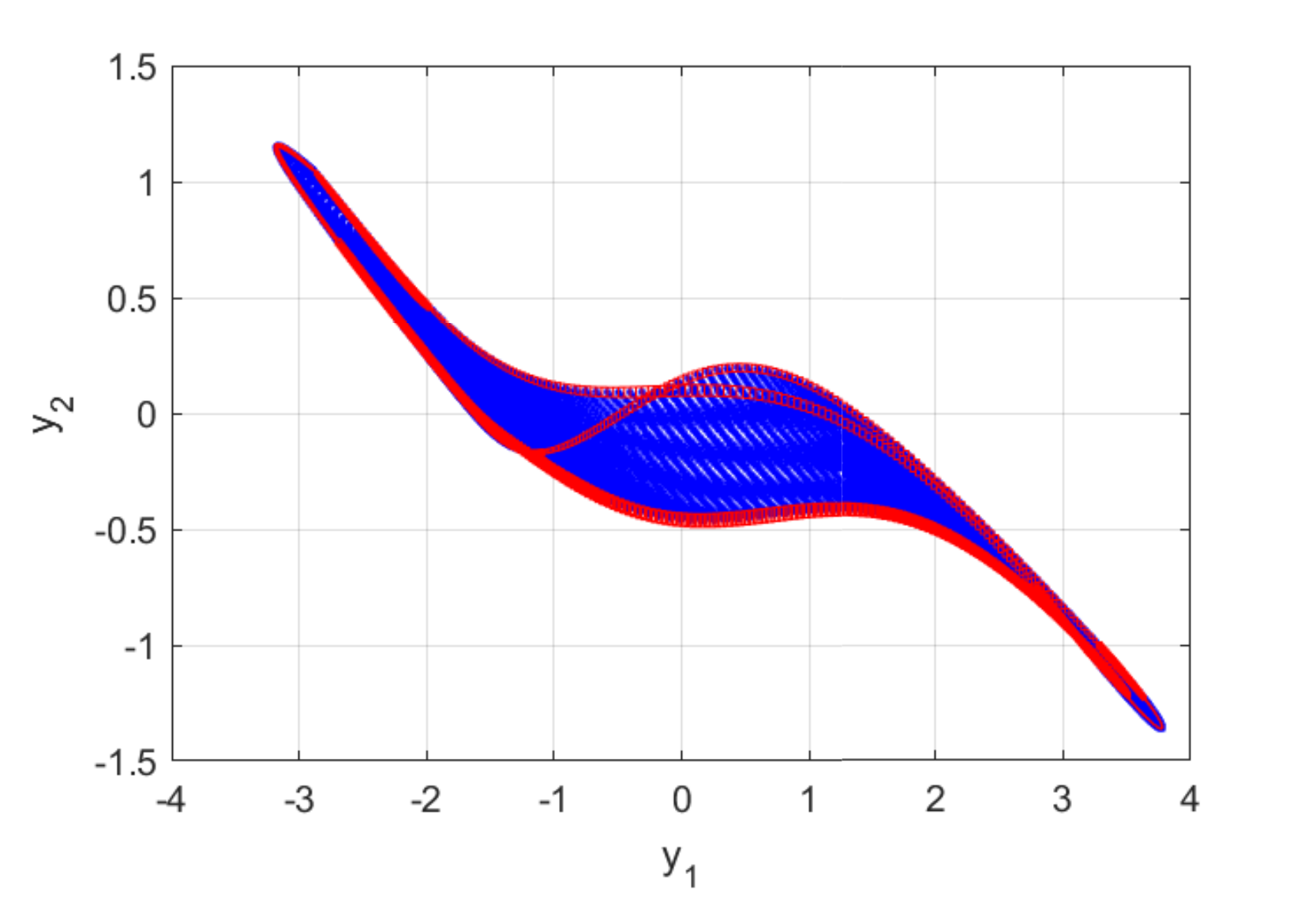}
%\caption{fig2}
\label{illu_eps4}
\end{minipage}%
}%
\centering
\caption{Illustrations on Example \ref{ex2}, \ref{ex3} and \ref{ex4}}
\label{ex234_figure}

\end{figure}

It is interesting to find that the safety verification in the sense of Definition \ref{safety} can be addressed by performing reachability analysis on a subset of the input set $\mathcal{X}_{in}$. This subset is obtained via removing subsets in the input set $\mathcal{X}_{in}$, which features the NN with the homeomorphism property.
\begin{theorem}
\label{wadiao}
Let $\mathcal{A}\subseteq \mathcal{X}_{in}$ and $\mathcal{A}\cap \partial \mathcal{X}_{in}=\emptyset$, and $\bm{N}(\cdot): \mathcal{A}\rightarrow \mathcal{R}(\mathcal{A})$ be a homeomorphism with respect to the input set $\mathcal{A}$. Then, if the output reachable set of the closure of the set $\mathcal{X}_{in}\setminus \mathcal{A}$ is a subset of the safe set $\mathcal{X}_s$, i.e., $\mathcal{R}(\overline{\mathcal{X}_{in}\setminus \mathcal{A}})\subseteq \mathcal{X}_{s}$, the safety property that $\forall \bm{x}_0\in \mathcal{X}_{in}. \ \bm{N}(\bm{x}_0)\in \mathcal{X}_s$
holds.
\end{theorem}
\begin{proof}
Obviously, if $\mathcal{R}(\mathcal{A})\subseteq \mathcal{X}_{s}$ and $\mathcal{R}(\overline{\mathcal{X}_{in}\setminus \mathcal{A}})\subseteq \mathcal{X}_s$, the safety property that $\forall \bm{x}_0\in \mathcal{X}_{in}.\  \bm{N}(\bm{x}_0)\in \mathcal{X}_s$
holds.

According to Theorem \ref{sound}, we have that if $\mathcal{R}(\partial \mathcal{A}) \subseteq \mathcal{X}_{s}$, 
the safety property that $\forall \bm{x}_0\in \mathcal{A}.\  \bm{N}(\bm{x}_0)\in \mathcal{X}_s$ holds. 

According to the condition that $\mathcal{A}\subseteq \mathcal{X}_{in}$ and $\mathcal{A}\cap \partial \mathcal{X}_{in}=\emptyset$, we have that $\mathcal{A} \subseteq \mathcal{X}_{in}^{\circ}$ and thus $\partial \mathcal{A}\subseteq \overline{\mathcal{X}_{in}\setminus \mathcal{A}}$. Therefore, $\mathcal{R}(\overline{\mathcal{X}_{in}\setminus \mathcal{A}})\subseteq \mathcal{X}_s$ implies that $\forall \bm{x}_0\in \mathcal{X}_{in}. \ \bm{N}(\bm{x}_0)\in \mathcal{X}_s$. The proof is completed.
\end{proof}

Theorem \ref{wadiao} tells us that it is still possible to use a subset of the input set for addressing the safety verification problem in Definition \ref{safety}, even if the given NN is not a homeomorphism with respect to the entire input set $\mathcal{X}_{in}$. This is shown in Example \ref{ex4}.
\begin{example}
    \label{ex4}
    Consider the situation in Example \ref{ex3} again. If the entire input set is used for computations, all of $4\times 10^4$ small intervals participate in calculations. However, 
    Theorem \ref{wadiao} tells us that only 9071 intervals (i.e., subset $\overline{\mathcal{X}_{in}\setminus \mathcal{A}}$) are needed, which is much smaller than $4\times 10^4$. The computation results based on interval arithmetic are illustrated in Fig. \ref{illu_eps4}. It is noting that $9071$ intervals rather than 8527 $(=4\times 10^4 -31473)$ intervals are used since some intervals, which have non-empty intersection with the boundary of the input set $\mathcal{X}_{in}$ (since Theorem \ref{wadiao} requires $\mathcal{A}\cap \partial \mathcal{X}_{in} = \emptyset$), should participate in calculations.

\iffalse
    \begin{figure}[htb!]
\center
\includegraphics[width=1.8in]{Fig/illu_31.eps} 
%and  each component of $\bm{c}$ is in $[-10^2,10^2]$ 
\caption{blue region-- $\Omega(\mathcal{X}_{in})$; red region--$\Omega(\overline{\mathcal{X}_{in}\setminus \mathcal{A}})$}
\label{illu_eps4}
\end{figure}
\fi
\end{example}

\begin{remark}
According to Theorem \ref{wadiao}, we can also observe that the boundary of the output reachable set $\mathcal{R}(\mathcal{X}_{in})$ is included in the output reachable set of the input set $\overline{\mathcal{X}_{in}\setminus \mathcal{A}}$, i.e., $\partial \mathcal{R}(\mathcal{X}_{in}) \subseteq \mathcal{R}(\overline{\mathcal{X}_{in}\setminus \mathcal{A}})$. This can also be visualized in Fig. \ref{illu_eps4}. Consequently, this observation may open new research directions of addressing various problems of NNs \cite{ghorbani2019interpretation}.  For instance, it may facilitate the generation of adversarial examples, which are inputs causing the NN to falsify the safety property, and the characterization of decision boundaries of NNs, which are a surface that separates data points belonging to different class labels.  
\end{remark}

Therefore, we arrive at an algorithm for safety verification of non-invertible NNs, which is formulated in Algorithm \ref{alg: iNNs1}.
\begin{algorithm}[tb]
\caption{Safety Verification Framework for Non-Invertible NNs}
\label{alg: iNNs1}

\begin{algorithmic}[1] %[1] enables line numbers
\REQUIRE a non-invertible NN $\bm{N}(\cdot): \mathbb{R}^n \rightarrow \mathbb{R}^n$, an input set $\mathcal{X}_{in}$ and a safe set $\mathcal{X}_s$.\\
\ENSURE \textbf{Safe} or \textbf{Unknown}.
\STATE determine a subset $\mathcal{A}$ of the set $\mathcal{X}_{in}$ such that $\bm{N}(\cdot): \mathbb{R}^n \rightarrow \mathbb{R}^n$ is a homeomorphism with respect to it;
\STATE apply existing methods to compute an over-approximation $\Omega(\overline{\mathcal{X}_{in}\setminus  \mathcal{A}})$;
\IF {$\Omega(\overline{\mathcal{X}_{in}\setminus  \mathcal{A}})\subseteq \mathcal{X}_s$} 
\STATE return \textbf{Safe}
%\STATE $\mathcal{B}_{out}=T(\mathbb{N}, \mathcal{B}_{in})$.\label{Line 2}
%\STATE $\widehat{\mathcal{X}_{out}}=$CVX$(\mathcal{B}_{out})$.\label{Line 3}
\ELSE
\STATE return \textbf{Unknown}
\ENDIF
%\STATE $\mathcal{X}_{in}=\widehat{\mathcal{X}_{out}}$.
%\STATE \textbf{return} $\widehat{\mathcal{X}_{out}}$
\end{algorithmic}
\end{algorithm}

\begin{theorem}[Soundness]
    \label{sound1}
If Algorithm \ref{alg: iNNs1} returns \textbf{\em Safe},  the safety property that $\forall \bm{x}_0\in \mathcal{X}_{in}. \ \bm{N}(\bm{x}_0) \in \mathcal{X}_s$
holds.    
\end{theorem}
\begin{proof}
This conclusion can be assured by Theorem \ref{wadiao}.
\end{proof}

\begin{remark}
    The set-boundary reachability method can also be applied to intermediate layers in a given NN, rather than just the input and output layers. Suppose that there exists a sub-NN $\bm{N}'(\cdot):\mathbb{R}^{n'}\rightarrow \mathbb{R}^{n'}$, which maps the input of the $l$-th layer to the output of the $k$-th layer, in the given NNs, and its input set is $\mathcal{X}'_{in}$ which is an over-approximation of the output reachable set of the $(l-1)$-th layer. If $\bm{N}'(\cdot): \mathbb{R}^{n'}\rightarrow \mathbb{R}^{n'}$ is a homeomorphism with respect to $\mathcal{X}'_{in}$, we can use $\partial \mathcal{X}'_{in}$ to compute an over-approximation $\Omega'(\partial \mathcal{X}'_{in})$ of the output reachable set $\{\bm{y}\mid \bm{y}=\bm{N}'(\bm{x}_0),\bm{x}_0\in \partial \mathcal{X}'_{in}\}$; otherwise, we can apply Theorem \ref{wadiao} and compute an over-approximation $\Omega'(\overline{\mathcal{X}'_{in}\setminus \mathcal{A}})$ of the output reachable set $\{\bm{y}\mid \bm{y}=\bm{N}'(\bm{x}_0),\bm{x}_0\in \overline{\mathcal{X}'_{in}\setminus \mathcal{A}}\}$. In case that the $k$-th layer is not the output layer of the NN, we need to construct a simply connected set, like convex polytope, zonotope or interval, to cover $\Omega'(\partial \mathcal{X}'_{in})$ or $\Omega'(\overline{\mathcal{X}'_{in}\setminus \mathcal{A}})$ for the subsequent layer-by-layer propagation. This set is an over-approximation of the output reachable set of the $k$-th layer, according to Lemma 1 in \cite{xue2016under}.
\end{remark}

\begin{remark}
    Any existing over-approximate  reachability methods such as interval arithmetic- \cite{wang2018efficient}, zonotopes- \cite{singh2018fast}, star sets \cite{tran2019star} based methods, which are suitable for given NNs, can be used to compute the involved over-approximations, i.e., $\Omega(\partial \mathcal{X}_{in})$ and $\Omega(\overline{\mathcal{X}_{in}\setminus \mathcal{A}})$, in Algorithm \ref{alg: iNNs} and \ref{alg: iNNs1}.
\end{remark}

\subsection{Safety Verification with Open Maps}
\label{safety-om}
Actually, the homeomorphism property imposes a strict limitation on NNs, i.e., the same input and output layer dimensions.  Just take the MNIST dataset with handwritten digital images of 28$\times$28 pixels as an example, fully-connected neural networks for the classification task are generally with 784 inputs and 10 outputs and they violate the basic requirement for satisfying homeomorphisms. Subsequently, in this subsection, we relax homeomorphisms into open maps, which are easier to meet in NNs while keeping the necessary properties for safety verification  via set-boundary analysis. More concretely, open maps ensure that the boundaries of the output set come from the boundary of the input set, and the redundancy situation that mapping the input set's boundary to the output set's boundary has no impact on the safety verification, declared in the following.

First, we give the identification method of open maps in NNs. On one hand, for the affine transformations involved in NNs, the conditions required for open maps are declared as follows.

\begin{lemma}
Suppose $f: \mathbb{R}^{m} \rightarrow \mathbb{R}^{n}$, $m\geq n$ and $f=\bm{W}\bm{x}+\bm{b}$, if \  $\bm{W}$ is full rank, then $f$ is an open map.
\label{af_open}
\end{lemma}

\begin{proof}
 Considering the linear map $f_1: \mathbb{R}^{m} \rightarrow \mathbb{R}^{n}$, and  $f_1 = \bm{Wx}$, for any $\bm{y} \in \mathbb{R}^{n}$,  the equation  $\bm{y}=\bm{W}\bm{x}$ with $\rank(\bm{W})=n$  exists a solution $\bm{x} \in \mathbb{R}^{m}$. Thus,  $f_1$ is a  linear map from $\mathbb{R}^{m}$ onto $\mathbb{R}^{n}$. Meanwhile, it is a bounded  (continuous) linear map. With the open map theorem (Lemma \ref{open map}), $f_1$ is an open map. Further, obviously, the mapping $f$ is an open map.
\end{proof}

On the other hand, it comes to the activation function utilized in NNs. Herein, We restrict the activation function to be strictly monotonic and continuous. It is not a strict requirement that most of the activation functions satisfy it, just like $\mathtt{Sigmoid}$, $\mathtt{Tanh}$ and  the variants $\mathtt{LeakyReLU}$ \cite{maas2013rectifier} and $\mathtt{ELU}$ (Exponential Linear Unit) \cite{clevert2015fast}. Lemma \ref{f_open} and \ref{acf_open}  reveal the reason why we impose the \emph{strictly monotonic and continuous} constraint on them.

\begin{lemma}
Strictly monotonic and continuous activation functions are open maps.
\label{acf_open}
\end{lemma}
\begin{proof}
    According to the lifted operator in Eq. \eqref{actfunc} and Lemma \ref{comp}, it is obvious that any strictly monotonic and continuous functions are open maps.
\end{proof}

Based on the analysis of affine transformations and activation functions contained in NNs, from the view of open maps, it is readily to reveal that open maps widely exist in neural networks, which is summarized in the following theorem.

\begin{theorem}\label{nohid}
Given an NN $\bm{N}$ without hidden layers, which has $m$ inputs $n$ outputs, and the activation function $\sigma$ for the output layer is strictly monotonic and continuous. The mapping between the input and output layer is denoted  by $\bm{N}(\cdot): \mathbb{R}^{m} \rightarrow \mathbb{R}^{n}$, with input $\bm{x} \in \mathbb{R}^{m}$, output $\bm{y} \in \mathbb{R}^{n}$, i.e., $\bm{y}=\sigma(\bm{W}\bm{x}+\textbf{b})$, where $\bm{W} \in \mathbb{R}^{n \times m}$, $\bm{b} \in \mathbb{R}^{n}$. If $m\geq n$ and $\bm{W}$ is full rank, then the map $\bm{N}(\cdot): \mathbb{R}^{m} \rightarrow \mathbb{R}^{n}$ is an open map.
\end{theorem}

\begin{proof}
It is easy to conclude by combining Lemma \ref{af_open}, \ref{acf_open} and Lemma \ref{comp}.
\end{proof}

Further, we generalize the neural network in sense of Theorem \ref{nohid} into more general networks, which are named the  trapezoidal network  structures. 

\begin{theorem}\label{hidblock}
Given an NN  $\bm{N}$ with $l$ layers, whose accompanied function is denoted by $\bm{N}(\cdot): \mathbb{R}^{m} \rightarrow \mathbb{R}^{n}$, with input $\bm{x} \in \mathbb{R}^{m}$, output $\bm{y} \in \mathbb{R}^{n}$, i.e., $\bm{y}=\bm{N}(\bm{x})=\bm{N}_{l} \circ \dots \circ \bm{N}_{2}\circ \bm{N}_{1}(\bm{x})$. If the numbers of neurons located in any two adjacent layers are non-increasing and all the weight matrices are full rank,  also, activation functions are strictly monotonic and continuous, then the map $\bm{N}(\cdot): \mathbb{R}^{m} \rightarrow \mathbb{R}^{n}$  is an open map. 
\end{theorem}
\begin{remark}
For the safety verification with open maps, non-increasing layer dimensions of NNs are required and that is why they are termed trapezoidal network  structures. Moreover, for those networks which do not share the non-increasing structure globally, this style structure may also exist in them locally.
\end{remark}

The following examples show that open maps are more relaxed properties compared with homeomorphisms.

\begin{example}
\label{exm_home}
Consider an NN, which has 2 inputs, 2 outputs and 2 hidden layers consisting of 4 and 3 neurons, respectively. The same input and output dimensions are designed for the possible homeomorphism property.  The input set is $\mathcal{X}_{in}=[-0.5,0.5]^2$. 
The activation functions for the hidden layer and the output layer are $\mathtt{Sigmoid}$ and $\mathtt{Purelin}$ functions, respectively, whose weight matrices and bias vectors are listed as follows. 
\begin{gather*}
     \bm{W}_1 = \begin{bmatrix}
    -0.3143  & -1.2349 \\
    0.6374  &  0.0476 \\
    0.5337  &  0.9933 \\
    0.1006  & -0.3207 
    \end{bmatrix},  \bm{b}_1 = \begin{bmatrix}
   0.9499 \\
    0.7459 \\
    1.8229 \\
   -0.7684
    \end{bmatrix}
    \\
     \bm{W}_2 = \begin{bmatrix}
    0.7289 &  -0.8494   & 1.0250 &  -1.2558 \\
    1.4836 &  -1.2679 &   1.8014  & -0.8406 \\
   -0.0152 &   1.8449  &  0.6851  & -1.7672
    \end{bmatrix}, \bm{b}_2= \begin{bmatrix}
    -1.8237 \\
   -0.8818 \\
   -1.3181
    \end{bmatrix}
    \\
    \bm{W}_3 = \begin{bmatrix}
    0.7084  & -0.5921 &   1.1720\\
    1.6097 &  -1.5994  & -0.1263
    \end{bmatrix},  \bm{b}_3 = \begin{bmatrix}
    0.7494 \\
    0.7323
    \end{bmatrix}
\end{gather*}    
 For this neural network, based on interval arithmetic, we can show that the determinant of the Jacobian matrix $\frac{\partial \bm{y}}{\partial \bm{x}_0}=\frac{\partial \bm{N}(\bm{x}_0)}{\partial \bm{x}_0}$ is not non-zero for any $\bm{x}_0\in \mathcal{X}_{in}$. Therefore, the map $\bm{N}(\cdot):\mathcal{X}_{in} \rightarrow \mathcal{R}(\mathcal{X}_{in})$ is not a homeomorphism with respect to the input set $\mathcal{X}_{in}$, leading to $\mathcal{R}(\partial \mathcal{X}_{in})\neq \partial \mathcal{R}(\mathcal{X}_{in})$. This statement is also verified via the visualized results in Fig. \ref{exm_nonhome}. Whereas, the map between the first hidden layer and the output layer is an open map via testing the conditions described in Theorem \ref{hidblock}.
\end{example}

\begin{figure}[htbp]
\centering
\subfigure[Example \ref{exm_home}: \textcolor{blue}{$\mathcal{R}(\mathcal{X}_{in})$}; \textcolor{red}{$\mathcal{R}(\partial \mathcal{X}_{in})$}]{
\begin{minipage}[t]{0.5\linewidth}
\centering
\includegraphics[width=2in]{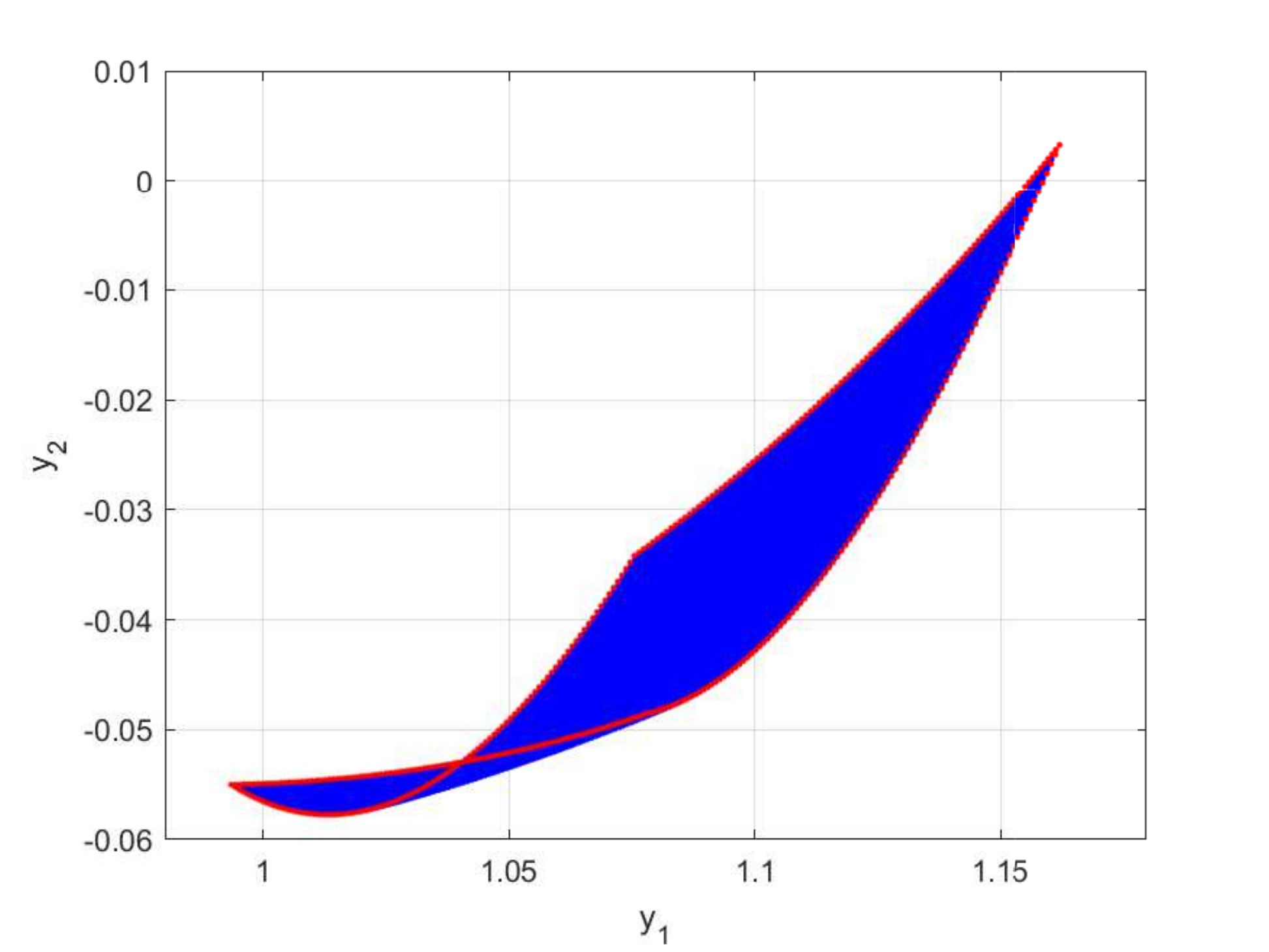}
%\caption{fig1}
\label{exm_nonhome}
\end{minipage}%
}%
\subfigure[Output reachable sets: \textcolor{blue}{$\Omega(\mathcal{X}_{in})$}; \textcolor{red}{$\Omega_O(\mathcal{X}_{in})$}]{
\begin{minipage}[t]{0.5\linewidth}
\centering
\includegraphics[width=2in]{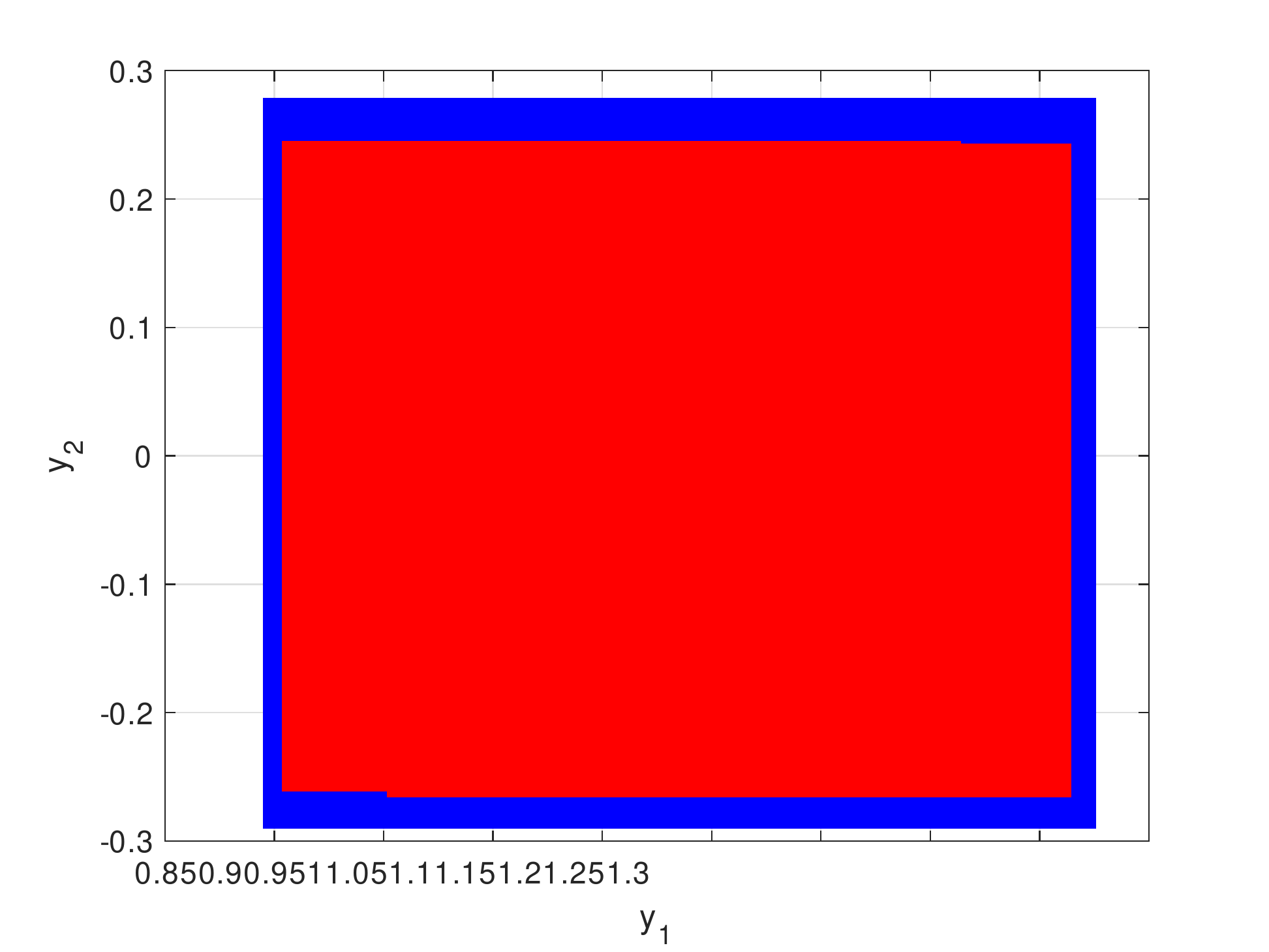}
%\caption{fig2}
\label{refine}
\end{minipage}%
}%
\centering
\caption{Illustrations on Example \ref{exm_home} and Example \ref{exm_refine}. }
\label{fig_exam}
\end{figure}

From Example \ref{exm_home}, it can be observed that open maps are easy to hold in neural networks, but also to be verified in practice. Based on the property of open maps, the safety verification of NNs establishing open maps can be carried out with the procedures presented in Algorithm \ref{alg: iNNs}, demonstrated in the following example. More importantly, the properties of open maps on set-boundary analysis provide a soundness guarantee of the safety verification, as illustrated in Theorem \ref{sound}.

\begin{example}
\label{exm_refine}
Considering the neural network in Example \ref{exm_home} again, we also take the input set as $\mathcal{X}_{in}=[-0.5, 0.5]^{2}$. In this example, we take the Interval Bound Propagation (IBP) \cite{gowal2018effectiveness} as the baseline, which is an over-approximation reachable sets computing algorithm with intervals and the reachable set of IBP is $\Omega(\mathcal{X}_{in}) = [0.8952, 1.2753]\times [-0.2896, 0.2783]$. As for the safety verification with open maps, the first iteration works with the entire input set as the dimension of the first hidden layer is larger than the input dimension. Afterward, open maps are established during all the remaining loops. Consequently, for the reachable set of the first loop, $[0.5437, 0.8487]\times [0.5995, 0.7481]\times [0.7426, 0.9300]\times [0.2731, 0.3641]$, we extra its 8 (2 $\times$ 4) set boundaries for the subsequent computation and obtain the final reachable set $\Omega_O(\mathcal{X}_{in}) = [0.9040, 1.2640]\times [-0.2656, 0.2449]$. Obviously,  $\Omega_O(\mathcal{X}_{in})\subset \Omega(\mathcal{X}_{in})$ and the reachable set computed by set-boundary analysis reduces 10.11\% and 5.29\% in each dimension, respectively. This statement is also verified via the visualized results in Fig. \ref{refine}.
\end{example}

Through Example \ref{exm_refine}, we also see that set-boundary analysis based on open maps, like that on homeomorphisms, exhibits nice compatibility with over-approximate reachability methods. This means that open maps based set-boundary analysis can further refine the output reachable sets computed by existing reachability tools. In addition, we do not require the entire NNs to establish open maps, and as long as certain layers within them establish open maps, then set-boundary analysis can be coupled with existing methods.

\subsection{Comparisons and Discussion}
For better understanding, we also make some comparisons between homeomorphisms and open maps established in NNs in the following.

    \begin{itemize}
        \item  \textbf{Establishment condition.} Homeomorphisms require two neural network layers with the same dimensions, while open maps demand trapezoidal network structures. Both these properties impose specific structure requirements for the NNs and the  trapezoidal network structures are more common in NNs. Also, some requirements are proposed for the activation functions.
        \item \textbf{Identification method and efficiency.} It is computationally prohibitive to verify the Jacobian matrix being non-zero to identify homeomorphisms in practice. For some invertible NNs, like neural ODE, the homeomorphism property exists naturally without checking. In contrast, the identification of open maps is more time-saving by algebraic analysis, i.e., checking matrix ranks. 
        \item  \textbf{Verification Guarantee.} Homeomorphisms and open maps both provide soundness guarantees for our set-boundary analysis based safety verification, and open maps have some redundancy situations, i.e., sometimes mapping the input set's 
        boundary point to the output set's interior point.
    \end{itemize}

In sum, for one thing, compared with homeomorphisms, open maps are much more relaxed properties existing in NNs, either from the aspect of their establishment or their identification. For another thing, even though open maps do not preserve all the topological properties like homeomorphisms, they still maintain the necessary topological property for our set-boundary based safety verification of NNs. Moreover, as shown in Examples \ref{ex1} and \ref{exm_refine}, compared with the entire input set based reachability methods, either homeomorphisms based or open maps based set-boundary analysis render tighter reachable sets and easier tackled safety verification with less computation burdens.

\section{Experiments}
\label{Sec:exp}
In this section, several examples of NNs are used to demonstrate the performance of the proposed set-boundary reachability method for safety verification. Experiments are conducted from the aspects of homeomorphisms and open maps, respectively. Recall that the proposed set-boundary method is applicable for any reachability analysis algorithm based on set representation, resulting in tighter and verifiable over-approximations when existing approaches fail.  Thus, we compare  the set-boundary method versus the entire set based one on some existing reachability tools in terms of effectiveness and efficiency.

\noindent\textbf{Experiment Setting.} All the experiments herein are run on the platform with MATLAB 2021a, Intel (R) Core (TM) i7-10750H CPU@2.60 GHz and RAM 16 GB. The codes and models are available from \url{https://github.com/laode2022/BoundaryNN}.

\subsection{Experiments on safety verification with homeomorphisms}
In this subsection, we conduct some experiments on safety verification utilizing homeomorphisms and local homeomorphisms, which are respectively evaluated on invertible NNs and general feedforward NNs.

\subsubsection{Experiments on Invertible NNs}
We carry out some examples involving neural ODEs and invertible feedforward neural networks.

\noindent{\textbf{Neural ODEs.}} We experiment on two widely-used neural ODEs in  \cite{manzanas2022reachability}, which are respectively a nonlinear 2-dimensional spiral \cite{chen2018neural} with the input set $\mathcal{X}_{in}=[1.5,2.5]\times [-0.5,0.5]$ and the safe set $\mathcal{X}_s=[-0.08, 0.9]\times [-1.5,-0.3]$  and a 12-dimensional controlled cartpole \cite{gruenbacher2020lagrangian} with the input set $\mathcal{X}_{in}=[-0.001,0.001]^{12}$ and the safe set $\mathcal{X}_s=\{\bm{y}\in \mathbb{R}^{12} \mid y_1 \in [0.0545,0.1465], y_2\in [0.145,0.725]\}$. For simplicity, we respectively denote them by $\bm{N}_1$ and $\bm{N}_2$.

\begin{figure}[htbp]
\centering
\subfigure[\textcolor{blue}{$4\times4$} Vs. \textcolor{red}{$4^2$}; \textcolor{blue}{Safe}, \textcolor{red}{Unknown}]{
\begin{minipage}[t]{0.48\linewidth}
\centering
\includegraphics[width=2in]{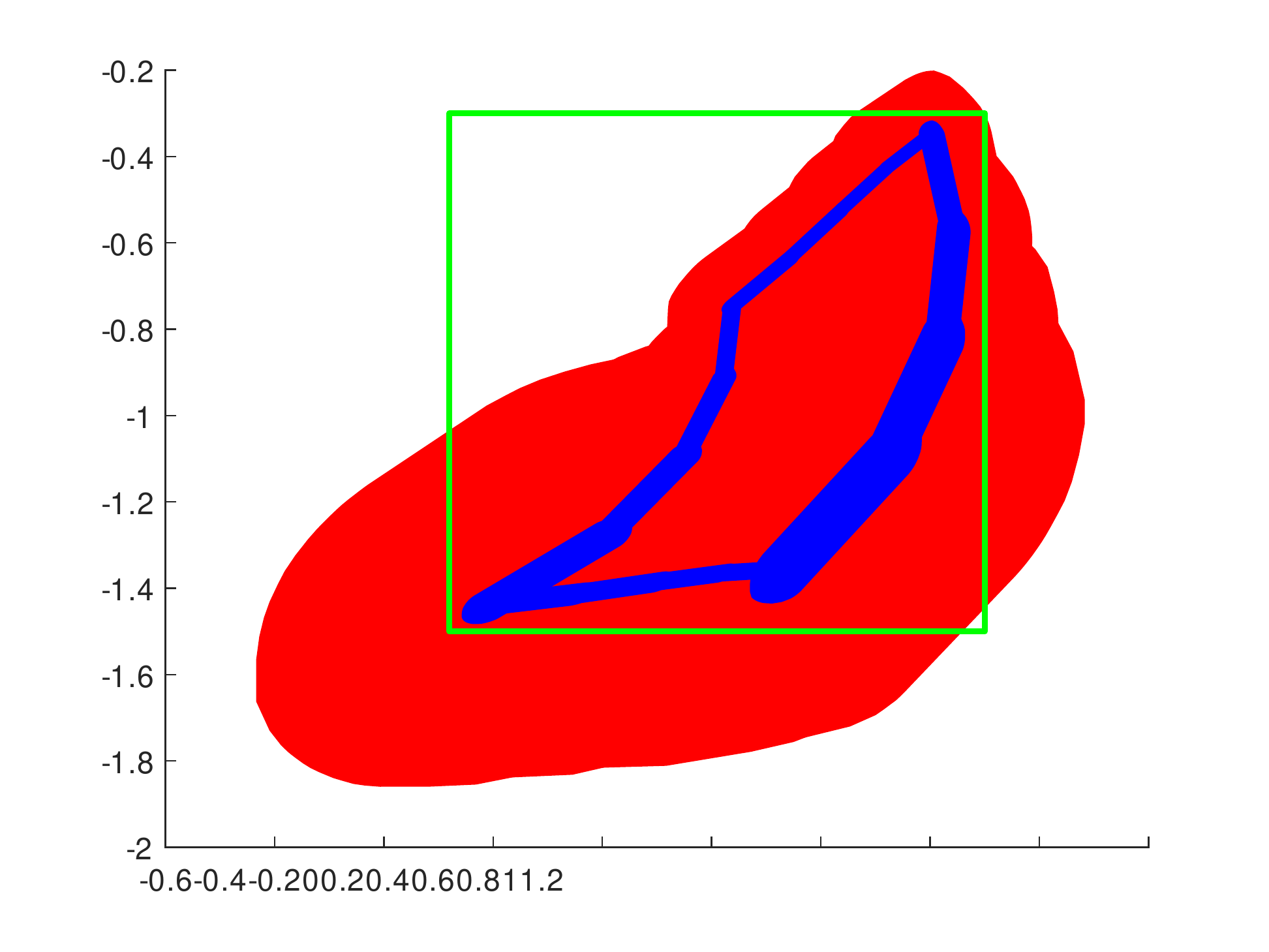}
%\caption{fig2}
\label{spiral_2}
\end{minipage}%
}%
\subfigure[\textcolor{blue}{$4\times4$} Vs. \textcolor{red}{$7^2$}; \textcolor{blue}{Safe}, \textcolor{red}{Safe}]{
\begin{minipage}[t]{0.48\linewidth}
\centering
\includegraphics[width=2in]{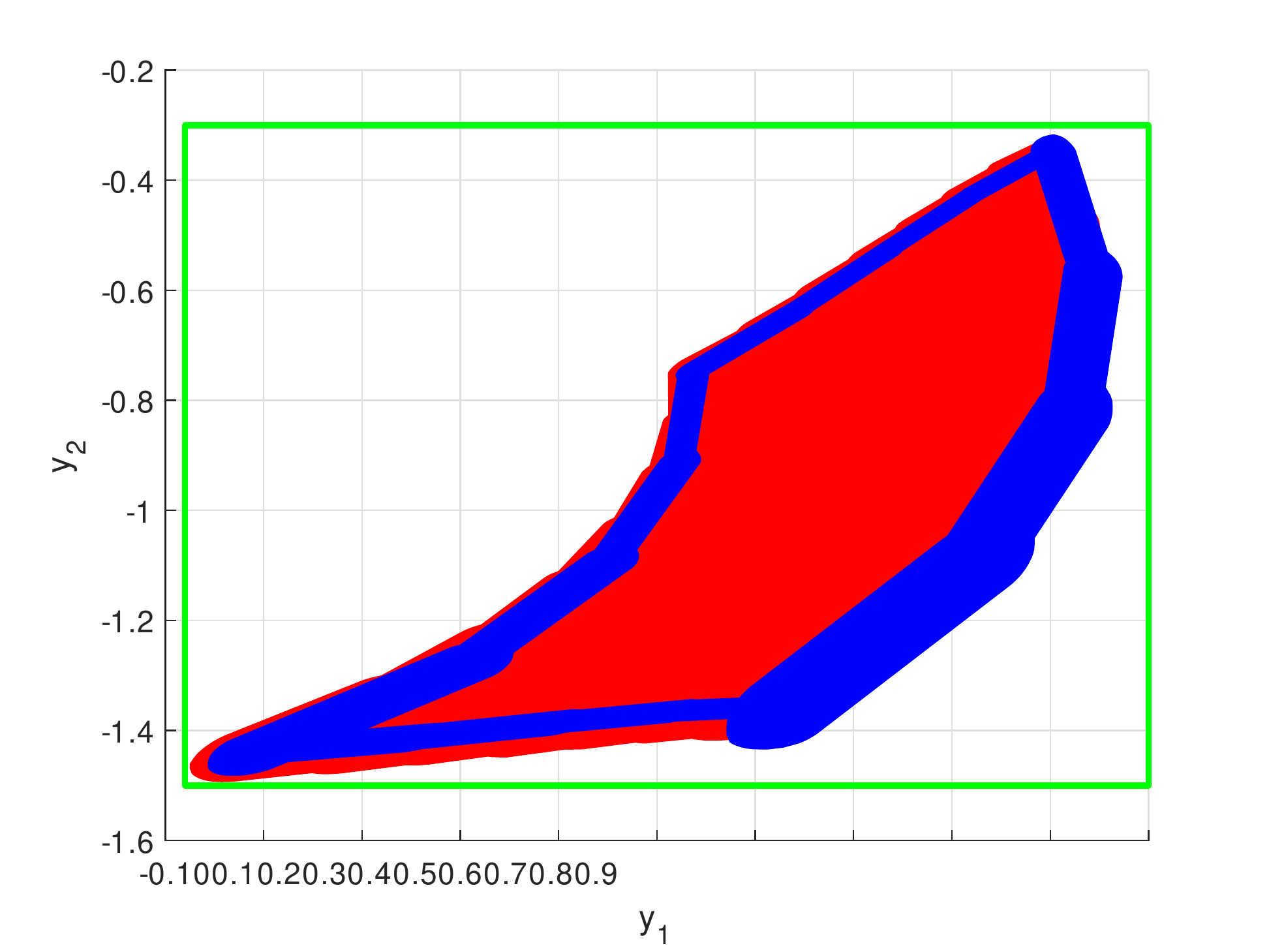}
%\caption{fig2}
\label{spiral_3}
\end{minipage}
}%%
%\\
%\centering
%\subfigure[Reachable Sets]{
%\begin{minipage}[t]{0.5\linewidth}
%\centering
%\includegraphics[width=1.35in]{Fig/Cartpole_1.1_1e-3.eps}
%\caption{fig1}
\label{FPA}
%\end{minipage}%
%}%
%\subfigure[Verification:  \textcolor{blue}{Safe}, \textcolor{red}{Unknown}]{
%\begin{minipage}[t]{0.5\linewidth}
%\centering
%\includegraphics[width=1.35in]{Fig/cartpole_res_region.eps}
%\caption{fig2}
%\label{FPA_conc}
%\end{minipage}%
%}%
\caption{Verification on $\bm{N}_1$.
\textcolor{blue}{$\Omega(\partial \mathcal{X}_{in})$}; \textcolor{red}{$\Omega(\mathcal{X}_{in})$}; \textcolor{green}{$\partial \mathcal{X}_{s}$}}
\label{node}
\end{figure}

Here, we take zonotopes as abstract domains and compare the output reachable sets computed by the proposed set-boundary reachability method and the entire input set based method. The over-approximate reachability analysis is performed using the continuous reachability analyzer CORA \cite{althoff2015introduction}, with some minor modifications \cite{manzanas2022reachability}. 

\begin{table}[htbp]

    \centering
     \caption{Time consumption of safety verification on $\bm{N}_1$ and $\bm{N}_3$}
    \setlength{\tabcolsep}{3mm}{
    \begin{tabular}{*{7}{c}}
  \toprule
  \multirow{2}*{\textbf{Partition}} & \multicolumn{3}{c}{\textbf{Network} $\bm{N}_1$} & \multicolumn{3}{c}{\textbf{Network} $\bm{N}_3$}  \\
  \cmidrule(lr){2-4}\cmidrule(lr){5-7}
  &\textbf{Entire set} & \textbf{Boundary} & \textbf{Ratio} & \textbf{Entire set} & \textbf{Boundary} & \textbf{Ratio}  \\
  \midrule
 1 & 3.3417&	29.8787&	8.941   & 1.5080 &	5.6662 &	3.757 \\
 2 & 31.9848&	106.0663&	3.316  & 5.9560 &	11.4954 &	1.930 \\
 3& 111.6617 &	159.6125&	1.429  &  13.5740&	17.1452&	1.263 \\
 4 & 235.0475&	\textbf{\textcolor{blue}{220.8258}}&	0.939& 23.9392&	22.9996&	0.961\\
 5& 325.2500&	\textcolor{gray}{-}&	\textcolor{gray}{-} &  37.2290 &	\textbf{\textcolor{blue}{28.4822}}& 0.765 \\
 6& 477.5187&	\textcolor{gray}{-}&	\textcolor{gray}{-} & 54.4542&	\textcolor{gray}{-}&	\textcolor{gray}{-} \\
 7& \textbf{\textcolor{red}{671.3200}}&	\textcolor{gray}{-}&	\textcolor{gray}{-} &73.9100&	\textcolor{gray}{-}&	\textcolor{gray}{-} \\
 8& \textcolor{gray}{-}&	\textcolor{gray}{-}&	\textcolor{gray}{-} & \textbf{\textcolor{red}{95.5307}}&	\textcolor{gray}{-}&	\textcolor{gray}{-} \\
 \midrule
 \textbf{Total} & \textbf{1856.1240} & \textbf{516.3833} & \textbf{0.278} & \textbf{306.0975}& \textbf{85.7886} & \textbf{0.280}\\
  \bottomrule
\end{tabular}}
   
    \label{tab:computation time}
\end{table}

The time horizon and the time step are respectively $[0, 6]$ and 0.01 for verification computations on $\bm{N}_1$. The performance of our set-boundary reachability method and the comparison with the entire input set based method are summarized in Table \ref{tab:computation time}. The number $n$ in the `Partition' column denotes that the entire input set is divided into equal $n^2$ subsets, and the boundary of the input set into $4n$ subsets. During each partition, if we fail to verify the safety property, we continue the use of partition operator to obtain smaller subsets for computations. Otherwise, we terminate the partition. In each partition-verification round, the corresponding computation time  from the entire input set based method and our set-boundary reachability method are respectively listed in the `Entire set' and `Boundary' column of Table \ref{tab:computation time}, which shows their ratios as well. We also present the total verification time in the last row, i.e., the sum of computation times in all rounds. These statements also applies to other tables. Our set-boundary reachability method for $\bm{N}_1$ returns `\textcolor{blue}{Safe}' when the boundary of the input set is partitioned into $16$ equal subsets. However, the safety property for $\bm{N}_1$ is  verified until the entire input set is partitioned into $49$ equal subsets. The computed output reachable sets for $\bm{N}_1$ are displayed in Fig. \ref{spiral_2} and \ref{spiral_3}. It is observed from Table \ref{tab:computation time} that the computation time from the set-boundary reachability method is reduced by 72.2\%, compared to the entire input set based method.

For network $\bm{N}_2$, when the time horizon is $[0.0, 1.1]$ and the time step is 0.01, we first use the entire set based method to verify the safety property. Unfortunately, it fails. Consequently, we have to partition the entire input set into small subsets for the verification. In this case, the boundary of the input set is a good choice and thus we utilize the set-boundary reachability method for the safety verification. Without further splitting, our set-boundary reachability method returns  `\textcolor{blue}{Safe}' with the time consumption of around 7400.5699 seconds. %The corresponding computed output reachable sets for $\bm{N}_2$ are displayed in Fig. \ref{FPA}. Fig. \ref{FPA_conc} shows the reachable sets at the time instant $t=1.1$.

\begin{figure}[htbp]
\centering
\subfigure[$\epsilon=0.375$, \textcolor{blue}{Safe}, \textcolor{red}{Safe}]{
\begin{minipage}[t]{0.33\linewidth}
\centering
\includegraphics[width=1.6in]{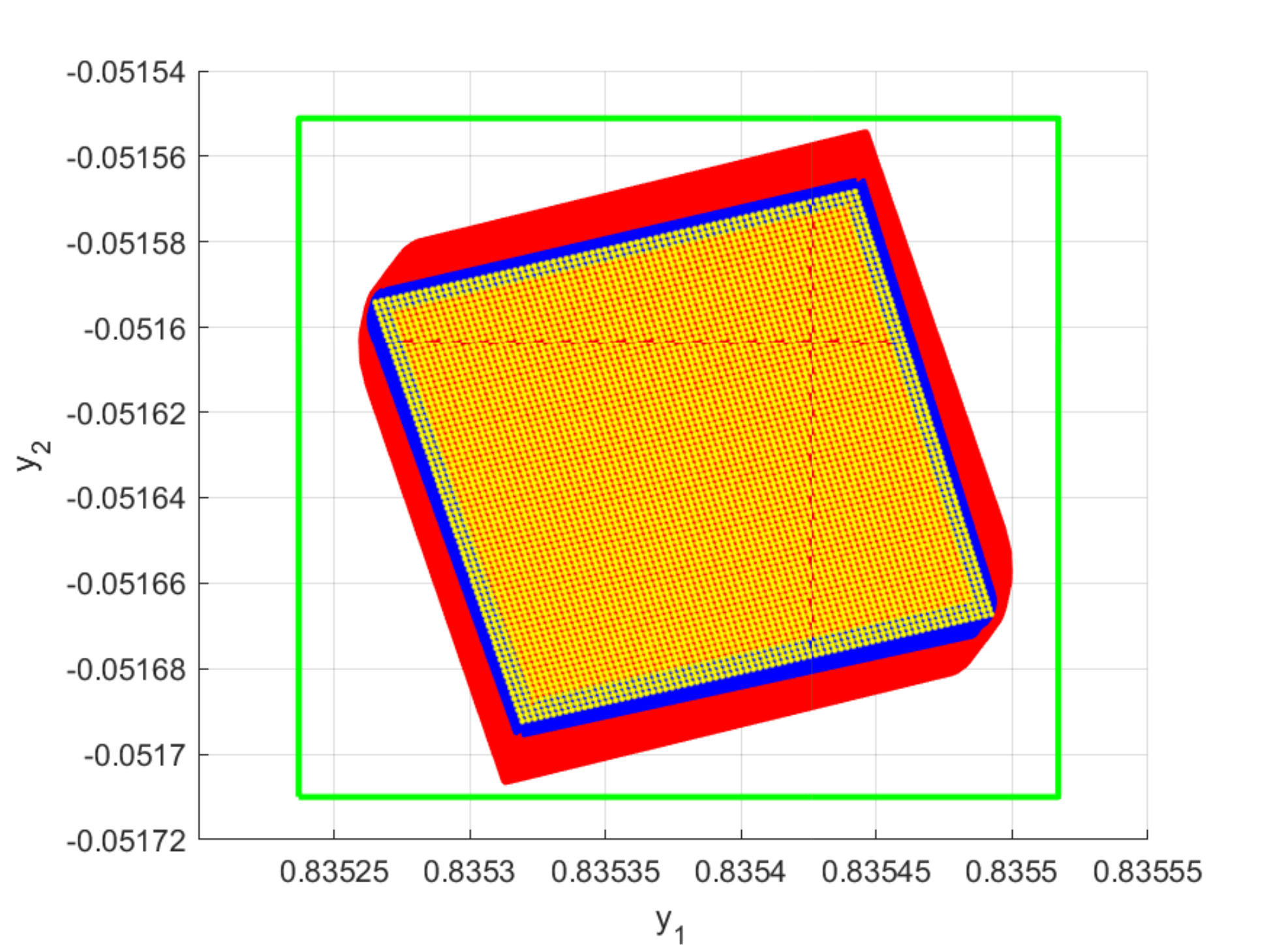}
%\caption{fig1}
\label{0.125}
\end{minipage}%
}%
\subfigure[$\epsilon=0.400$,  \textcolor{blue}{Safe}, \textcolor{red}{Unknown}]{
\begin{minipage}[t]{0.33\linewidth}
\centering
\includegraphics[width=1.6in]{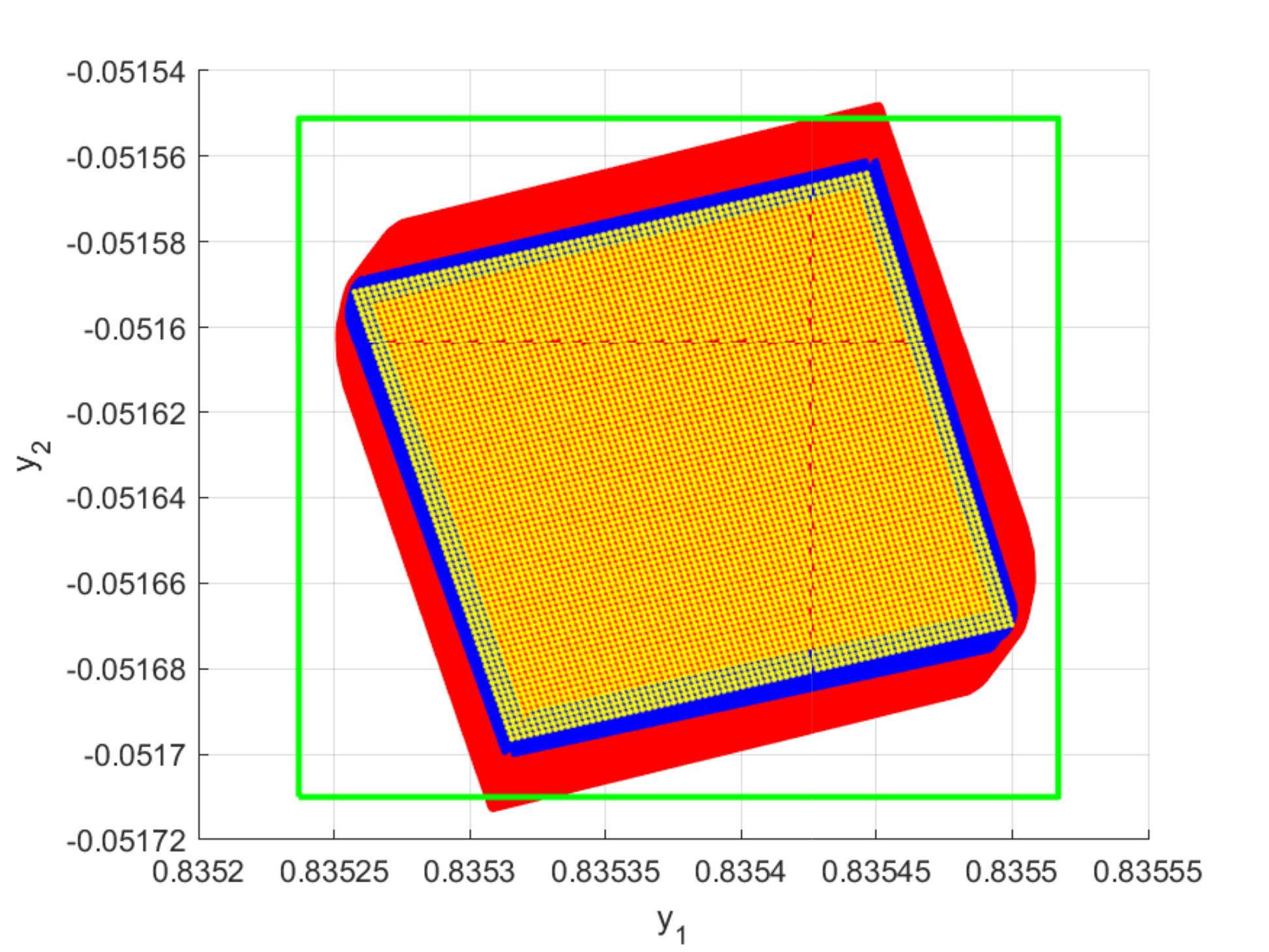}
%\caption{fig2}
\label{0.25}
\end{minipage}%
}%
\subfigure[$\epsilon=0.425$,  \textcolor{blue}{Safe}, \textcolor{red}{Unknown}]{
\begin{minipage}[t]{0.33\linewidth}
\centering
\includegraphics[width=1.6in]{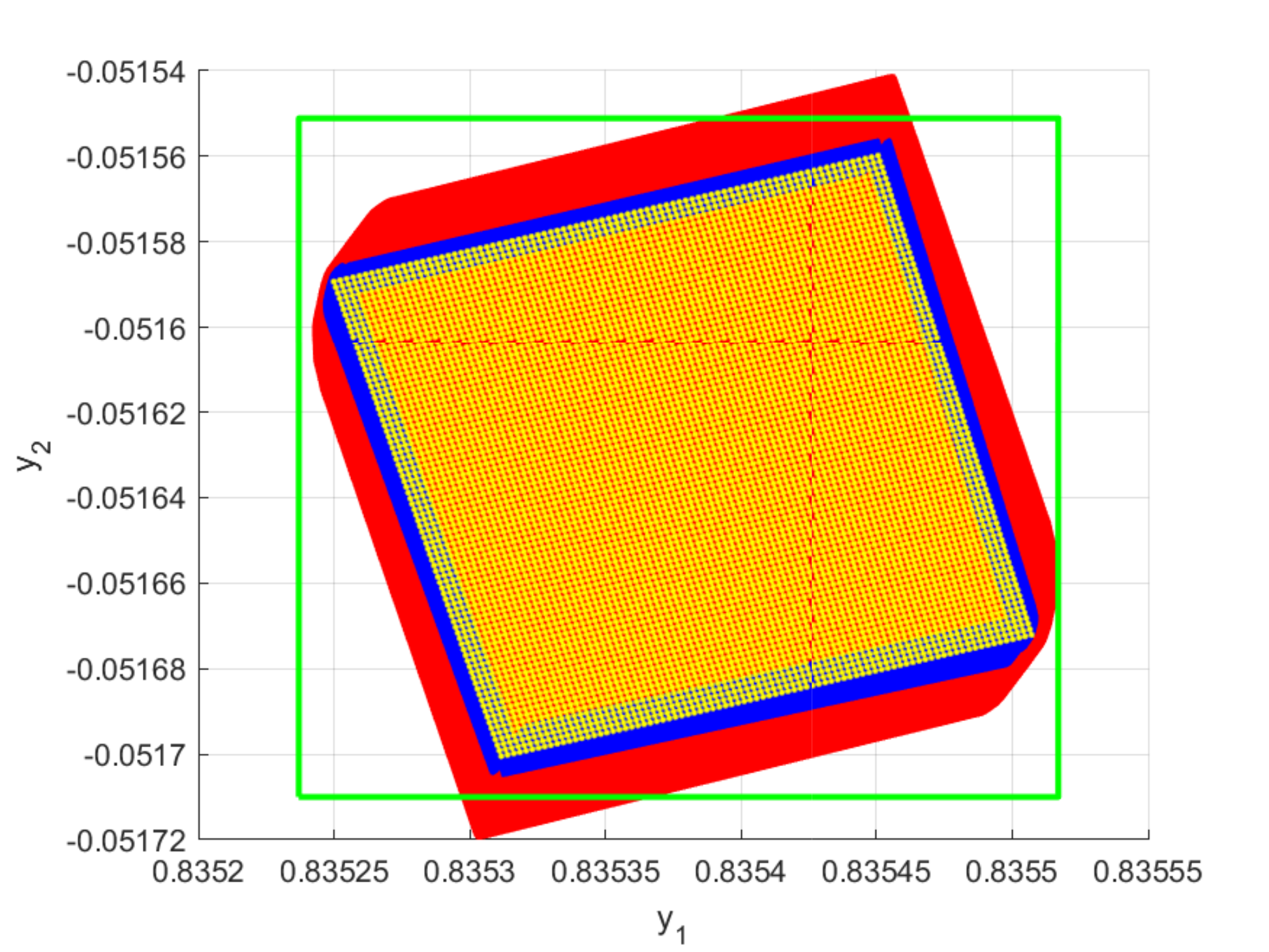}
%\caption{fig1}
\label{0.375}
\end{minipage}%
}%
\\
\subfigure[$\epsilon=0.450$,  \textcolor{blue}{Safe}, \textcolor{red}{Unknown}]{
\begin{minipage}[t]{0.33\linewidth}
\centering
\includegraphics[width=1.6in]{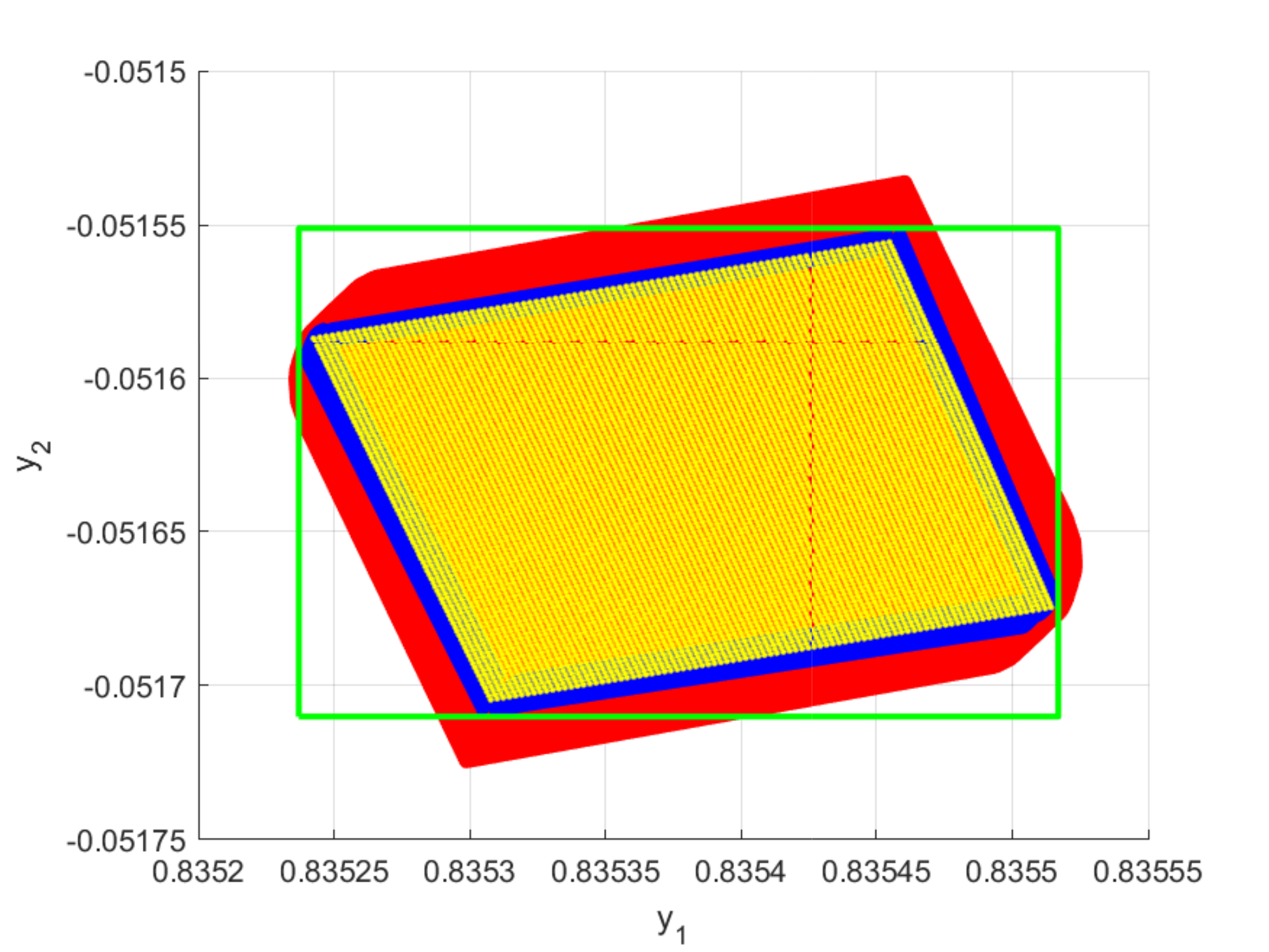}
%\caption{fig2}
\label{0.425}
\end{minipage}%
}%
\subfigure[\textcolor{blue}{$5\times4$} Vs. \textcolor{red}{$7^2$}. \textcolor{blue}{Safe}, \textcolor{red}{Unknown}]{
\begin{minipage}[t]{0.34\linewidth}
\centering
\includegraphics[width=1.6in]{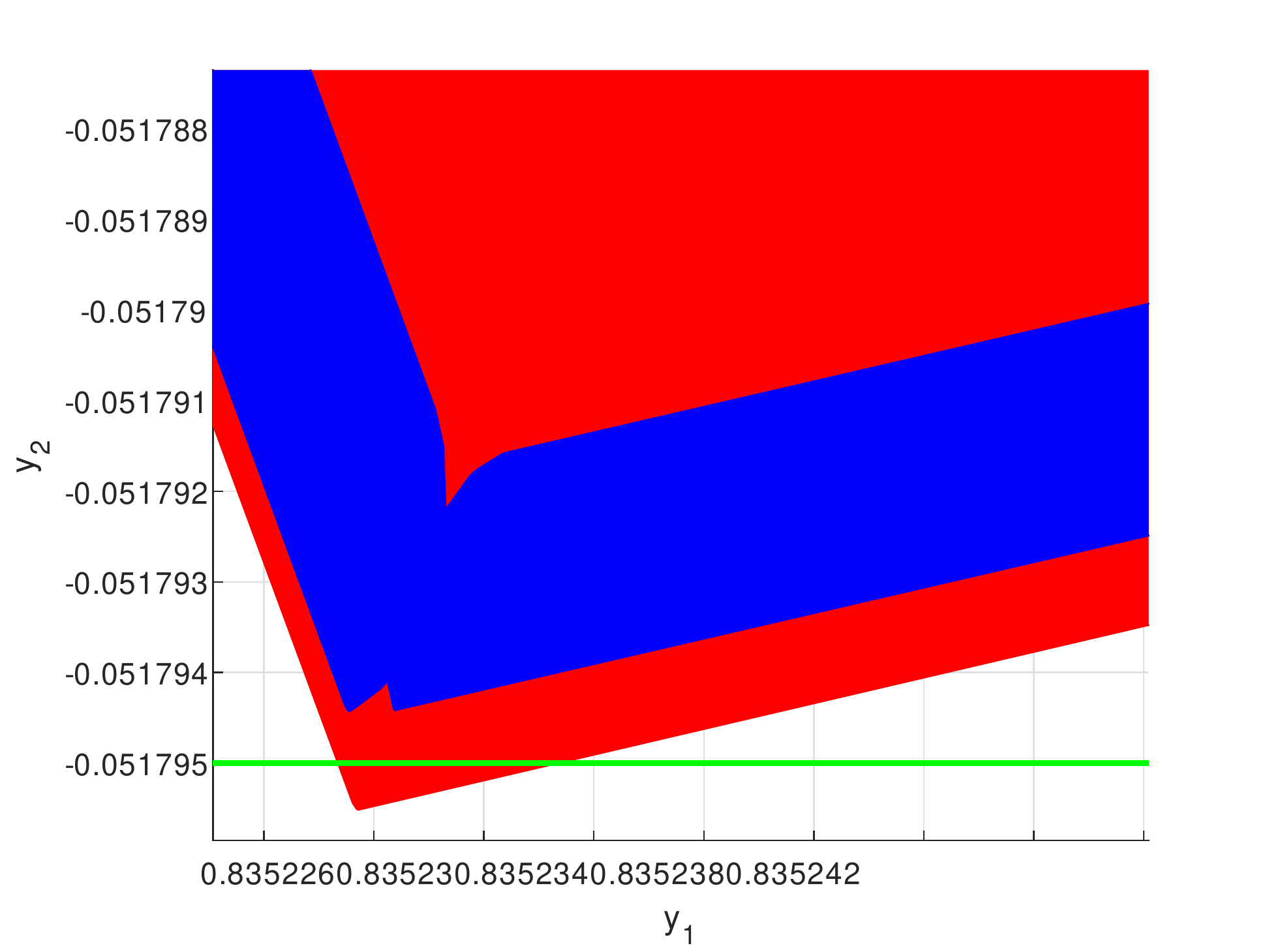}
%\caption{fig2}
\label{-11_2}
\end{minipage}%
}%
\subfigure[\textcolor{blue}{$5\times4$} Vs. \textcolor{red}{$8^2$}. \textcolor{blue}{Safe}, \textcolor{red}{Safe}]{
\begin{minipage}[t]{0.33\linewidth}
\centering
\includegraphics[width=1.6in]{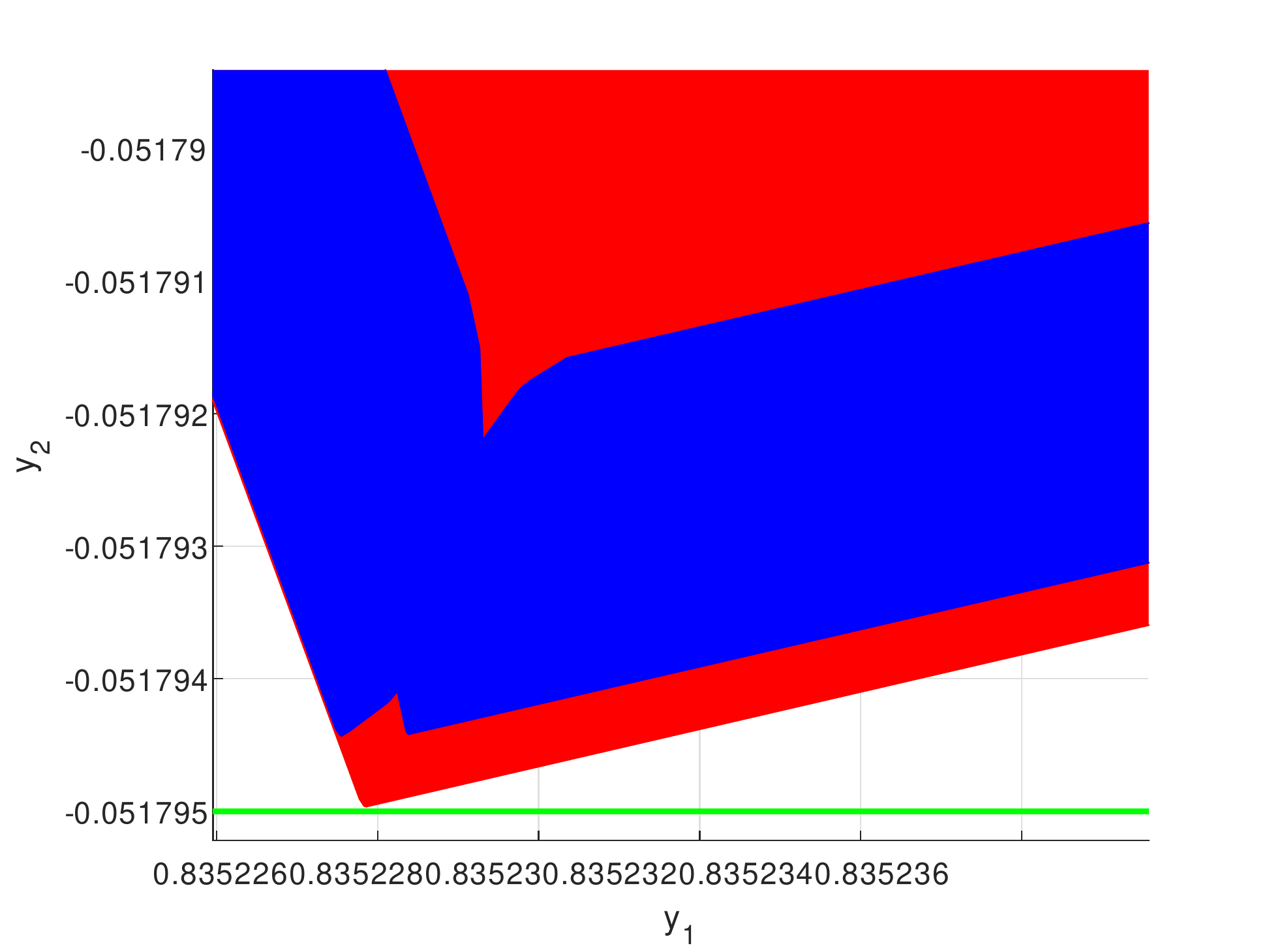}
%\caption{fig2}
\label{-11_3}
\end{minipage}
}%%
\centering
\caption{Safety verification on $\bm{N}_3$.
\textcolor{blue}{$\Omega(\partial \mathcal{X}_{in})$}; \textcolor{red}{$\Omega(\mathcal{X}_{in})$}; \textcolor{green}{$\partial \mathcal{X}_{s}$}}
\label{2-dim inn}

\end{figure}

\noindent\textbf{Feedforward NNs.} Rather than considering neural ODEs, we instead take more general invertible NNs, such as feedforward NNs,  into account.  The invertibility of  NNs used here, i.e., $\bm{N}_3$ and $\bm{N}_4$,  are assured by their Jacobian determinant not being zero. The NN $\bm{N}_3$ is fully connected with $\mathtt{Sigmoid}$ activation functions, having an input/output layer with dimension 2 and  5 hidden layers with size 10. The NN $\bm{N}_4$ is similar to $\bm{N}_3$, 
consisting 5 hidden layers with size 5 and its input/output dimensions are 3. The results of safety verification of $\bm{N}_3$ and $\bm{N}_4$ are demonstrated in Fig. \ref{2-dim inn} and \ref{3-dim inn1}. The input sets in Fig. \ref{2-dim inn} are  $[-0.375,0.375]^2$, $[-0.4,0.4]^2$, $[-0.425,0.425]^2$, $[-0.45,0.45]^2$ and $[-1.0,1.0]^2$ (Fig. \ref{-11_2}-\ref{-11_3}.) respectively  and those of Fig. \ref{3-dim inn1} are $[-0.45,0.45]^3$, $[-0.475,0.475]^3$,  $[-0.5,0.5]^3$. Their safe sets $\mathcal{X}_{s}$ are respectively $ [0.835237,0.835517]\times [-0.0517100,-0.0515511]$ (Fig. \ref{0.125}-\ref{0.425}), $ [0.835078,0.83567]\times [-0.051795,-0.05146]$ (Fig. \ref{-11_2}-\ref{-11_3}) and $[-0.5391325, -0.5391025] \times [-0.9921175,-0.99209530]\times[-0.348392,-0.3483625]$ (Fig. \ref{3-dim inn1}), whose boundaries are shown in green color.  The over-approximate reachability analysis is implemented with polynominal zonotope domains \cite{kochdumper2022open}, which is a recently proposed  NN verification tool via propagating polynominal zonotopes through networks and is termed OCNNV for brevity herein.

\begin{figure}[htbp]
\centering

\subfigure[$y_1-y_2$]{
\begin{minipage}[t]{0.32\linewidth}
\centering
\includegraphics[width=1.6in]{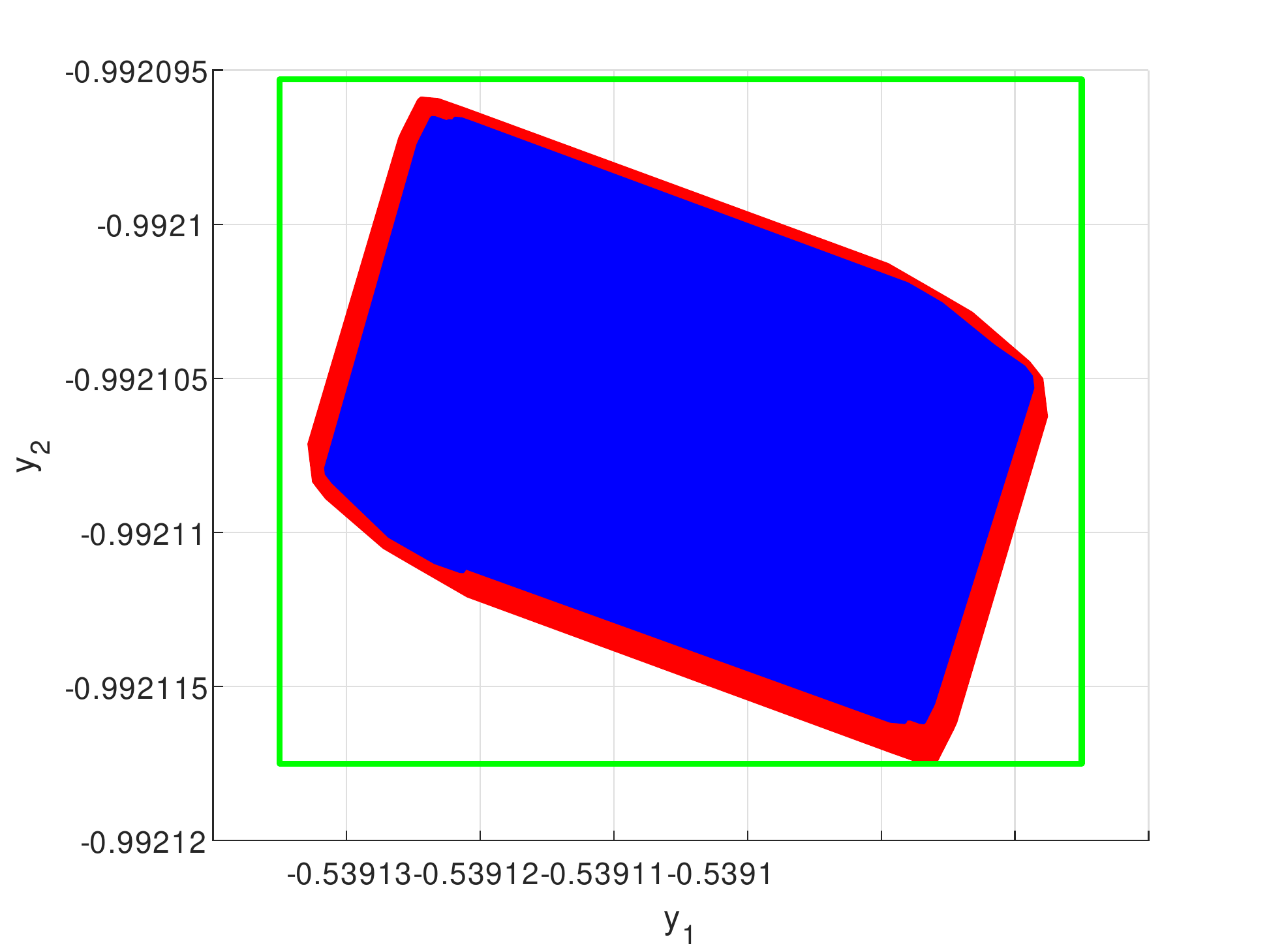}
%\caption{fig1}
\end{minipage}%
}%
\subfigure[$y_1-y_3$]{
\begin{minipage}[t]{0.32\linewidth}
\centering
\includegraphics[width=1.6in]{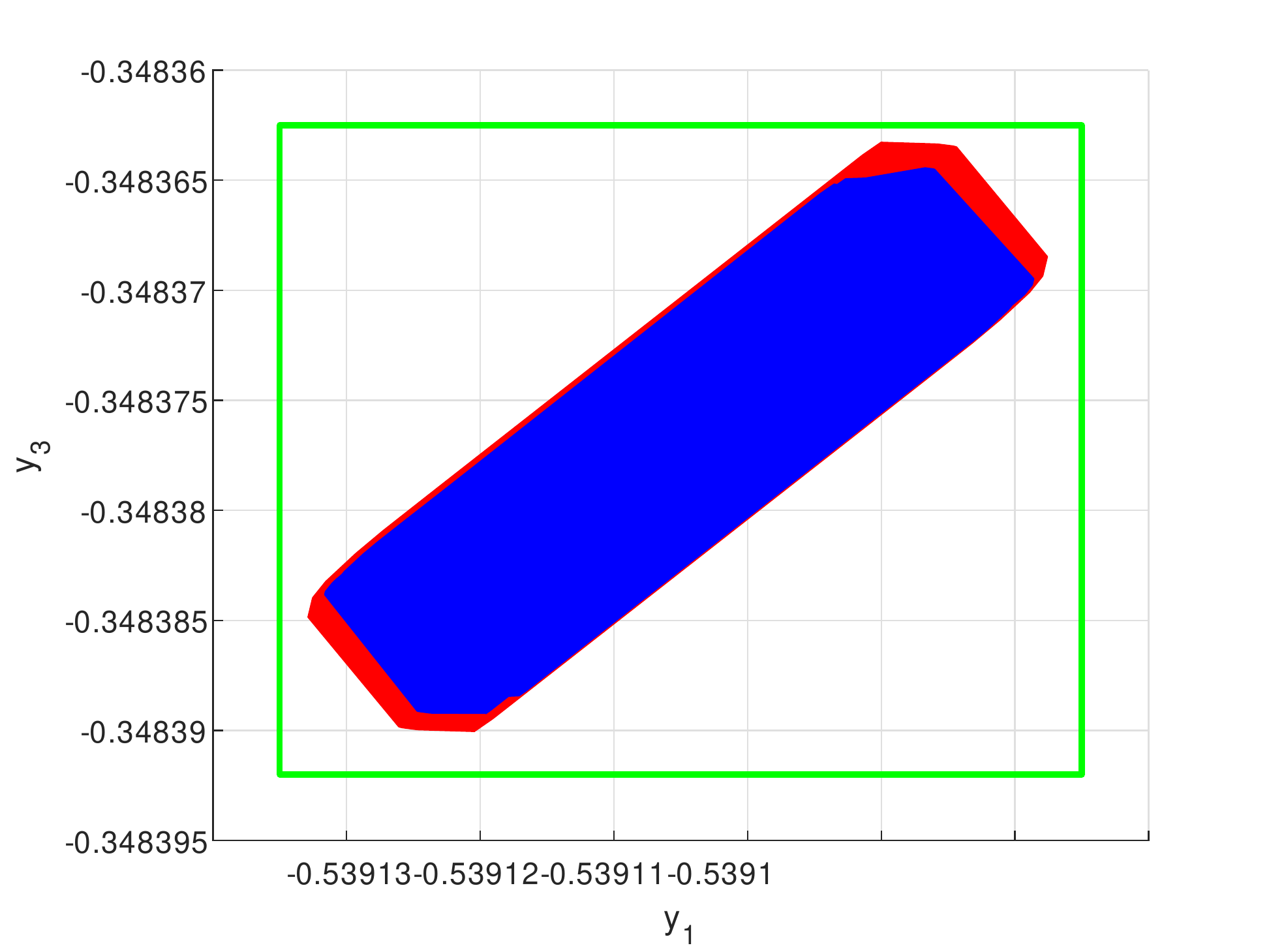}
%\caption{fig2}
\end{minipage}%
}%
\subfigure[$y_2-y_3$]{
\begin{minipage}[t]{0.32\linewidth}
\centering
\includegraphics[width=1.6in]{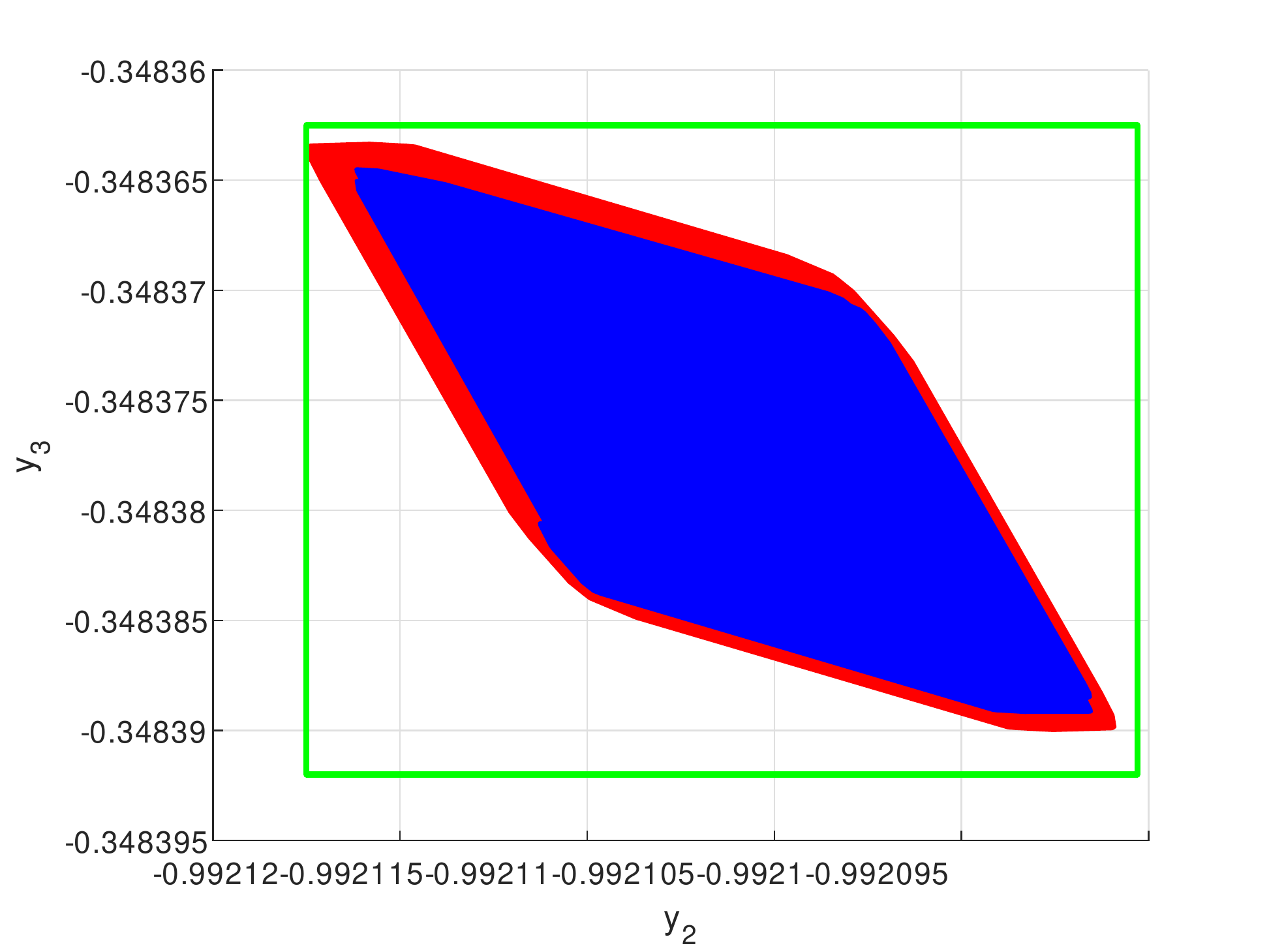}
%\caption{fig2}
\end{minipage}
}%%
\\
$\epsilon=0.450$, \textcolor{blue}{1 $\times$ 6} Vs. \textcolor{red}{1},\textcolor{blue}{Safe}, \textcolor{red}{Safe}
\\
\subfigure[$y_1-y_2$]{
\begin{minipage}[t]{0.32\linewidth}
\centering
\includegraphics[width=1.6in]{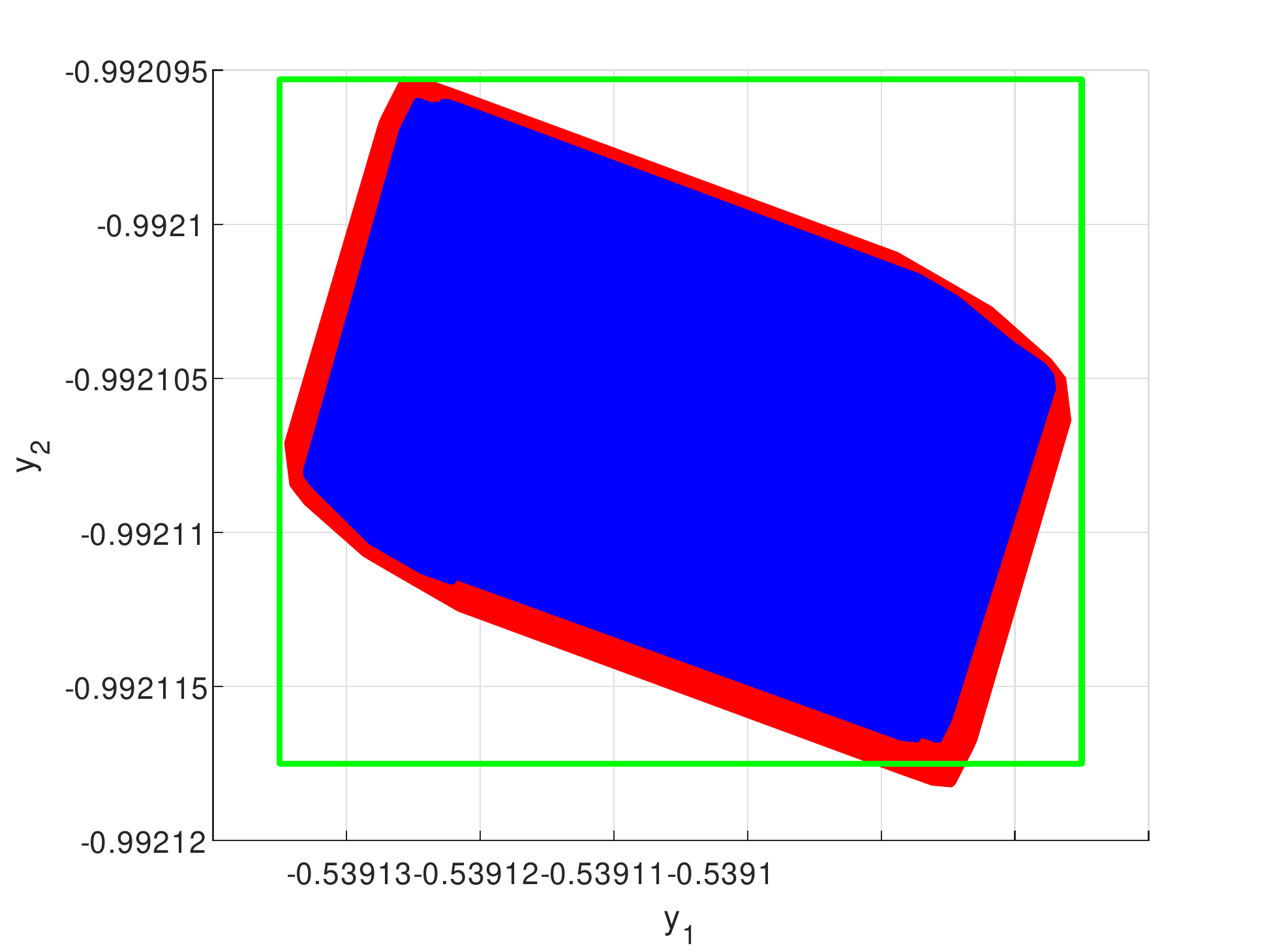}
\label{0.300-11}
\end{minipage}%
}%
\subfigure[$y_1-y_3$]{
\begin{minipage}[t]{0.32\linewidth}
\centering
\includegraphics[width=1.6in]{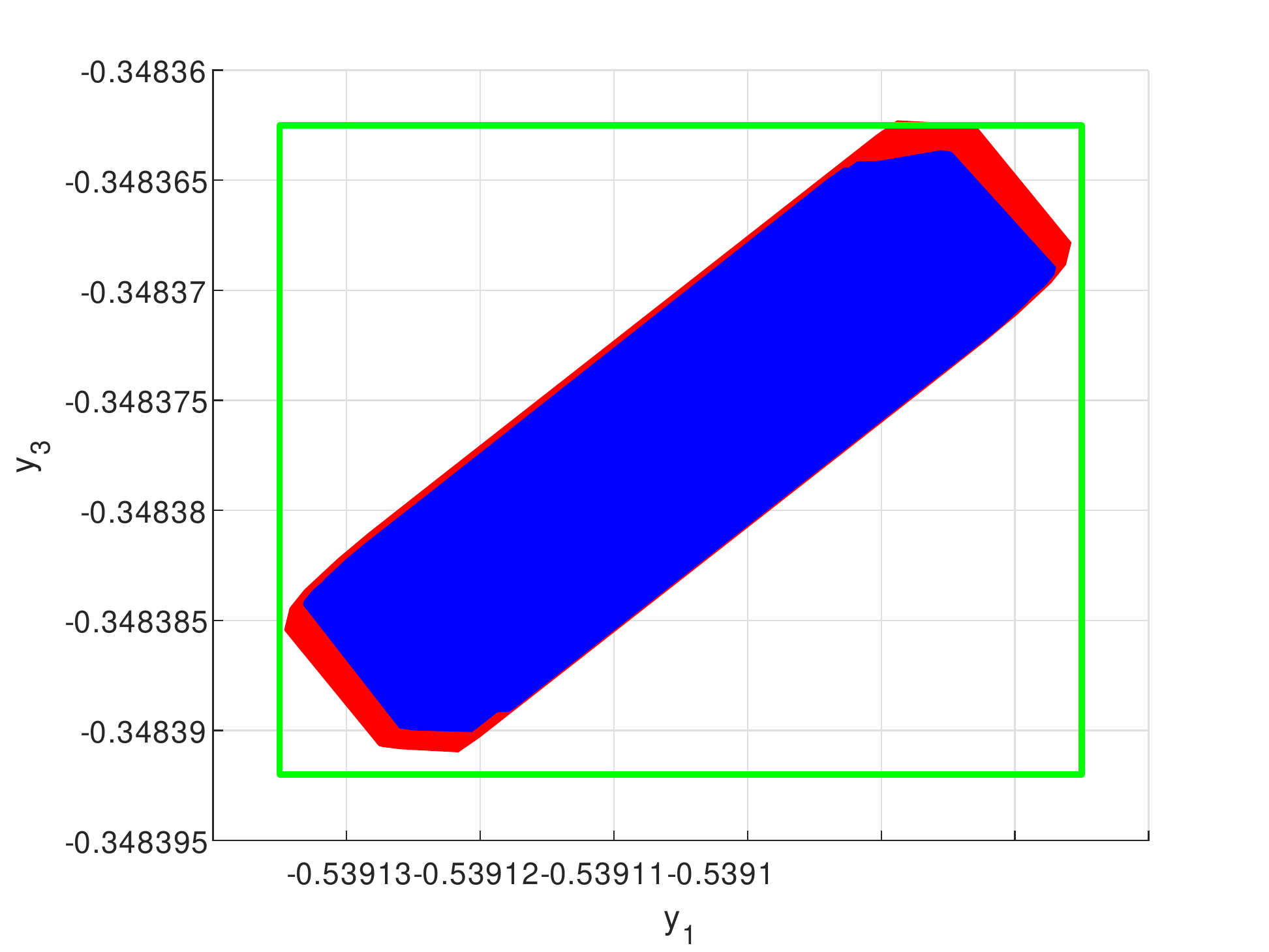}
\label{0.300-21}
\end{minipage}%
}%
\subfigure[$y_2-y_3$.]{
\begin{minipage}[t]{0.32\linewidth}
\centering
\includegraphics[width=1.6in]{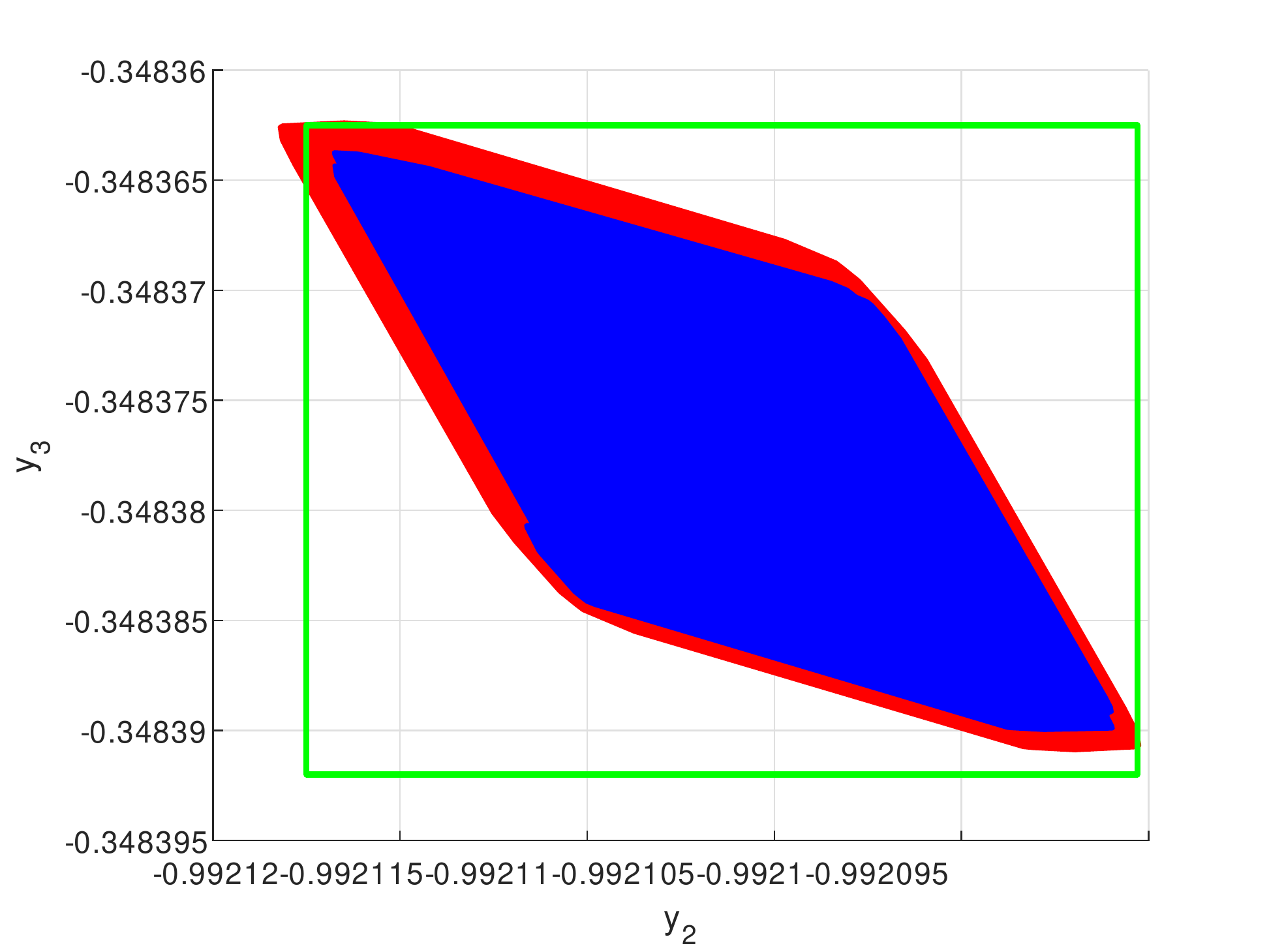}
\label{0.300-31}
\end{minipage}
}%%
\\
$\epsilon=0.475$, \textcolor{blue}{1 $\times$ 6} Vs. \textcolor{red}{1}, \textcolor{blue}{Safe}, \textcolor{red}{Unknown}
\\
\subfigure[$y_1-y_2$]{
\begin{minipage}[t]{0.3\linewidth}
\centering
\includegraphics[width=1.6in]{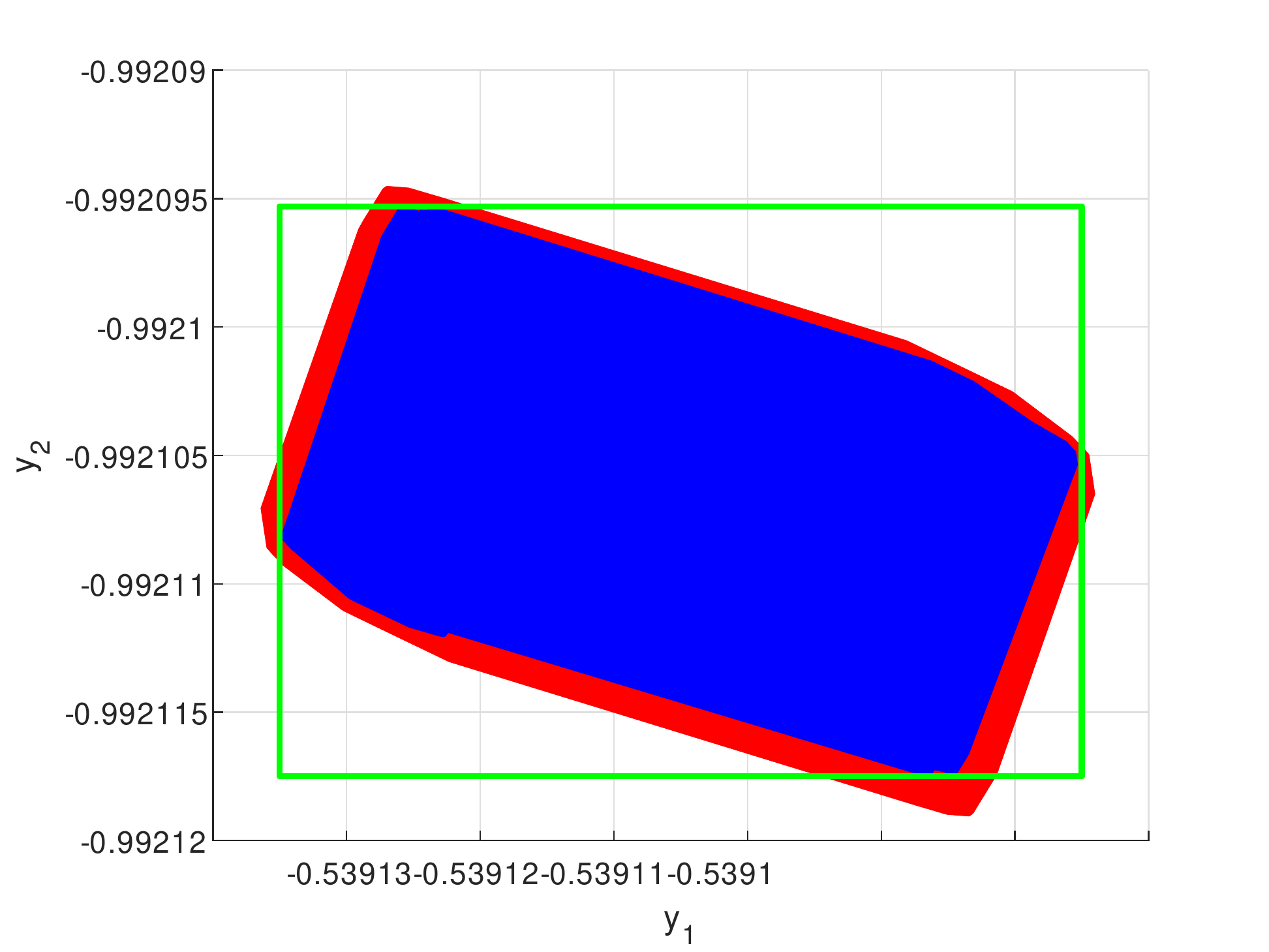}
\label{0.375-11}
\end{minipage}%
}%
\subfigure[$y_1-y_3$]{
\begin{minipage}[t]{0.3\linewidth}
\centering
\includegraphics[width=1.6in]{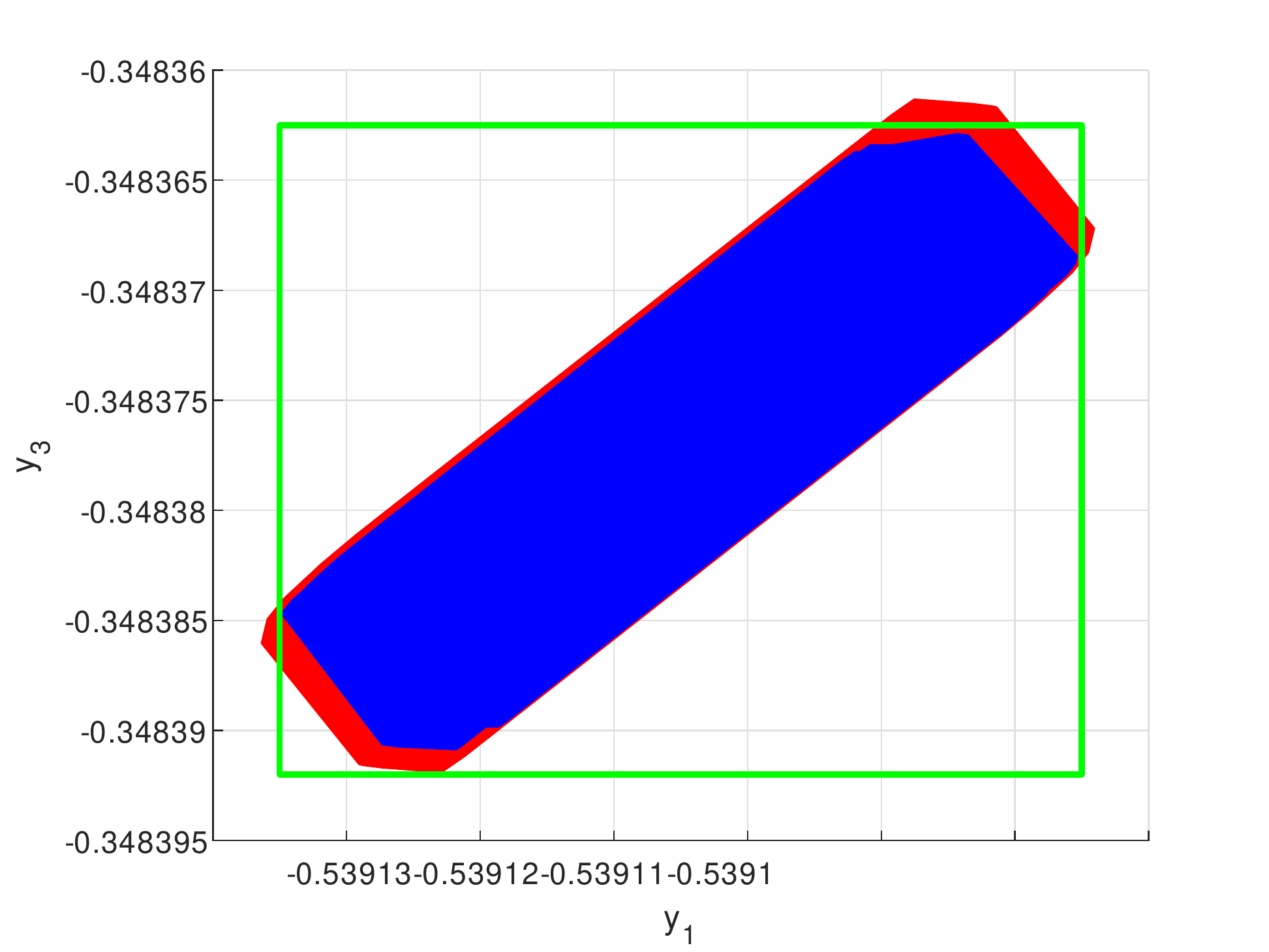}
\label{0.375-21}
\end{minipage}%
}%
\subfigure[$y_2-y_3$]{
\begin{minipage}[t]{0.32\linewidth}
\centering
\includegraphics[width=1.6in]{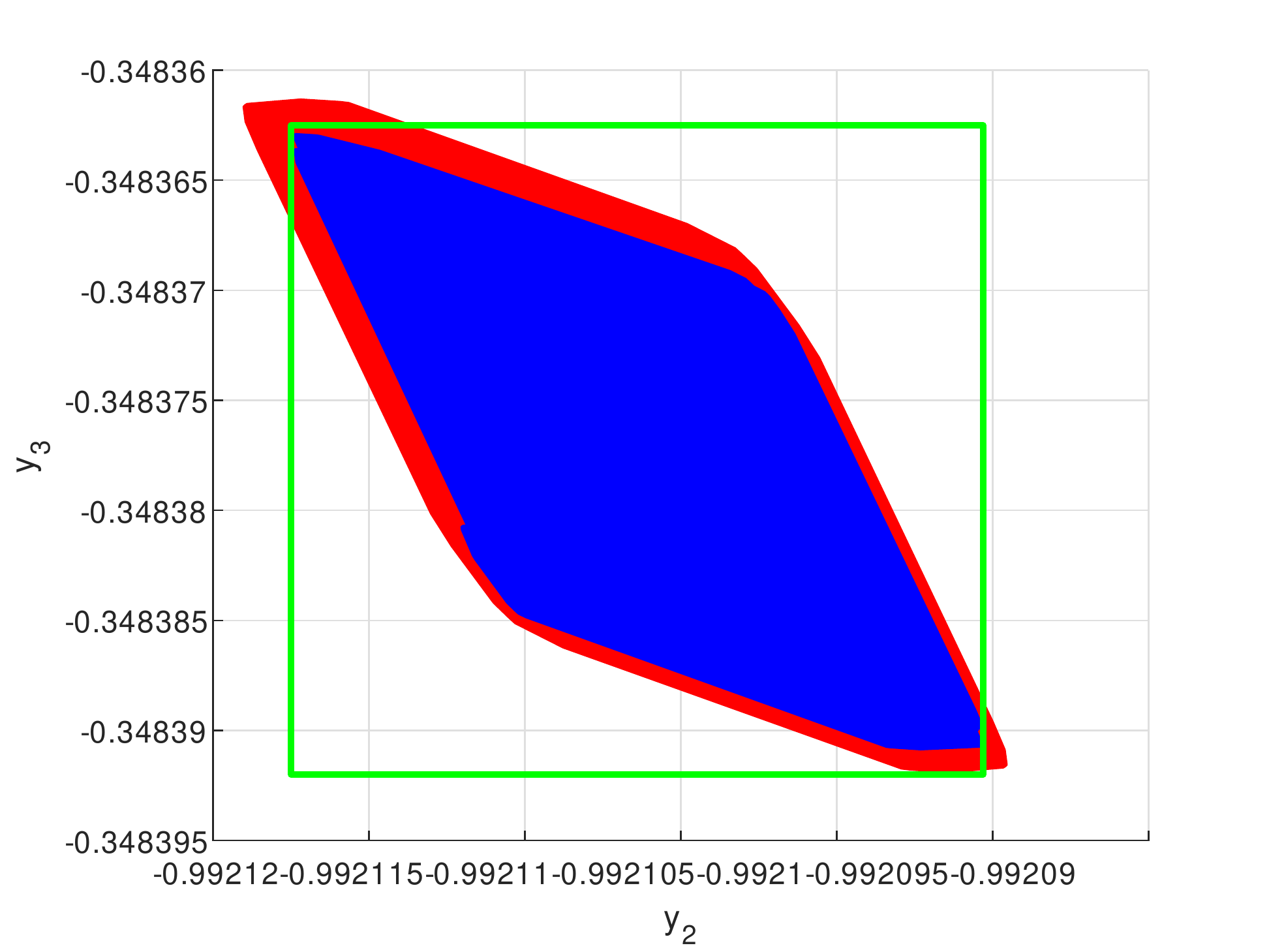}
\label{0.375-31}
\end{minipage}
}%%
\\
$\epsilon=0.500$,\textcolor{blue}{1$\times$ 6} Vs. \textcolor{red}{1},  \textcolor{blue}{Safe}, \textcolor{red}{Unknown}
\centering
\caption{Safety verification on $\bm{N}_4$.
\textcolor{blue}{$\Omega(\partial \mathcal{X}_{in})$}; \textcolor{red}{$\Omega(\mathcal{X}_{in})$}; \textcolor{green}{$\partial \mathcal{X}_{s}$}}
\label{3-dim inn1111}
\end{figure}

\begin{figure}[htbp]
\centering

\subfigure[$y_1-y_2$]{
\begin{minipage}[t]{0.3\linewidth}
\centering
\includegraphics[width=1.6in]{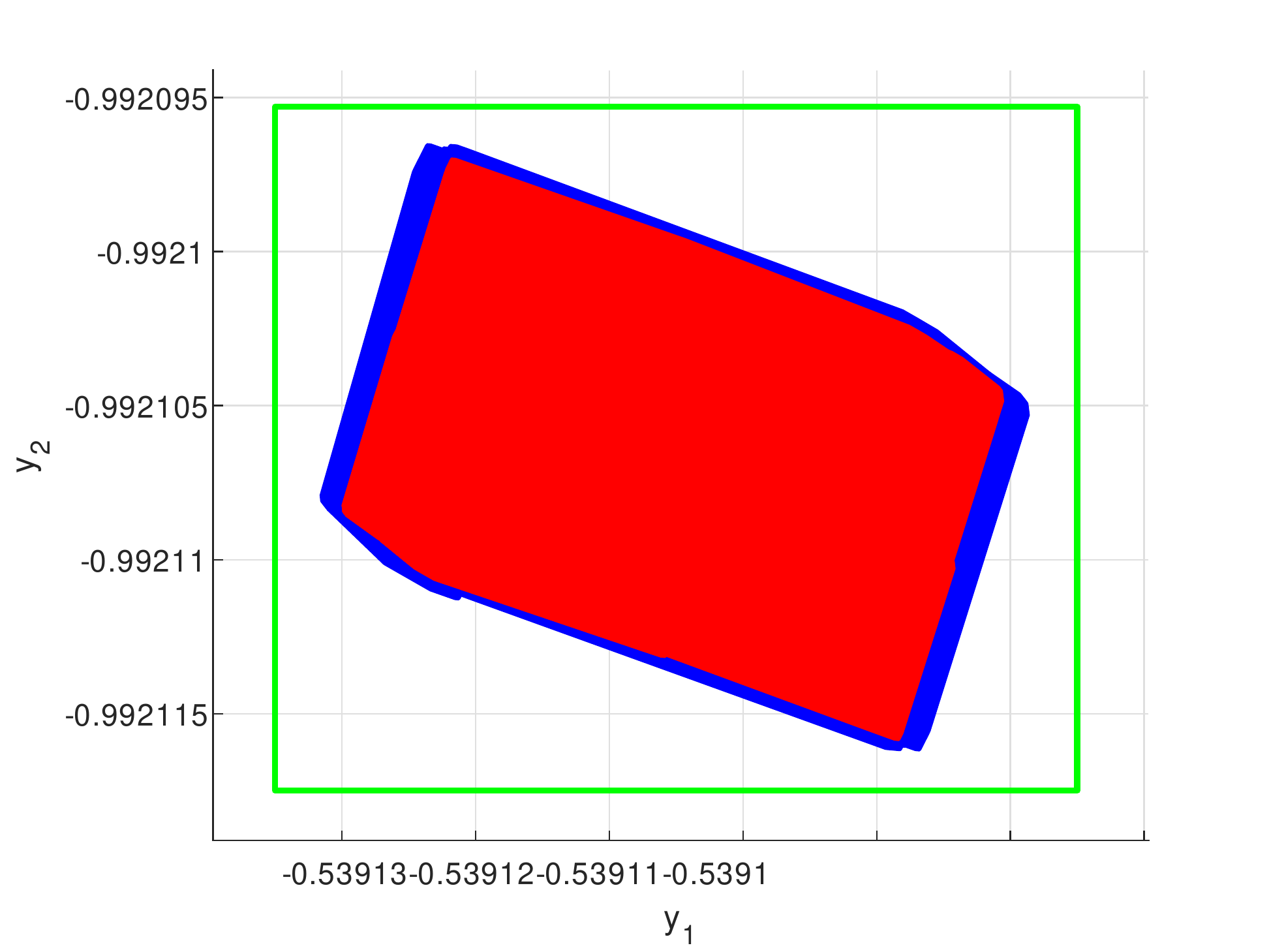}
%\caption{fig1}
\end{minipage}%
}%
\subfigure[$y_1-y_3$]{
\begin{minipage}[t]{0.3\linewidth}
\centering
\includegraphics[width=1.6in]{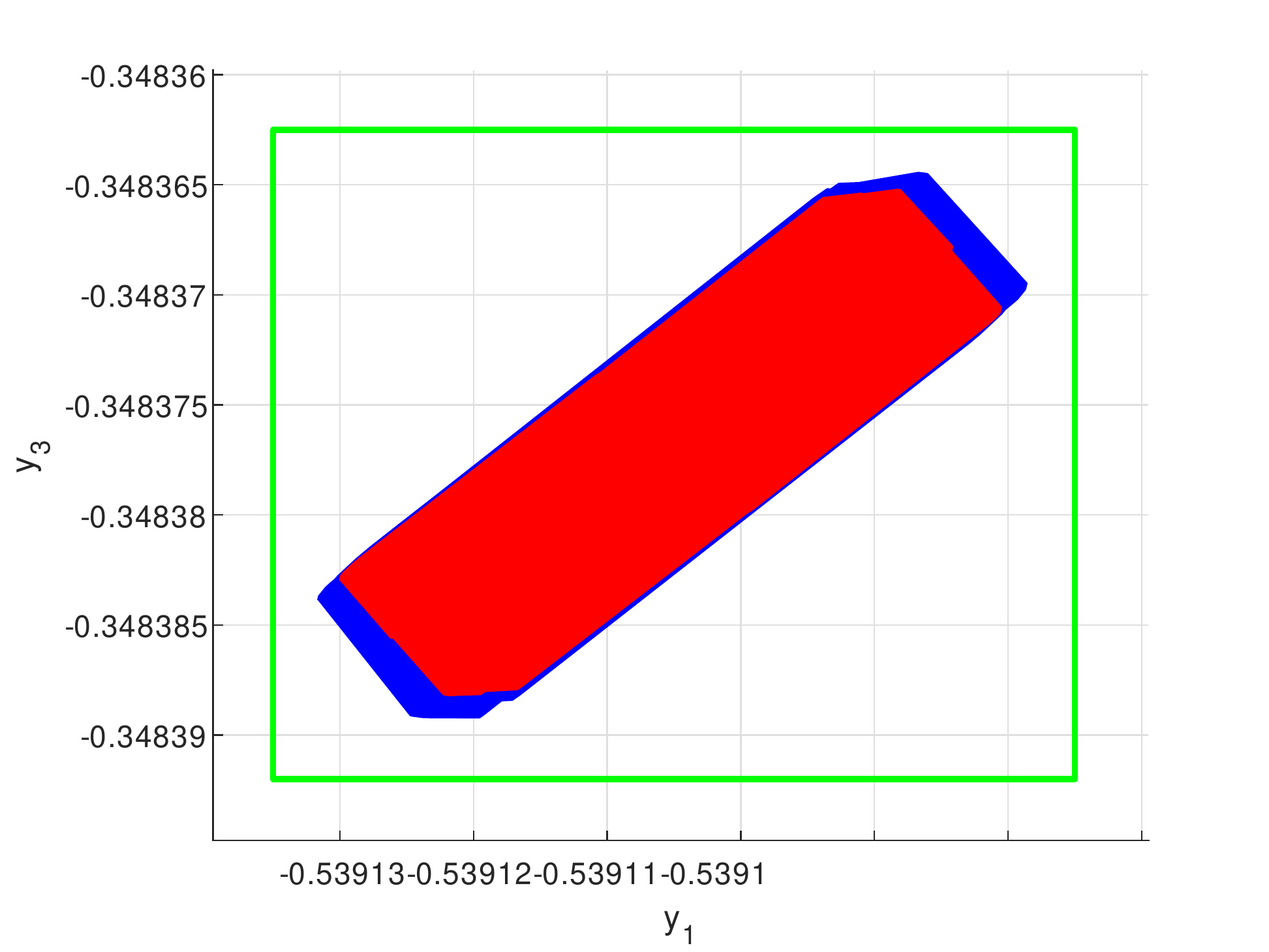}
%\caption{fig2}
\end{minipage}%
}%
\subfigure[$y_2-y_3$]{
\begin{minipage}[t]{0.3\linewidth}
\centering
\includegraphics[width=1.6in]{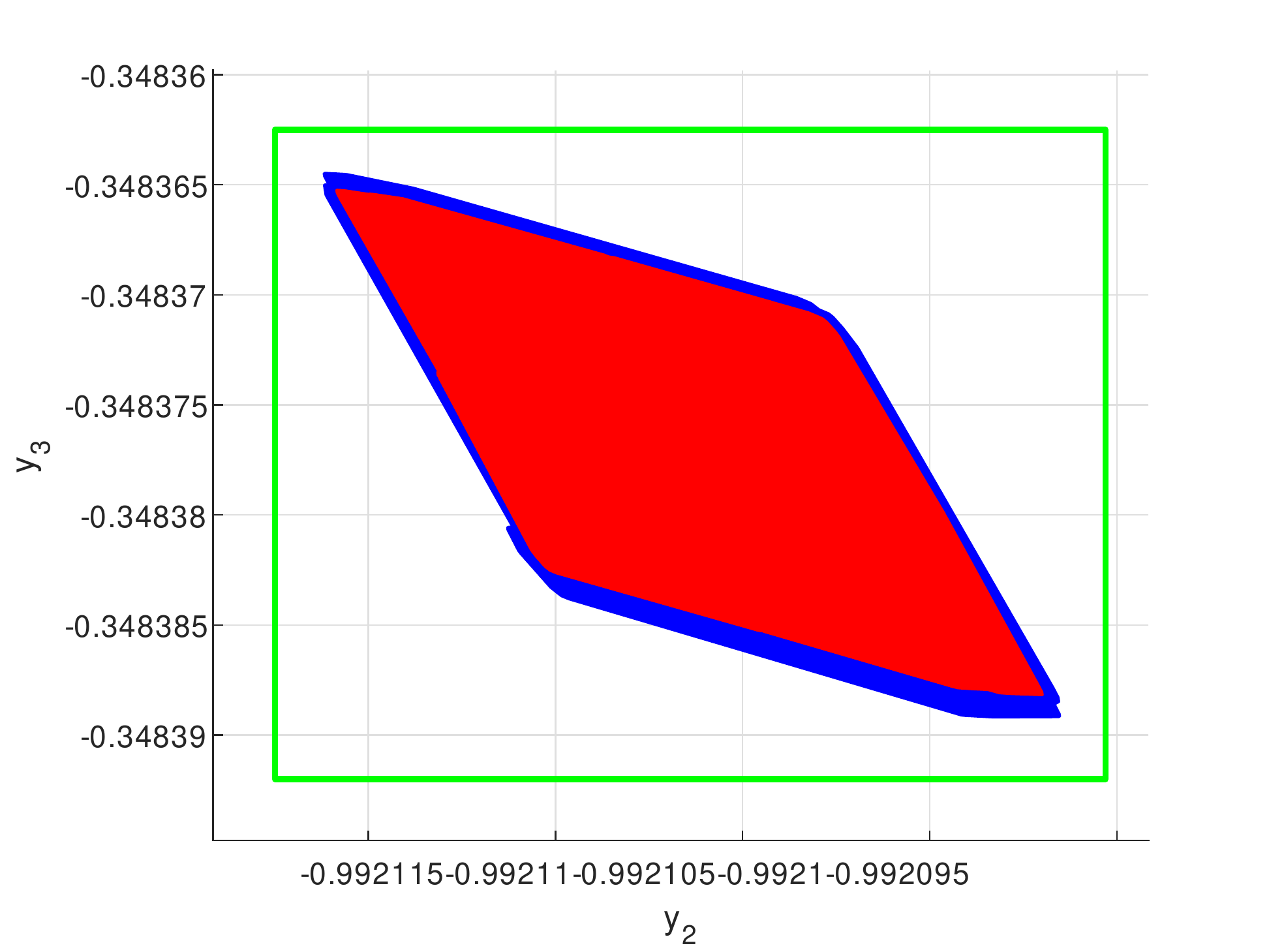}
%\caption{fig2}
\end{minipage}
}%%
\\
$\epsilon=0.450$, \textcolor{blue}{$1\times6$} Vs. \textcolor{red}{$2^3$}, \textcolor{blue}{Safe}, \textcolor{red}{Safe}
\\
\subfigure[$y_1-y_2$]{
\begin{minipage}[t]{0.3\linewidth}
\centering
\includegraphics[width=1.6in]{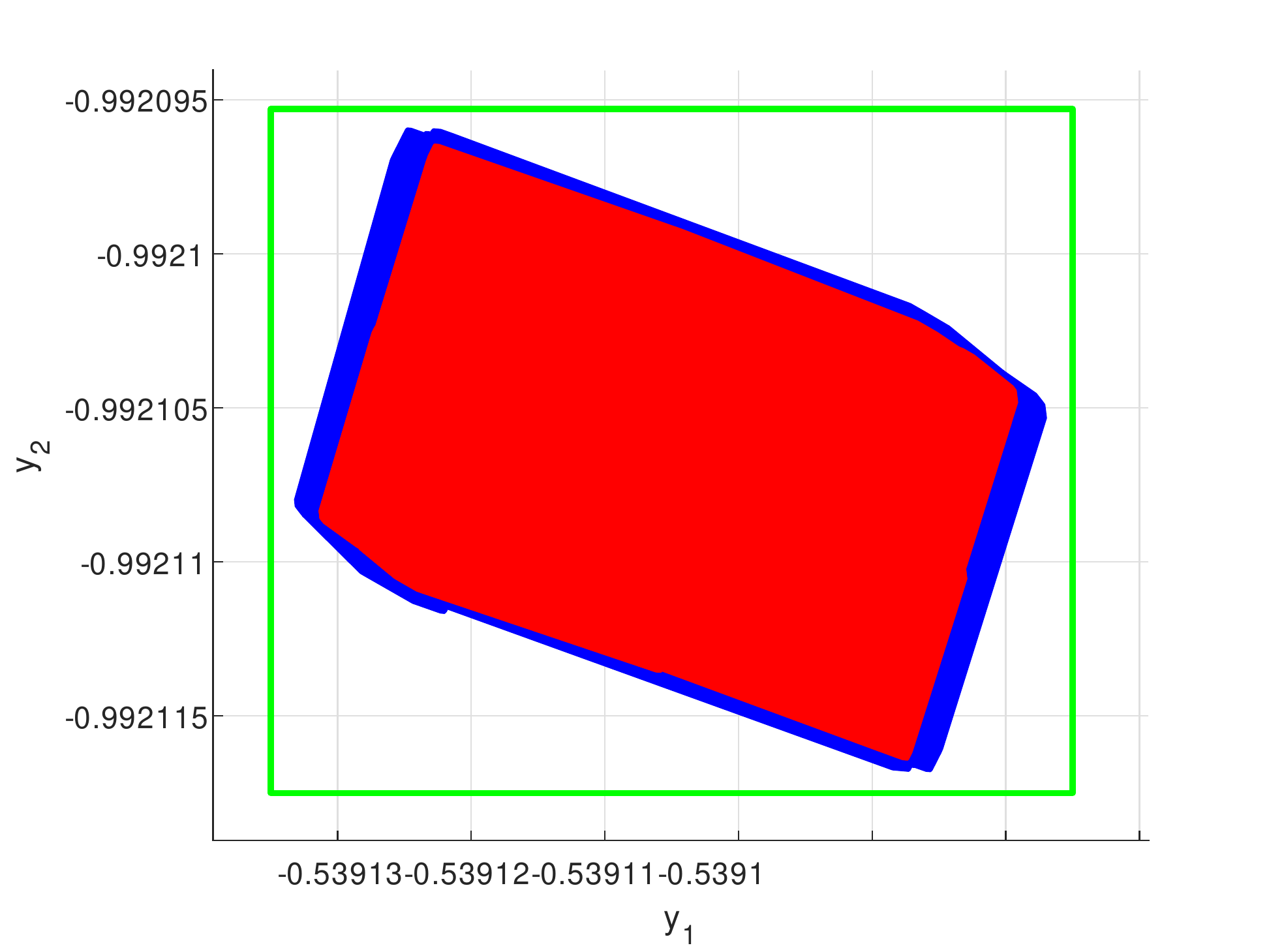}
\label{0.300-1}
\end{minipage}%
}%
\subfigure[$y_1-y_3$]{
\begin{minipage}[t]{0.3\linewidth}
\centering
\includegraphics[width=1.6in]{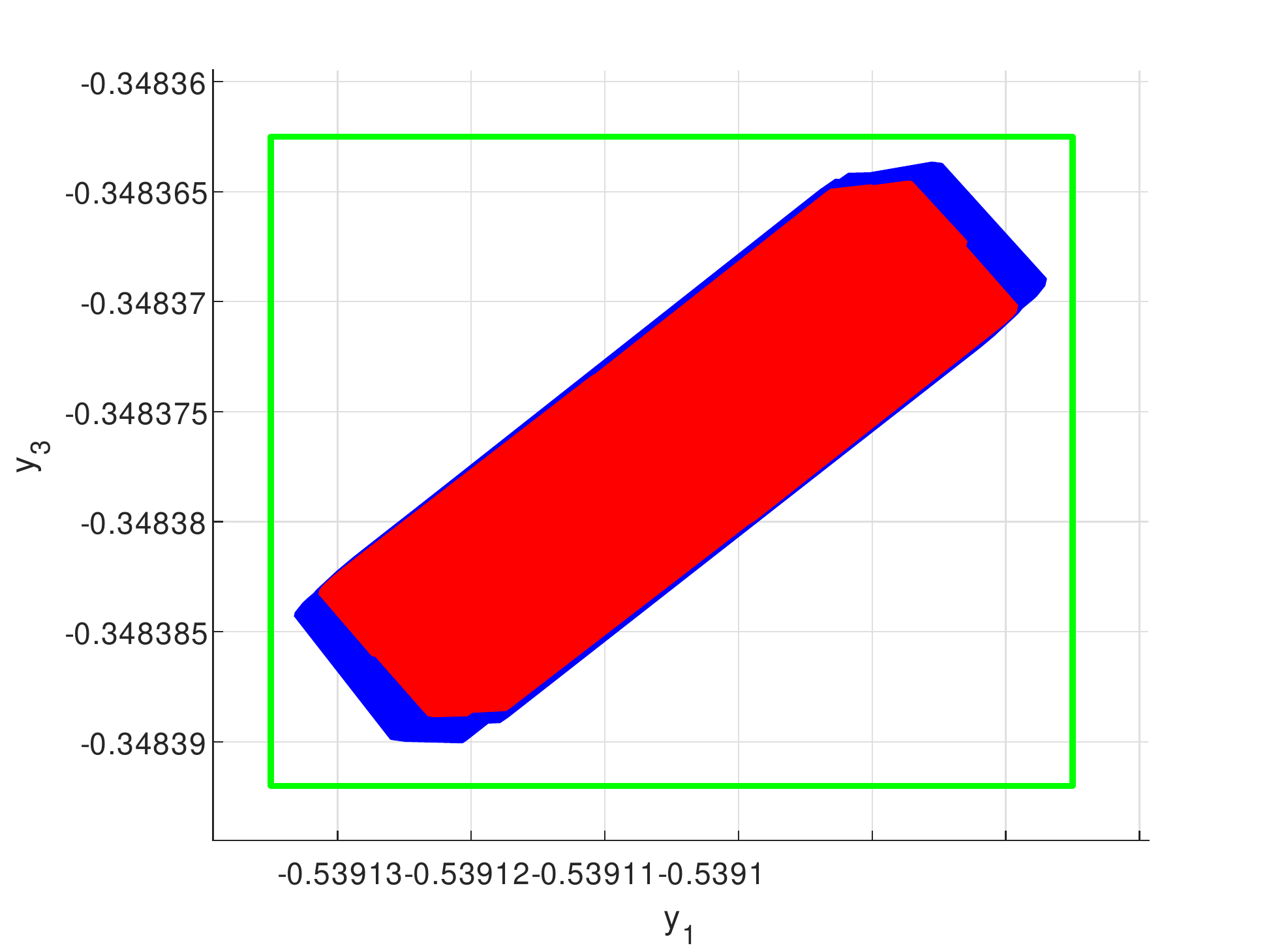}
\label{0.300-2}
\end{minipage}%
}%
\subfigure[$y_2-y_3$.]{
\begin{minipage}[t]{0.3\linewidth}
\centering
\includegraphics[width=1.6in]{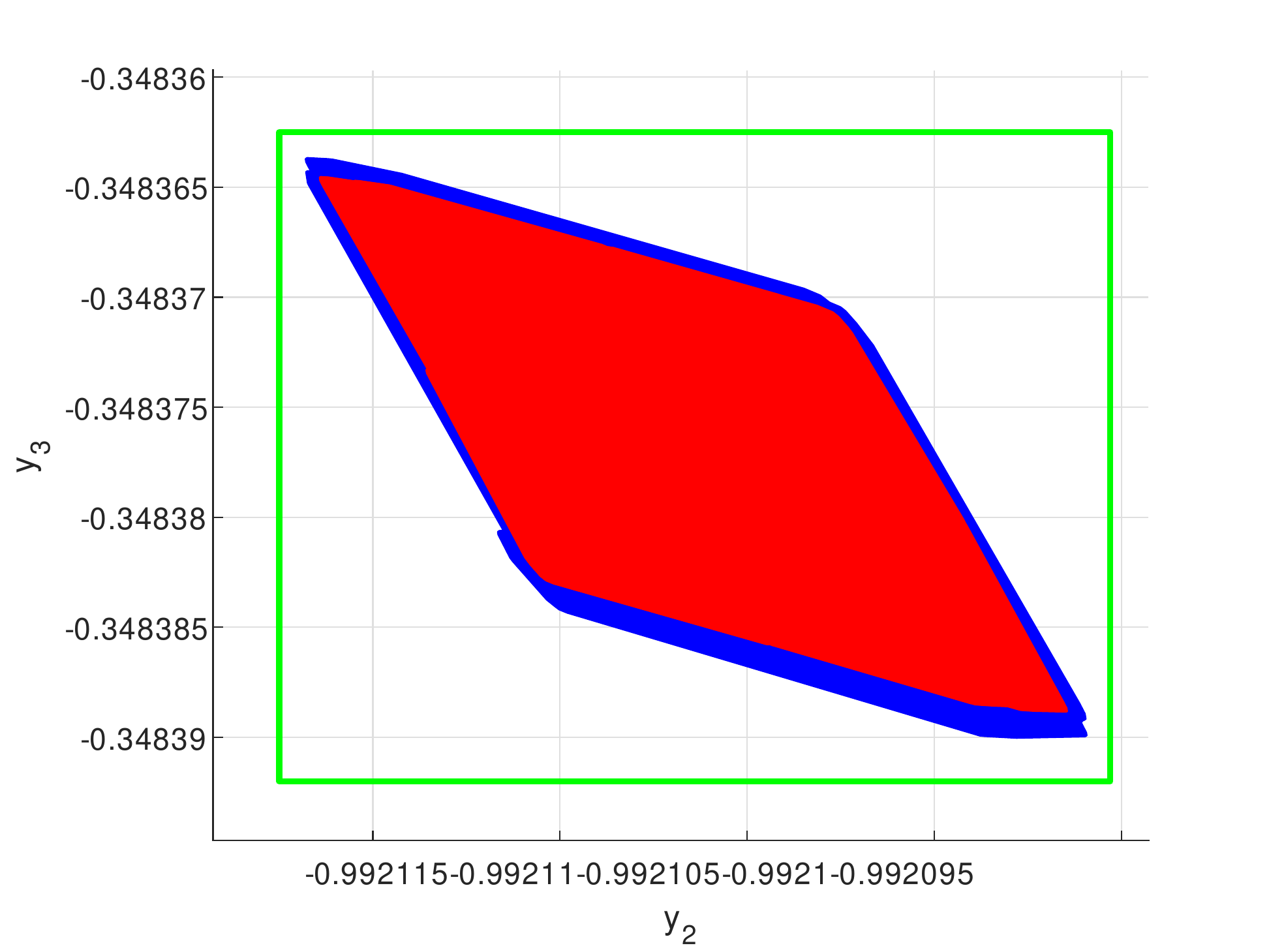}
\label{0.300-3}
\end{minipage}
}%%
\\
$\epsilon=0.475$, \textcolor{blue}{$1\times6$} Vs. \textcolor{red}{$2^3$}, \textcolor{blue}{Safe}, \textcolor{red}{Safe}
\\
\subfigure[$y_1-y_2$]{
\begin{minipage}[t]{0.3\linewidth}
\centering
\includegraphics[width=1.6in]{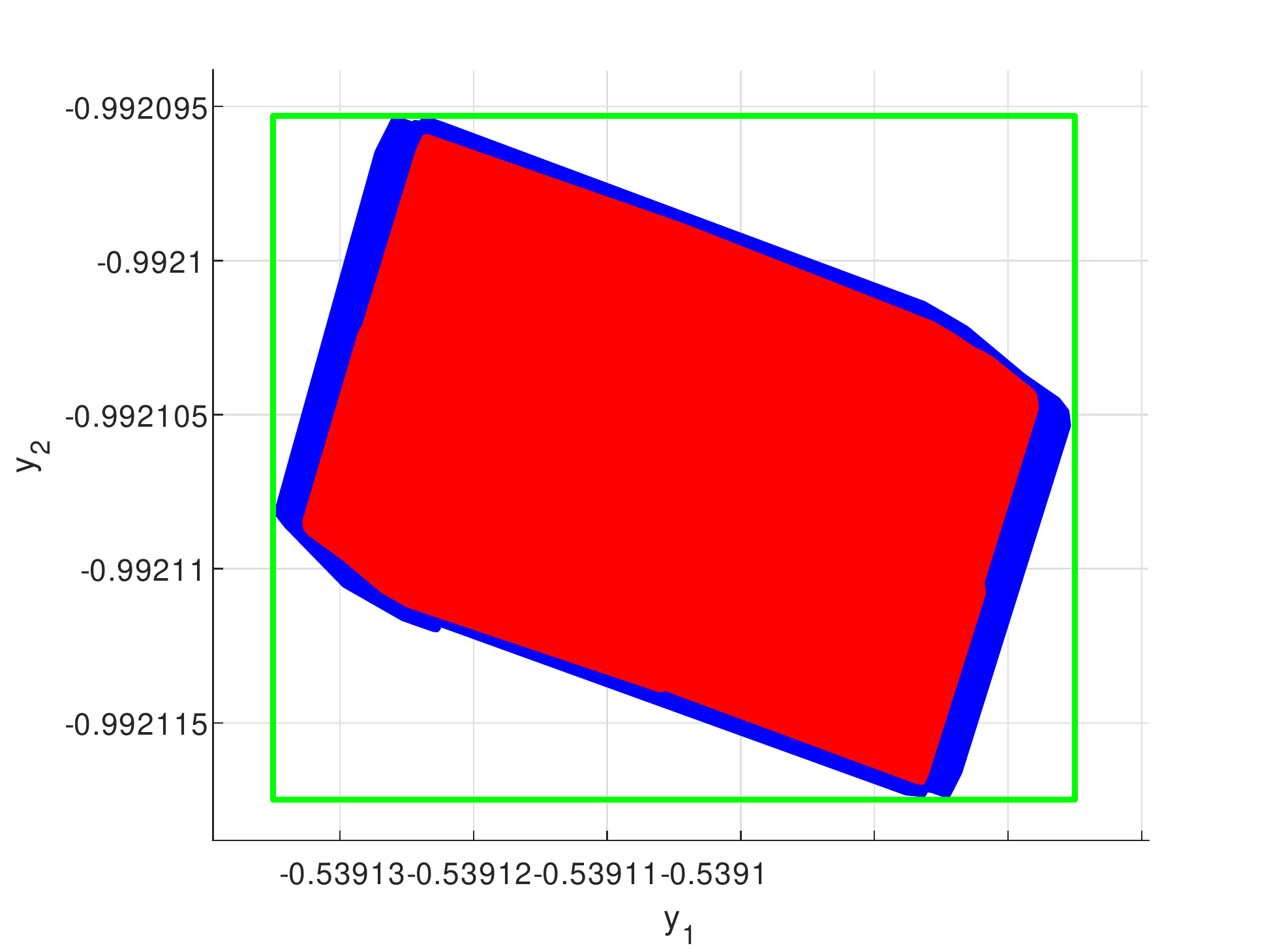}
\label{0.375-1}
\end{minipage}%
}%
\subfigure[$y_1-y_3$]{
\begin{minipage}[t]{0.3\linewidth}
\centering
\includegraphics[width=1.6in]{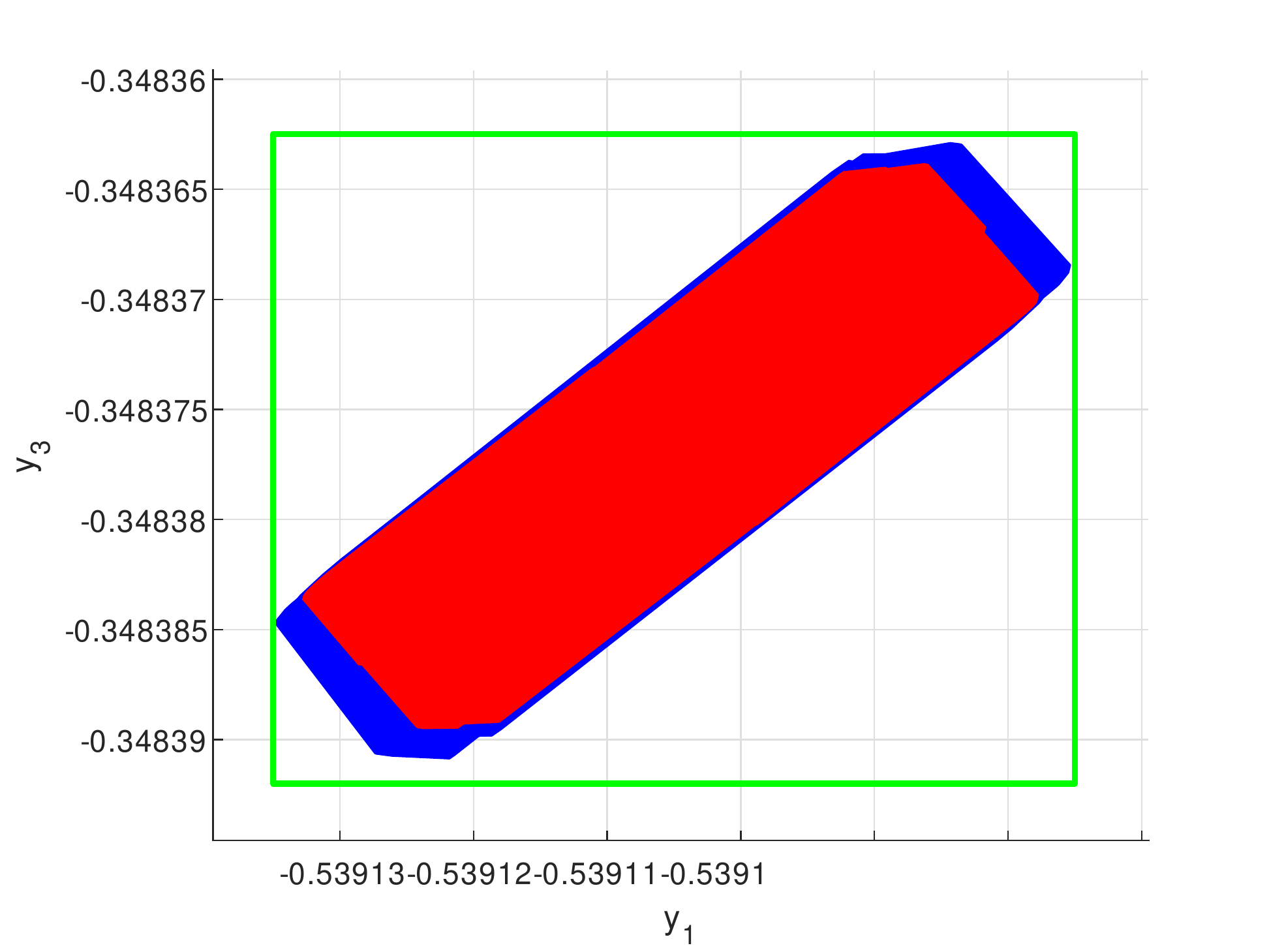}
\label{0.375-2}
\end{minipage}%
}%
\subfigure[$y_2-y_3$]{
\begin{minipage}[t]{0.3\linewidth}
\centering
\includegraphics[width=1.6in]{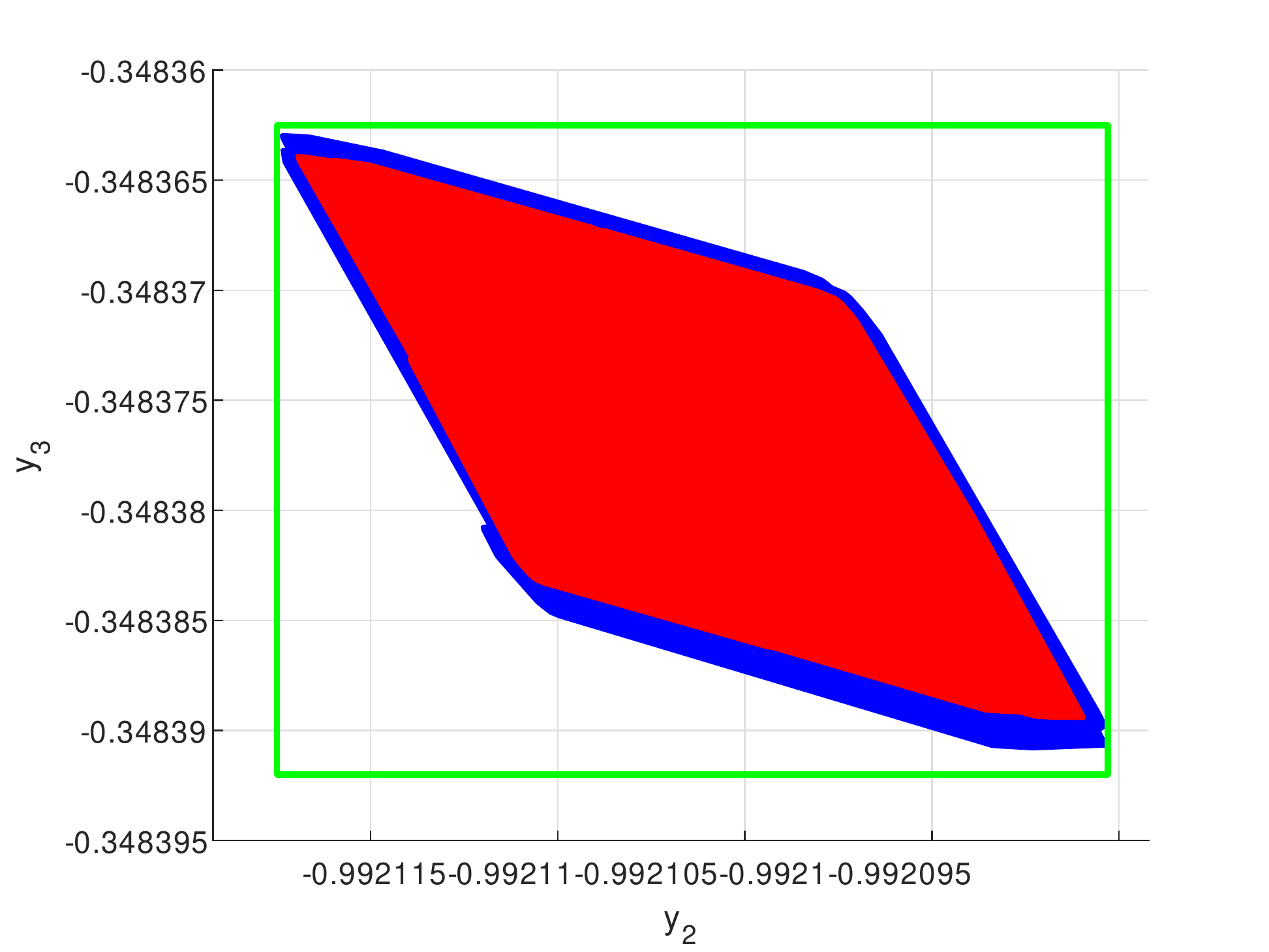}
\label{0.375-3}
\end{minipage}
}%%
\\
$\epsilon=0.500$, \textcolor{blue}{$1\times6$} Vs. \textcolor{red}{$2^3$}, \textcolor{blue}{Safe}, \textcolor{red}{Safe}
\centering
\caption{Safety verification on $\bm{N}_4$.
\textcolor{blue}{$\Omega(\partial \mathcal{X}_{in})$}; \textcolor{red}{$\Omega(\mathcal{X}_{in})$}; \textcolor{green}{$\partial \mathcal{X}_{s}$}}
\label{3-dim inn1}

\end{figure}

The output reachable sets from our set-boundary reachability method and the entire set based method are displayed in blue and red in Fig. \ref{2-dim inn} and \ref{3-dim inn1}, respectively. Further, we also show the exact output reachable sets estimated via the Monte-Carlo simulation method in Fig. \ref{2-dim inn}, which corresponds to the yellow regions. The visualized results show that the set-boundary reachability method can generate tighter output reachable sets than the entire set based method. As a result, our set-boundary reachability method can verify the safety properties successfully when $\epsilon \in \{0.375,0.400,0.425,0.450\}$, as shown in  Fig. \ref{0.125}-\ref{0.425}. In contrast, the entire set based method fails for cases with large input sets, demonstrated in Fig. \ref{0.25}-\ref{0.425}, since the computed output reachable sets are not included in safe sets. Furthermore, when the safety property cannot be verified with respect to the input set $[-1.0,1.0]^2$, we impose the uniform partition operator on both the entire input set and its boundary. When the boundary is divided into 20 equal subsets, the safety verification can be verified using the set-boundary reachability method (zoomed in, Fig. \ref{-11_2}) with the total verification time of 85.7886 seconds. However, when the entire input set is used, it should be partitioned into 64 equal subsets (zoomed in, Fig. \ref{-11_3}) and the total verification time is 306.5307 seconds. Consequently, the computation time from the set-boundary reachability method is reduced by 72.0\%, as opposed to the entire input set based method. These times are also listed in Table \ref{tab:computation time}.

\begin{table}[hbt]
    \centering
     \caption{Time consumption of safety verification on $\bm{N}_4$}
    \setlength{\tabcolsep}{3mm}{
    \begin{tabular}{*{7}{c}}
  \toprule
  \multirow{2}*{\textbf{Partition}} & \multicolumn{2}{c}{$\bm{\epsilon=0.450}$} & \multicolumn{2}{c}{$\bm{\epsilon=0.475}$} & \multicolumn{2}{c}{$\bm{\epsilon=0.500}$} \\
  \cmidrule(lr){2-3}\cmidrule(lr){4-5}\cmidrule(lr){6-7}
  &\textbf{Entire set} & \textbf{Boundary} &\textbf{Entire set} & \textbf{Boundary} & \textbf{Entire set} & \textbf{Boundary} \\
  \midrule
 1 & 0.7215&	\textbf{\textcolor{blue}{4.1115}}&	0.7246   & \textbf{\textcolor{blue}{4.2466}} &	0.7353 &	\textbf{\textcolor{blue}{4.3676}} \\
 2 & \textbf{\textcolor{red}{5.7110}}&	\textcolor{gray}{-}&	\textbf{\textcolor{red}{5.7957}}  & \textcolor{gray}{-} &	\textbf{\textcolor{red}{5.9281}} &	\textcolor{gray}{-} \\
  \midrule
 \textbf{Total} & \textbf{6.4325} & \textbf{4.1115} & \textbf{6.5203} & \textbf{4.2466}& \textbf{6.6634} & \textbf{4.3676}\\
  \bottomrule
\end{tabular}}
\label{tab:computation time N24}
\end{table}

For $\bm{N}_4$, the entire set based method fails to verify the safety property with respect to all the input sets, whose output reachable sets are shown in Fig. \ref{3-dim inn1111}. It succeeds with partitioning each input set into $2^3$ equal subsets, while the set-boundary reachability method succeeds without partitioning, as shown in Fig. \ref{3-dim inn1}.  It also can be observed from Table \ref{tab:computation time N24}, which shows the time consumption for safety verification on $\bm{N}_4$, that our set-boundary reachability method reduces the time consumption by around 35\% in all the cases.

\subsubsection{Experiments on General Feedforward NNs}

 When the homeomorphism property cannot be assured with respect to the given input region, our method is also able to facilitate the extraction of subsets from the input region for safety verification, as done in Algorithm \ref{alg: iNNs1}. 
In this subsection, we experiment on a non-invertible NN $\bm{N}_5$, which shares a similar structure with $\bm{N}_3$, with 3 hidden layers. The input set $\mathcal{X}_{in}$ and safe set $\mathcal{X}_s$ are $[-9.5, -9]\times [9.25, 9.5]$ and $[0.16294,0.16331]\times[-0.1393,-0.13855]$, respectively. For verifying the safety property successfully, the entire set based method implemented on the tool OCNNV has to divide the input set into 50 equal subsets and the time consumed is about 47.3114 seconds. The corresponding output reachable set is displayed in Fig.\ref{result} with the red region, which also shows the boundary of the safe region in green. Then, based on the tool OCNNV, we follow the computational procedure in Algorithm \ref{alg: iNNs1} to verify the safety property. The subset $\mathcal{A}=[-9.4,-9.05]\times [9.3,9.45] \cup [-9.45,-9.4]\times [9.35,9.45]$ rendering the NN homeomorphic is visualized in  Fig. \ref{initial} with the orange region, and the subset $\overline{\mathcal{X}_{in}\setminus \mathcal{A}}$ is shown with the blue region in Fig. \ref{initial}, which covers only 54\% of the initial input set and of which the corresponding output reachable set is illustrated in Fig.\ref{result} with the blue region. It is noting to point out that the blue region overlaps with the  boundary of the red region exactly. In this verification computation, the homeomorphism analysis takes around 10.3737  seconds and the output reachable set computations take 25.2487 seconds, totaling 35.6224 seconds. Therefore, our method reduces the time consumption of verification by 24.7\%, compared with the entire set based method.

\begin{figure}[htbp]
\centering
\subfigure[$\mathcal{X}_{in}$: \textcolor{orange}{$\mathcal{A}$} $\cup $\textcolor{blue}{$\overline{\mathcal{X}_{in}\setminus  \mathcal{A}}$}]{
\begin{minipage}[t]{0.5\linewidth}
\centering
\includegraphics[width=2in]{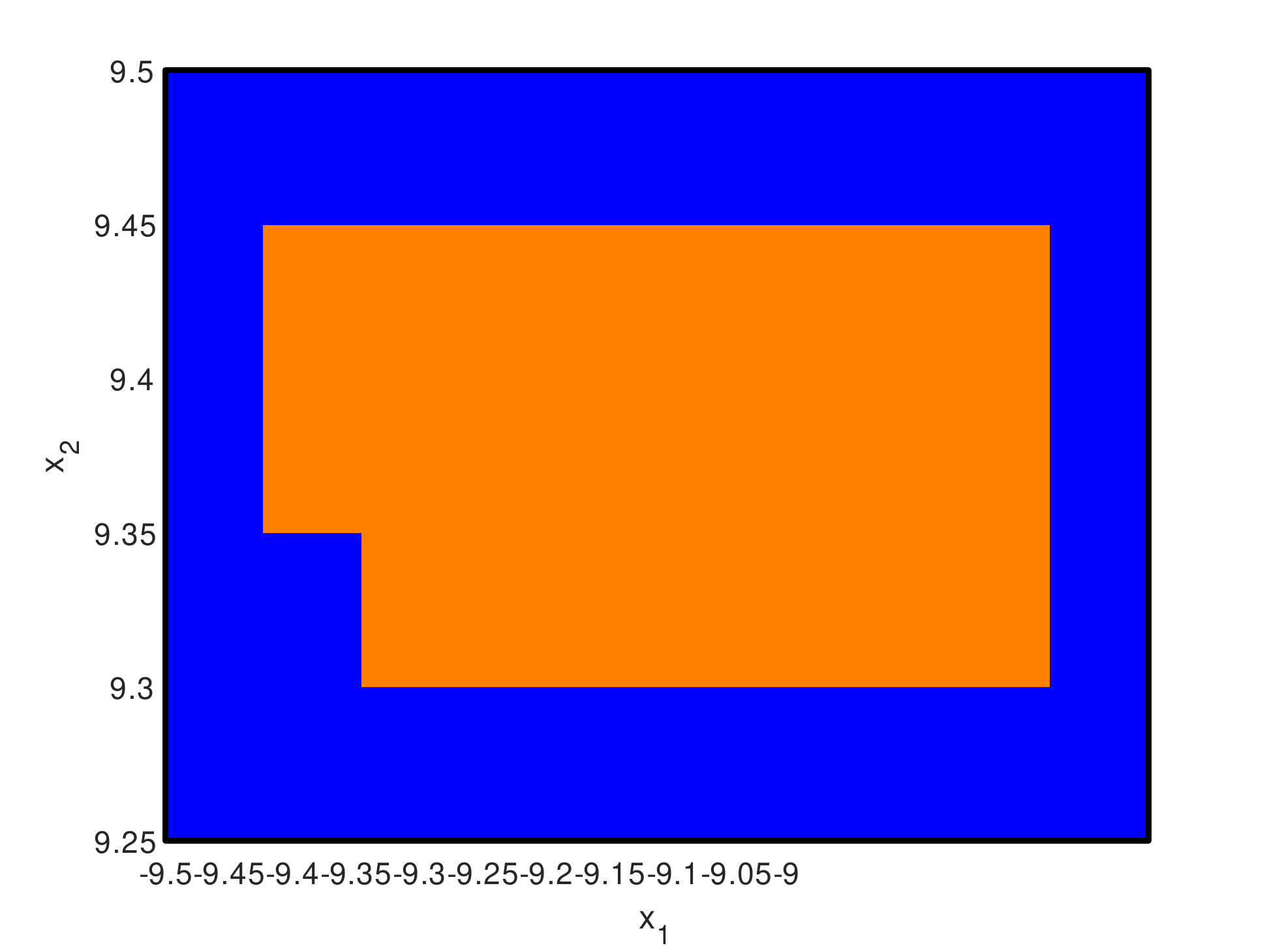}
%\caption{fig1}
\label{initial}
\end{minipage}%
}%
\centering
\subfigure[\textcolor{blue}{Safe}, \textcolor{red}{Safe}.
\textcolor{blue}{$\Omega(\overline{\mathcal{X}_{in}\setminus  \mathcal{A}})$}; \textcolor{red}{$\Omega(\mathcal{X}_{in})$}; \textcolor{green}{$\partial \mathcal{X}_{s}$}]
{
\begin{minipage}[t]{0.5\linewidth}
\centering
\includegraphics[width=2in]{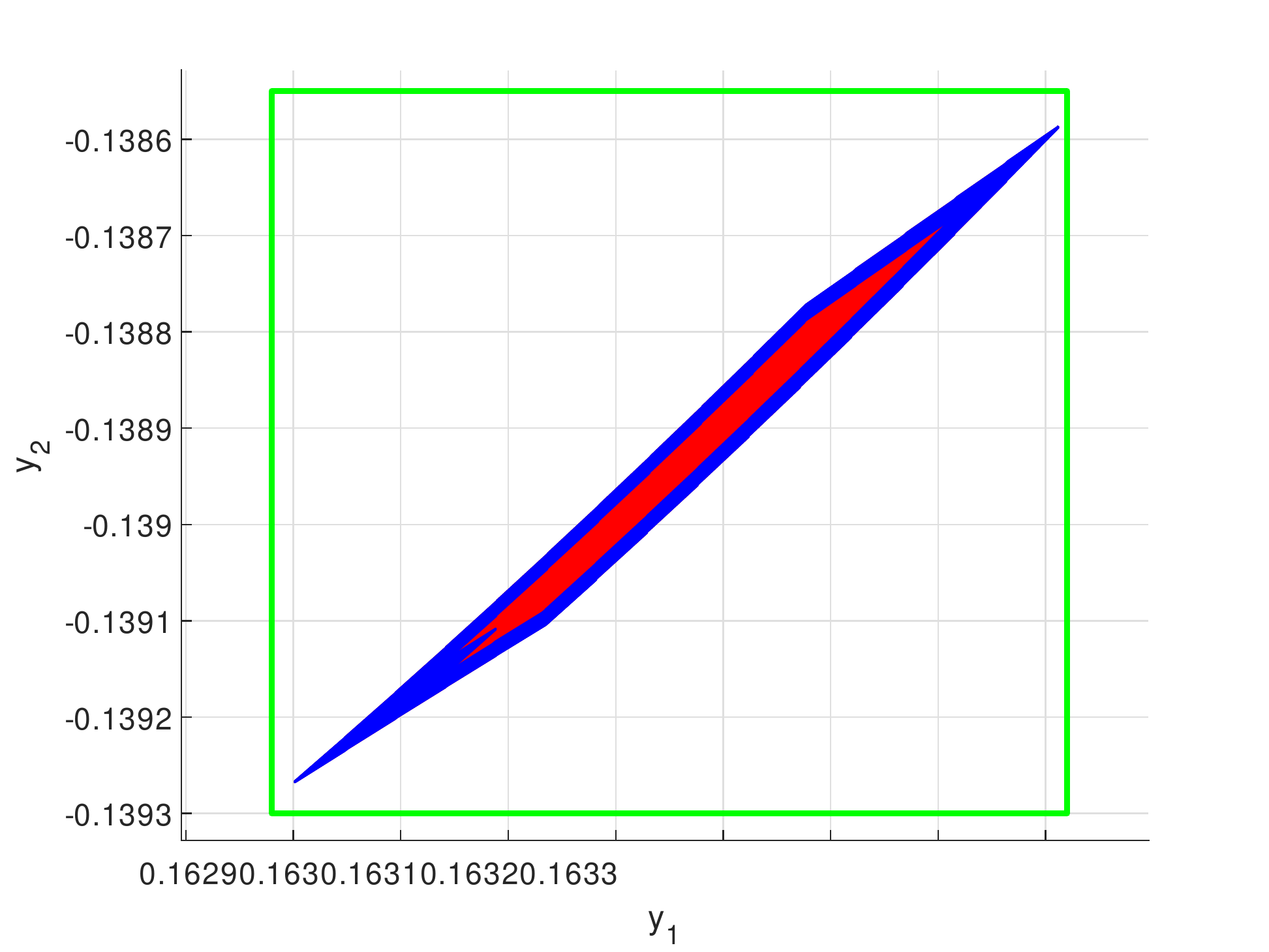}
%\caption{fig2}
\label{result}
\end{minipage}%
}%
\centering
\caption{Safety verification on $\bm{N}_5$}
\label{2-dim non-inn}

\end{figure}

\subsection{Experiments on safety verification with open maps}

In this subsection, we carry out the safety verification with open maps, compared with widely utilized neural network verification tools (or systems), that is IBP, DeepZ \cite{singh2018fast} and Verisig 2.0 \cite{ivanov2021verisig}, all of which are entire set based reachability methods. Recall that IBP and DeepZ compute an over-approximation of the reachable set based on interval and zonotope domains and Verisig 2.0 works with Taylor models. For simplicity, on a specific reachable set computation method \emph{M}, we denote \emph{M-O} as the reachable set computation method based on open maps, for example, \emph{DeepZ} and \emph{DeepZ-O}. Several NNs, $\bm{N}_6-\bm{N}_{11}$,  with different sizes are used to demonstrate the performance of the proposed reachability method by open maps, whose architectures are displayed in Table \ref{nn block}.  All the parameters of the NNs are generated randomly and they are tested for establishing open maps. By the way, during the generation process of NNs, we found that almost all the generated blocks satisfy the requirements of open maps. Besides, unlike boundary analysis based on homeomorphisms, the experimental NNs we utilized here own relatively large (output) dimensions. Therefore, in this subsection, we no longer explicitly show the safety regions, but focus more on comparing the ranges of output reachable sets computed from existing entire set based reachability methods and set-boundary analysis based ones. Generally, a smaller reachable set means a higher probability of passing the safety verification.

\begin{table}[htbp]
  \centering
  \caption{Structures of utilized neural networks.}
    \setlength{\tabcolsep}{5mm}{ \begin{tabular}{c|c|c|c}
    \Xhline{1.2pt}
    \textbf{NN} & \textbf{Structure} & \textbf{NN} & \textbf{Structure} \\
    \hline
    $\bm{N_6}$    & $\bm{4}$-4-3-3-$\bm{2}$  &
    $\bm{N_7}$      & $\bm{5}$-4-4-3-3-$\bm{2}$ \\
    \hline
    $\bm{N_8}$     & $\bm{50}$-45-40-35-30-25-20-15-$\bm{10}$ &
    $\bm{N_9}$     & \quad $\bm{50}$-46-42-38-34-30-26-22-18-14-$\bm{10}$\\
    \hline
    $\bm{N_{10}}$     & $\bm{80}$-76-72-68-64-$\bm{60}$  &
     $\bm{N_{11}}$    & $\bm{80}$-77-74-71-68-65-$\bm{62}$ \\
    \Xhline{1.2pt}
    \end{tabular}}
  \label{nn block}%
\end{table}%

For the reason that reachable sets computed by IBP tend to explode in high dimension and Verisig 2.0 is used in neural network controller systems (NNCS) where the NN controller has  a small number of neurons in each layer, that is to say, the over-approximation of the reachable sets with high dimensions resulted from IBP are overly pessimistic, with low accuracy, and those of Verisig 2.0 are prohibitively time-consuming. Consequently,  we only compared IBP, Verisig 2.0, and our set-boundary based reachability methods on small size NNs $\bm{N_6}$ and $\bm{N_7}$. For DeepZ, we make comparisons with all the NNs due to its well-performed scalability, i.e., $\bm{N_6}-\bm{N_{11}}$.

For each NN, we evaluate the entire set based reachability methods and the set-boundary based ones with varied perturbations, i.e., $\mathcal{X}_{in} = [-\epsilon, \epsilon]^{n}$, where $n$ is the  input dimension.  
For presentation clarity, we compute  the interval hulls of the superset outputted by IBP, DeepZ or Verisig 2.0 and those resulting from the proposed  boundary analysis based on open maps for subsequent comparisons. We take the interval range of each dimension as the benchmark and the range of  interval $[lb, ub]$ is denoted by $|ub - lb|$. The minimum, maximum and mean of the interval range ratios between the interval hulls of  the set-boundary based reachability methods and the entire based ones are compared comprehensively.

\begin{figure}[tbp]
\centering
\subfigure[$\epsilon=0.1$, \textcolor{blue}{IBP} Vs. \textcolor{red}{IBP-O}]{
\begin{minipage}[t]{0.32\linewidth}
\centering
\includegraphics[width=1.6in]{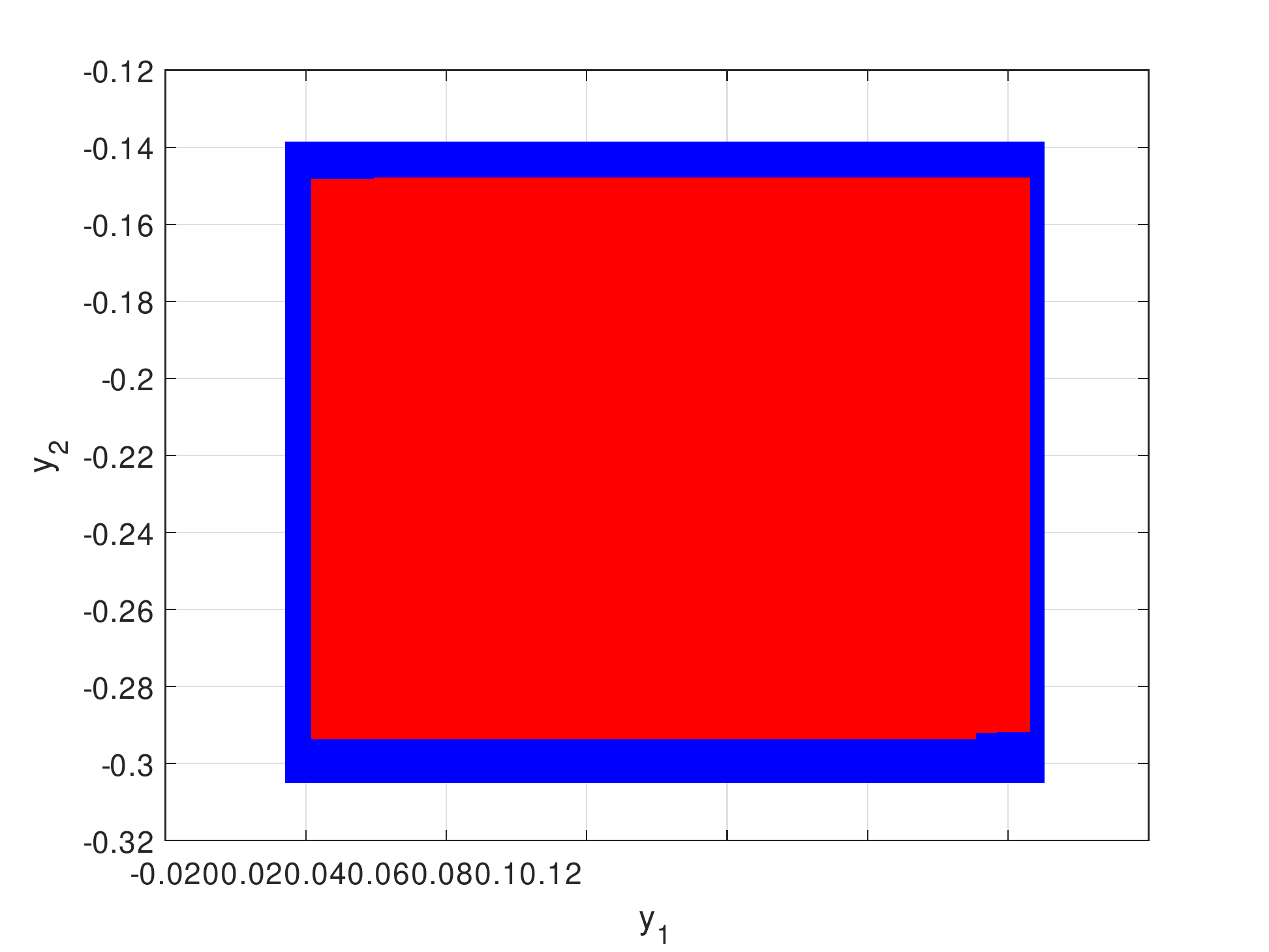}
%\caption{fig1}
\label{S1-ibp0.1}
\end{minipage}%
}%
\subfigure[$\epsilon=0.2$,  \textcolor{blue}{IBP} Vs. \textcolor{red}{IBP-O}]{
\begin{minipage}[t]{0.32\linewidth}
\centering
\includegraphics[width=1.6in]{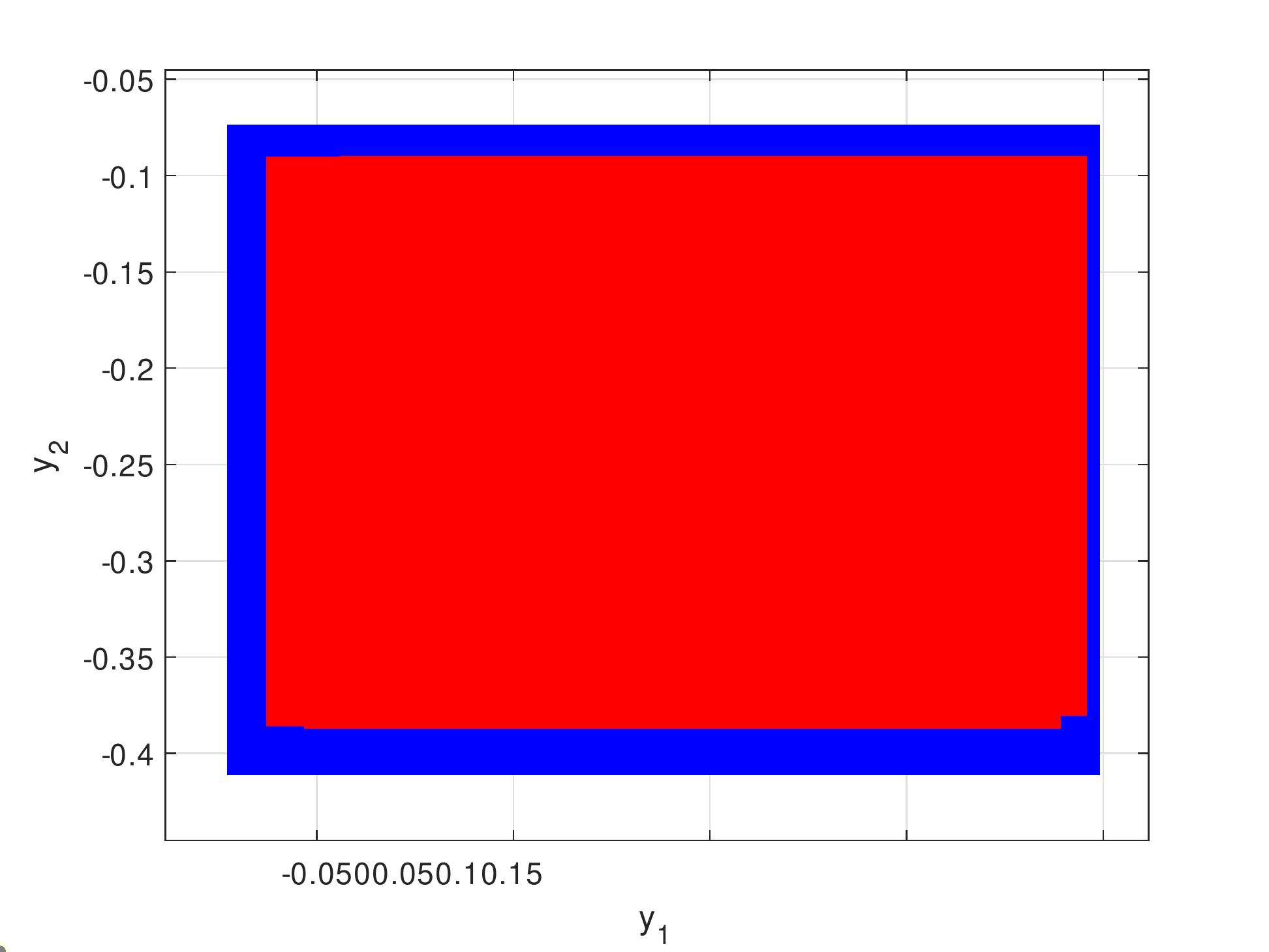}
%\caption{fig2}
\label{S1-ibp0.2}
\end{minipage}%
}%
\subfigure[$\epsilon=0.5$,  \textcolor{blue}{IBP} Vs. \textcolor{red}{IBP-O}]{
\begin{minipage}[t]{0.32\linewidth}
\centering
\includegraphics[width=1.6in]{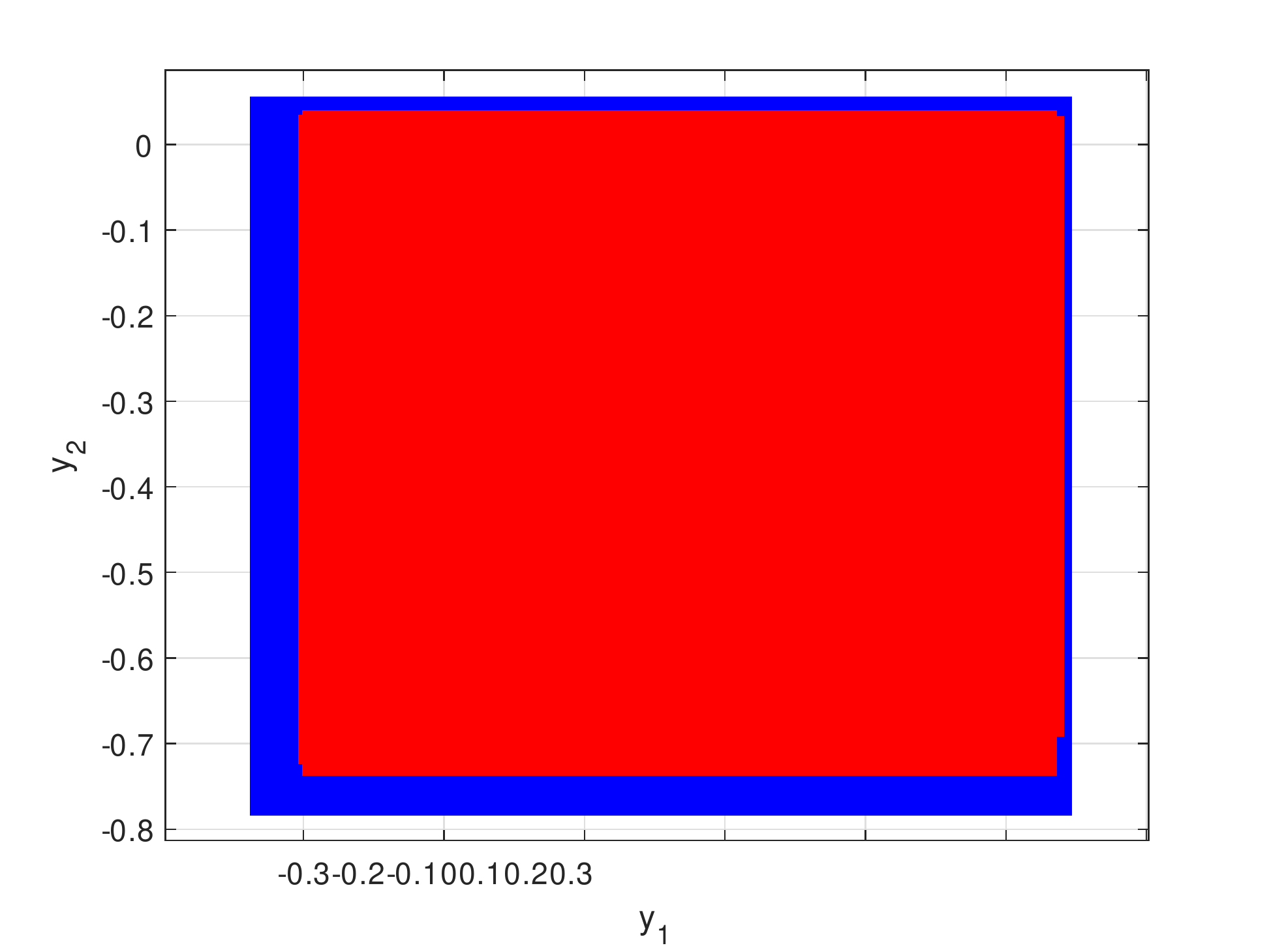}
%\caption{fig1}
\label{S1-ibp0.5}
\end{minipage}%
}%
\\
\subfigure[$\epsilon=0.1$,  \textcolor{blue}{DeepZ} Vs. \textcolor{red}{DeepZ-O}]{
\begin{minipage}[t]{0.32\linewidth}
\centering
\includegraphics[width=1.6in]{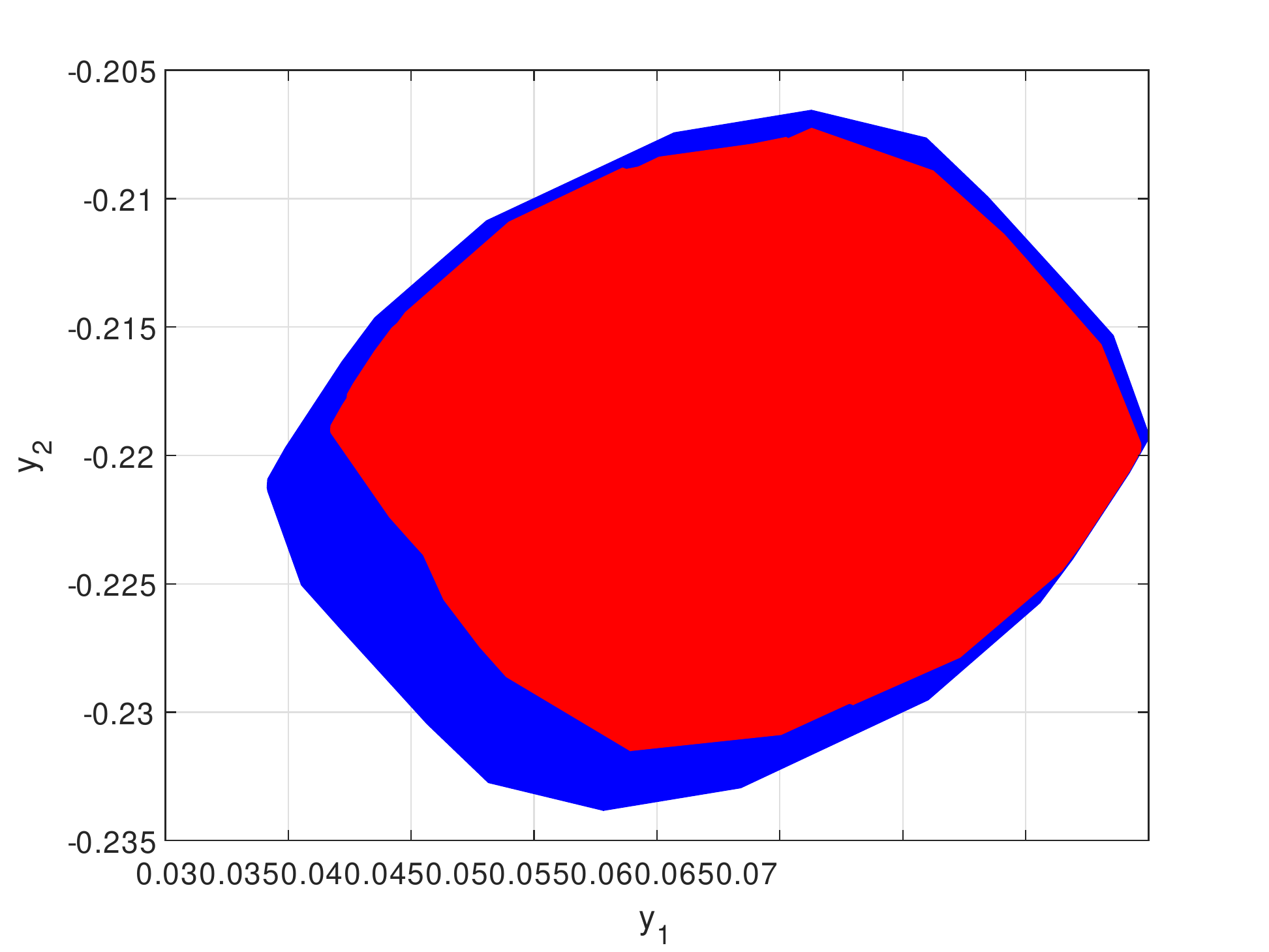}
%\caption{fig2}
\label{S1-deepz0.1}
\end{minipage}%
}%
\subfigure[$\epsilon=0.2$,  \textcolor{blue}{DeepZ} Vs. \textcolor{red}{DeepZ-O}]{
\begin{minipage}[t]{0.32\linewidth}
\centering
\includegraphics[width=1.6in]{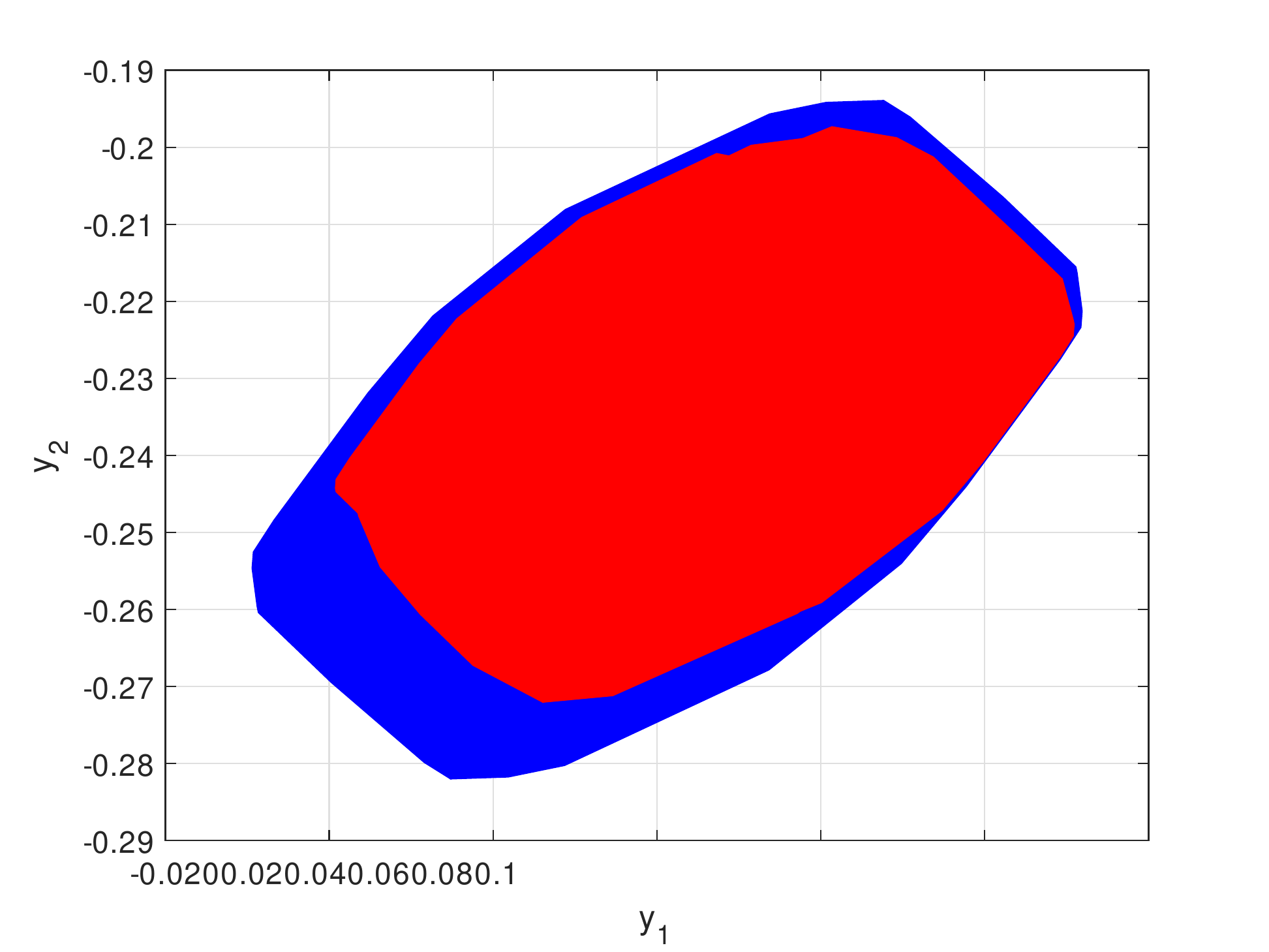}
%\caption{fig2}
\label{S1-deepz0.2}
\end{minipage}%
}%
\subfigure[$\epsilon=0.5$,  \textcolor{blue}{DeepZ} Vs. \textcolor{red}{DeepZ-O}]{
\begin{minipage}[t]{0.32\linewidth}
\centering
\includegraphics[width=1.6in]{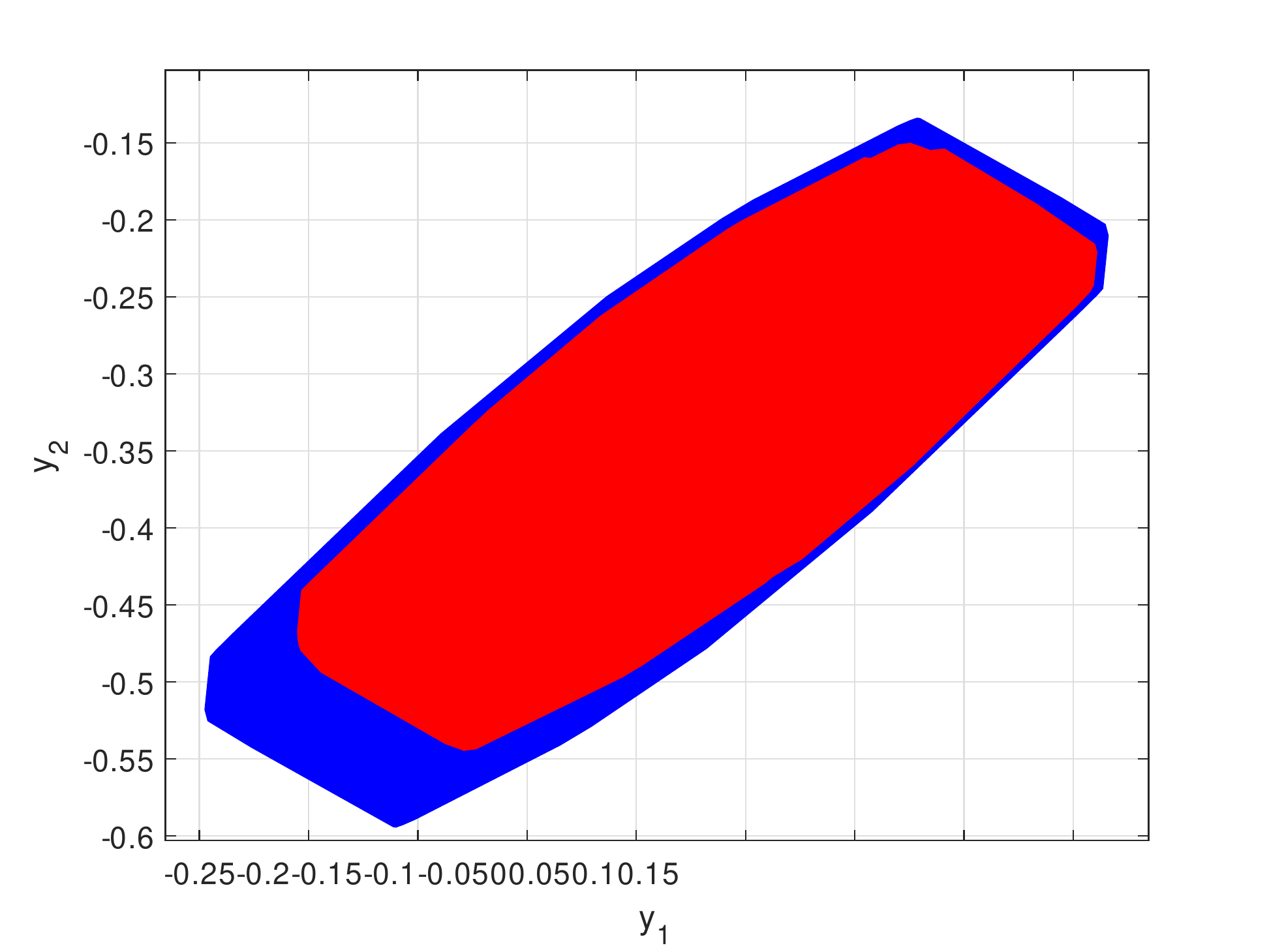}
%\caption{fig2}
\label{S1-deepz0.5}
\end{minipage}
}%%
\\
\subfigure[$\epsilon=0.1$, \textcolor{blue}{TM} Vs. \textcolor{red}{TM-O}]{
\begin{minipage}[t]{0.32\linewidth}
\centering
\includegraphics[width=1.6in]{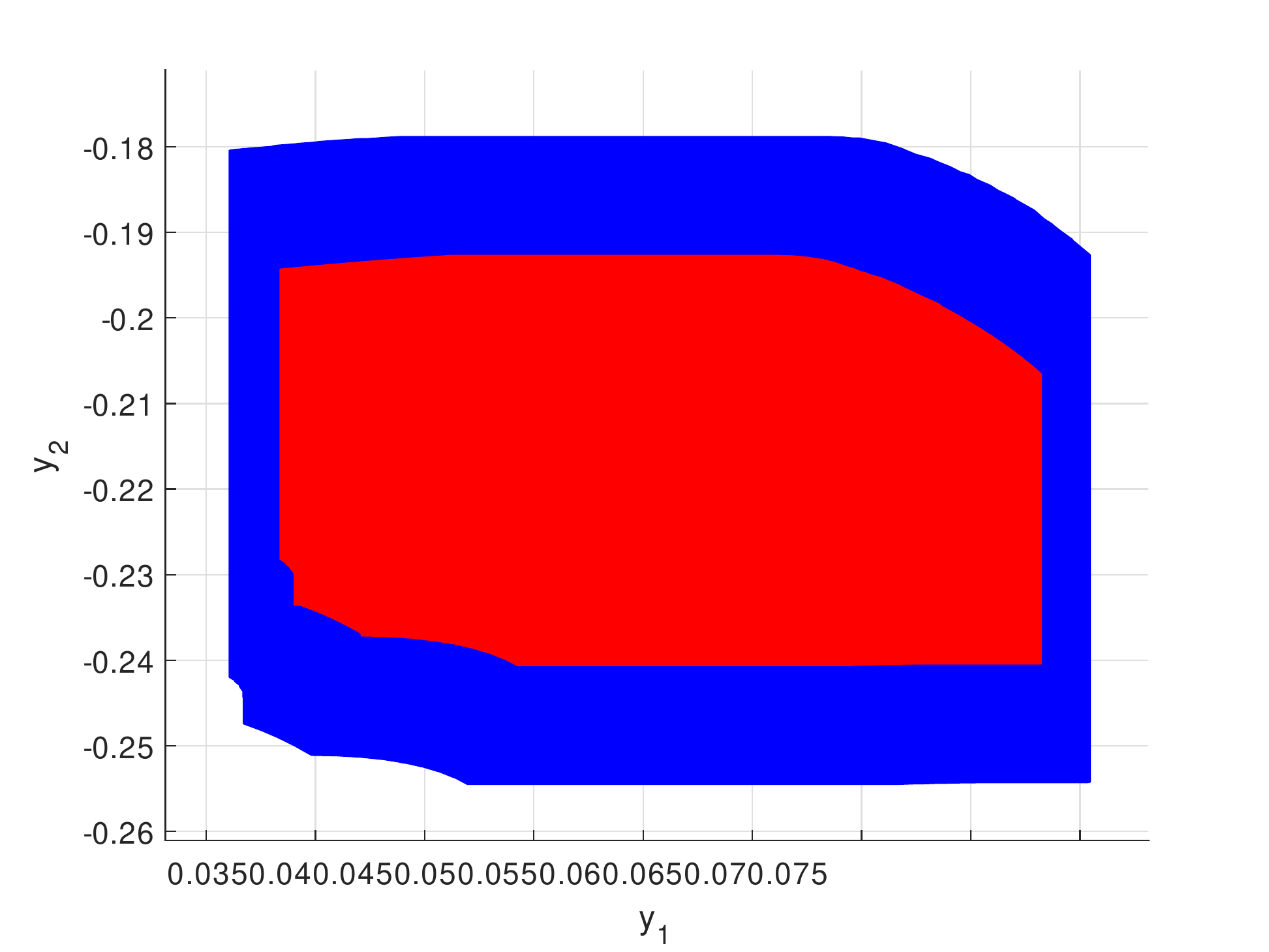}
%\caption{fig1}
\label{S1-tm0.1}
\end{minipage}%
}%
\subfigure[$\epsilon=0.2$, \textcolor{blue}{TM} Vs. \textcolor{red}{TM-O}]{
\begin{minipage}[t]{0.32\linewidth}
\centering
\includegraphics[width=1.6in]{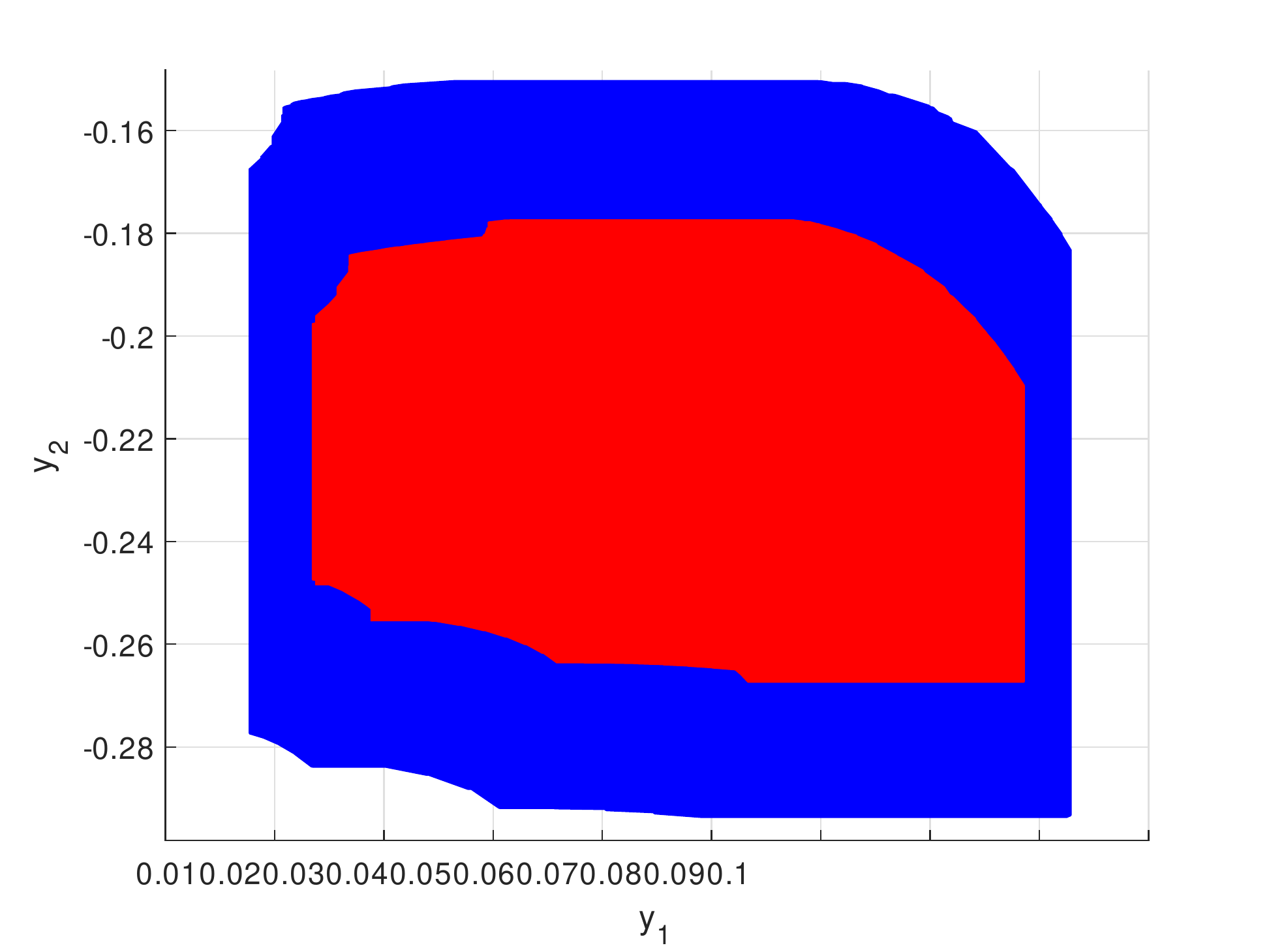}
%\caption{fig2}
\label{S1-tm0.2}
\end{minipage}%
}%
\subfigure[$\epsilon=0.5$, \textcolor{blue}{TM} Vs. \textcolor{red}{TM-O}]{
\begin{minipage}[t]{0.32\linewidth}
\centering
\includegraphics[width=1.6in]{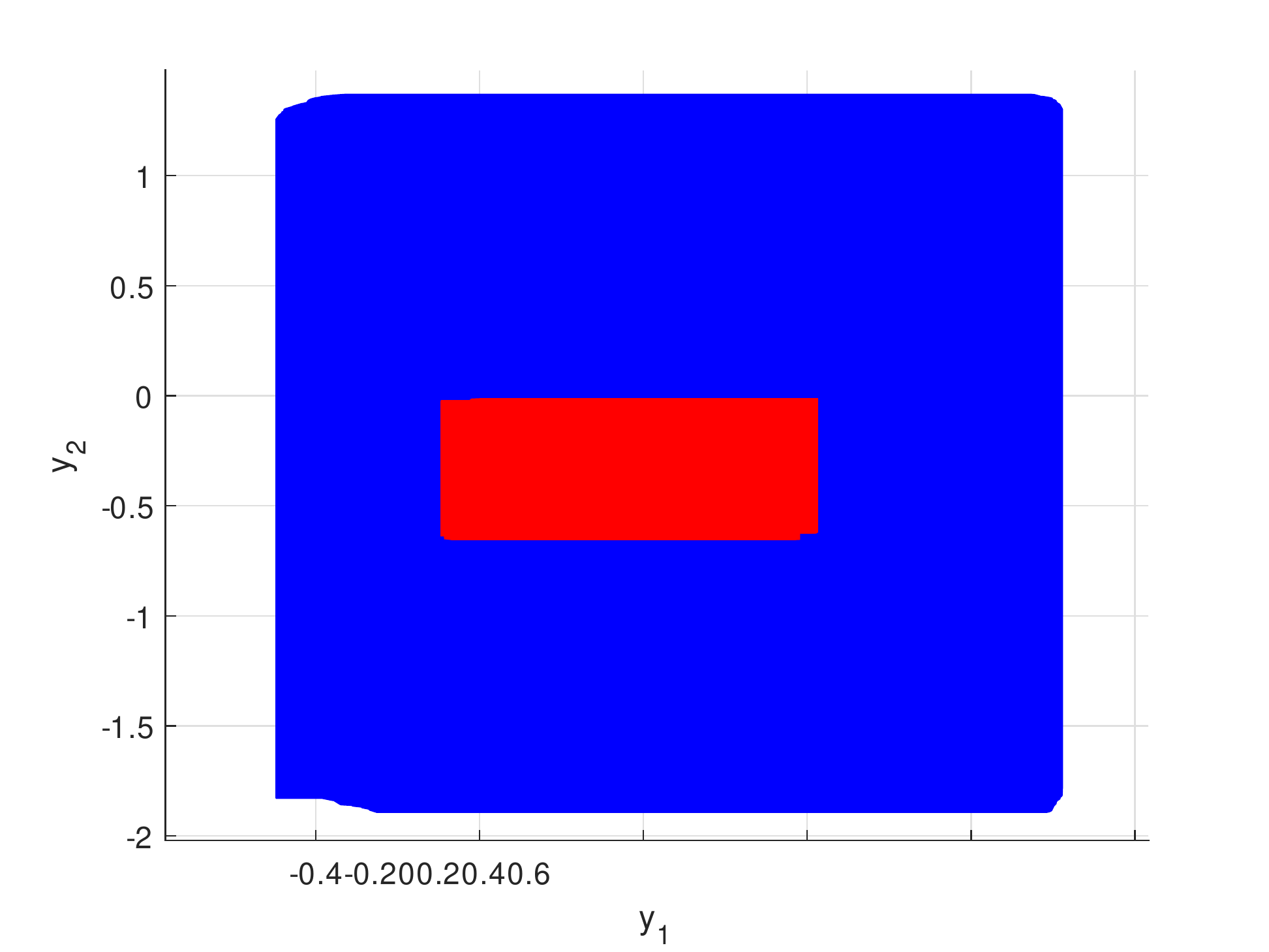}
%\caption{fig1}
\label{S1-tm0.5}
\end{minipage}%
}%
\centering
\caption{Reachable sets on NN $\bm{N_6}$.}
\label{block S1}
\end{figure}

\begin{table}[b]
  \centering
  \caption{Reachable set comparisons on NNs versus IBP and DeepZ.}
   \setlength{\tabcolsep}{4mm}{
    \begin{tabular}{c|c|c|c|c|c|c|c}
    \Xhline{1.2pt}
    \multirow{2}[4]{*}{\textbf{NN}} & \multirow{2}[4]{*}{\textbf{Perturbation}} & \multicolumn{3}{c|}{\textbf{IBP Vs. Boundary}} & \multicolumn{3}{c}{\textbf{DeepZ Vs. Boundary}} \bigstrut\\
\cline{3-8}          &       & \textbf{Min}   & \textbf{Max}   & \textbf{Mean}  & \textbf{Min}   & \textbf{Max}   & \textbf{Mean} \bigstrut\\
    \hline
    \multirow{3}[6]{*}{$\bm{N_6}$} & $\epsilon = 0.1$       &  0.8760     &  0.9472     &   0.9116    &  0.8887     & 0.9203     & 0.9045  \bigstrut\\
\cline{2-8}          & $\epsilon = 0.2$  & 0.8805  & 0.9408 & 0.9106 &0.8481 & 0.8898 & 0.8689 \bigstrut\\
\cline{2-8}          &  $\epsilon = 0.5$     & 0.9256      &   0.9318    &  0.9287     & 0.8561      &    0.8848   & 0.8705 \bigstrut\\
    \hline
    \multirow{3}[6]{*}{$\bm{N_7}$} &$\epsilon = 0.1$  &0.8843  & 0.8870  &0.8857  & 0.8544 & 0.9322 &0.8933  \bigstrut\\
\cline{2-8}          &  $\epsilon = 0.2$     &   0.8918    & 0.8946     & 0.8933      &0.8095       & 0.8577      &0.8336  \bigstrut\\
\cline{2-8}          & $\epsilon = 0.5$ &0.9353  & 0.9391
 & 0.9372 &0.7810  &0.8567  & 0.8188 \bigstrut\\
    \Xhline{1.2pt}
    \end{tabular}}
  \label{tab:IBP_deeepz}%
\end{table}%

Table \ref{tab:IBP_deeepz}  displays the comparison results versus IBP and DeepZ respectively on $\bm{N_6}$ and $\bm{N_7}$, and Table \ref{tab:deeepz} lists the comparisons versus DeepZ on the large size NNs, $\bm{N_8}$, $\bm{N_9}$, $\bm{N_{10}}$  and $\bm{N_{11}}$. It can be observed that our proposed set-boundary based reachability analysis computes a tighter output set in all the cases and can averagely reduce the entire set based over-approximated reachable sets by $3\%-25\%$ for IBP and $25\%-40\%$ for DeepZ. Moreover, the reachable sets computed by the original entire set based tools and the set-boundary based reachability analysis for NNs  $\bm{N_6}$, $\bm{N_7}$ are shown in Fig. \ref{block S1} and \ref{block S2}, which are in red and blue respectively.

\begin{figure}[t]
\centering
\subfigure[$\epsilon=0.1$,  \textcolor{blue}{IBP} Vs. \textcolor{red}{IBP-O}]{
\begin{minipage}[t]{0.32\linewidth}
\centering
\includegraphics[width=1.6in]{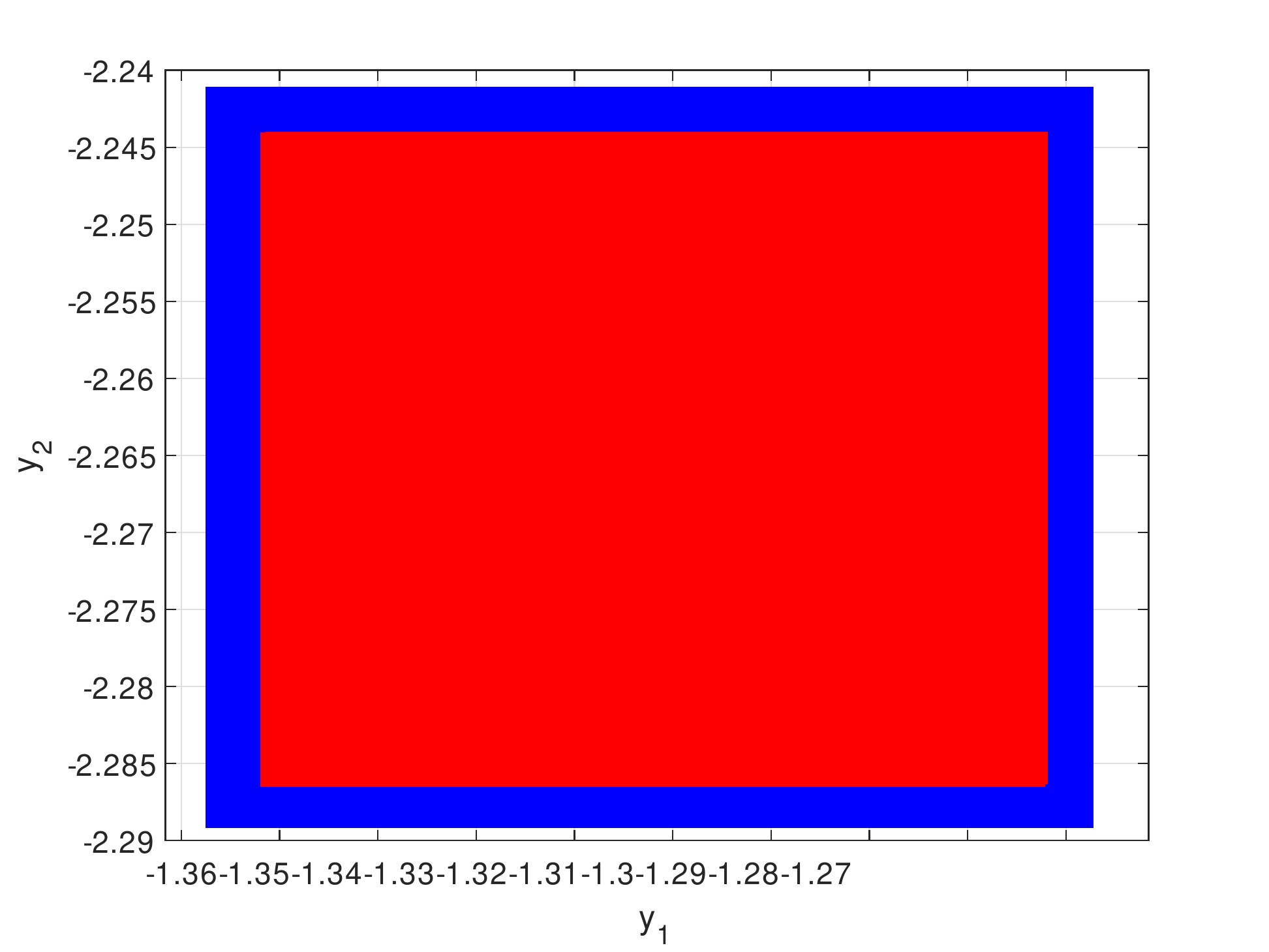}
%\caption{fig1}
\label{S2-ibp0.1}
\end{minipage}%
}%
\subfigure[$\epsilon=0.2$,  \textcolor{blue}{IBP} Vs. \textcolor{red}{IBP-O}]{
\begin{minipage}[t]{0.32\linewidth}
\centering
\includegraphics[width=1.6in]{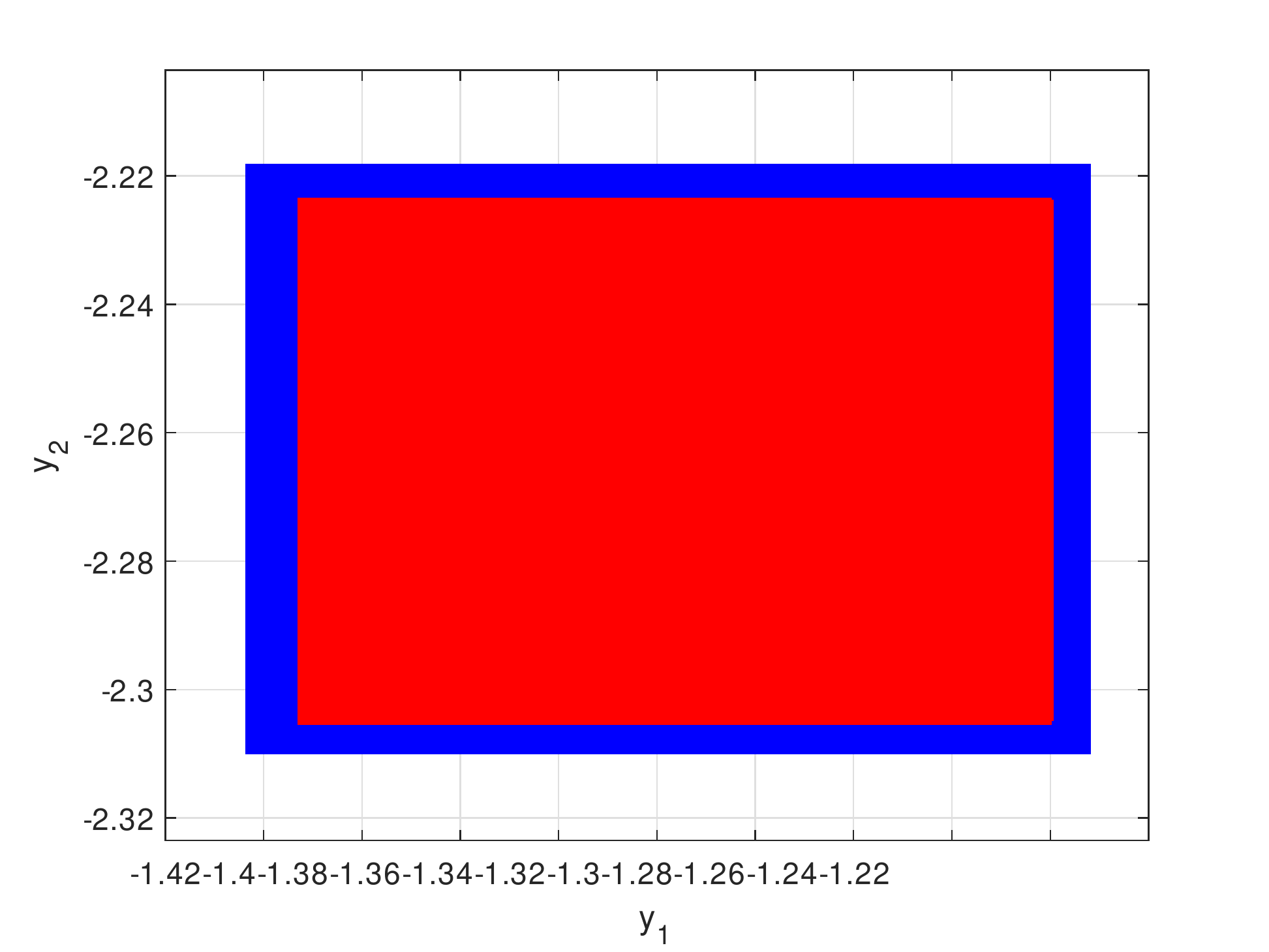}
%\caption{fig2}
\label{S2-ibp0.2}
\end{minipage}%
}%
\subfigure[$\epsilon=0.5$,  \textcolor{blue}{IBP} Vs. \textcolor{red}{IBP-O}]{
\begin{minipage}[t]{0.32\linewidth}
\centering
\includegraphics[width=1.6in]{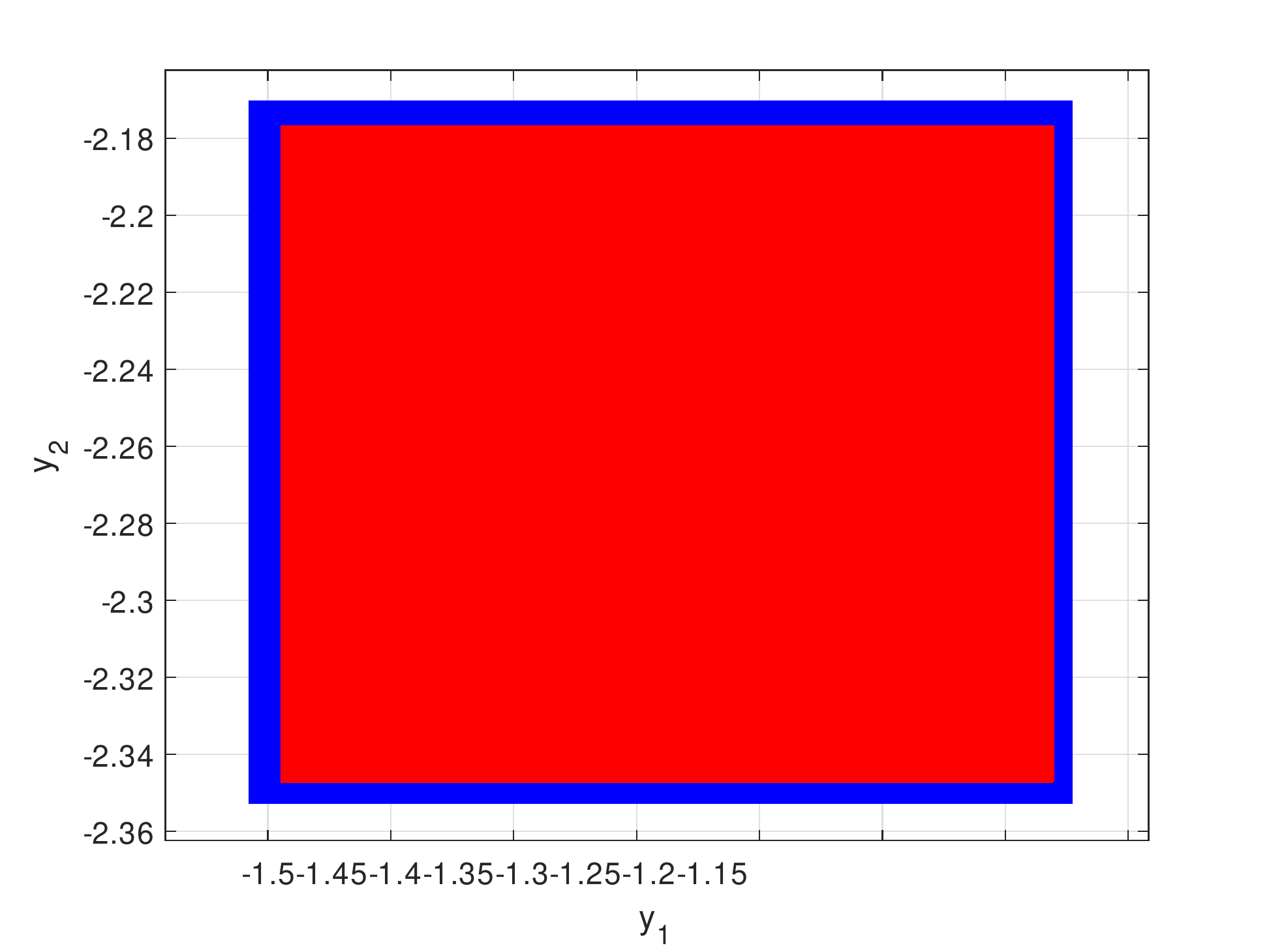}
%\caption{fig1}
\label{S2-ibp0.5}
\end{minipage}%
}%
\\
\subfigure[$\epsilon=0.1$,  \textcolor{blue}{DeepZ} Vs. \textcolor{red}{DeepZ-O}]{
\begin{minipage}[t]{0.32\linewidth}
\centering
\includegraphics[width=1.6in]{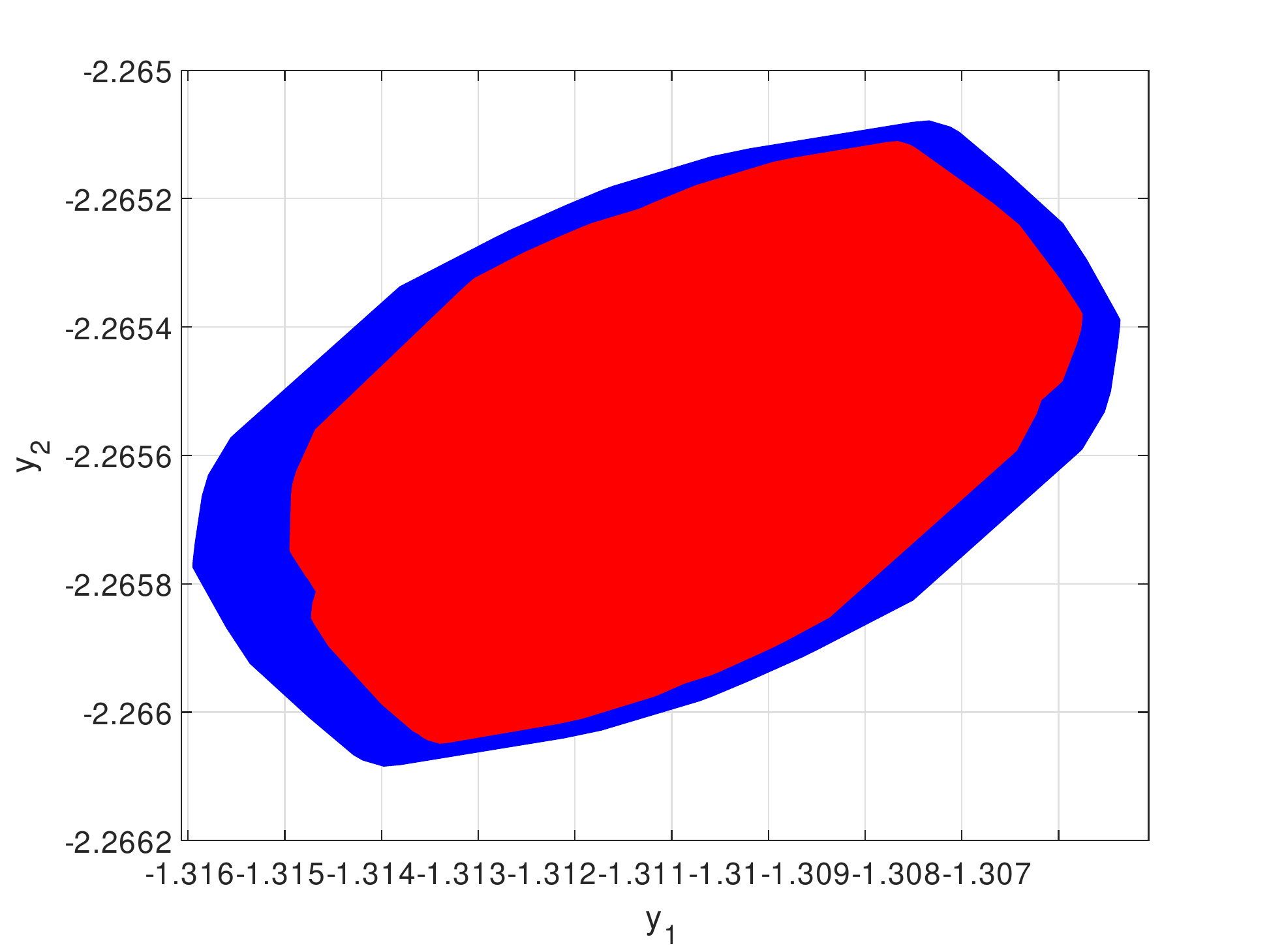}
%\caption{fig2}
\label{S2-deepz0.1}
\end{minipage}%
}%
\subfigure[$\epsilon=0.2$,  \textcolor{blue}{DeepZ} Vs. \textcolor{red}{DeepZ-O}]{
\begin{minipage}[t]{0.32\linewidth}
\centering
\includegraphics[width=1.6in]{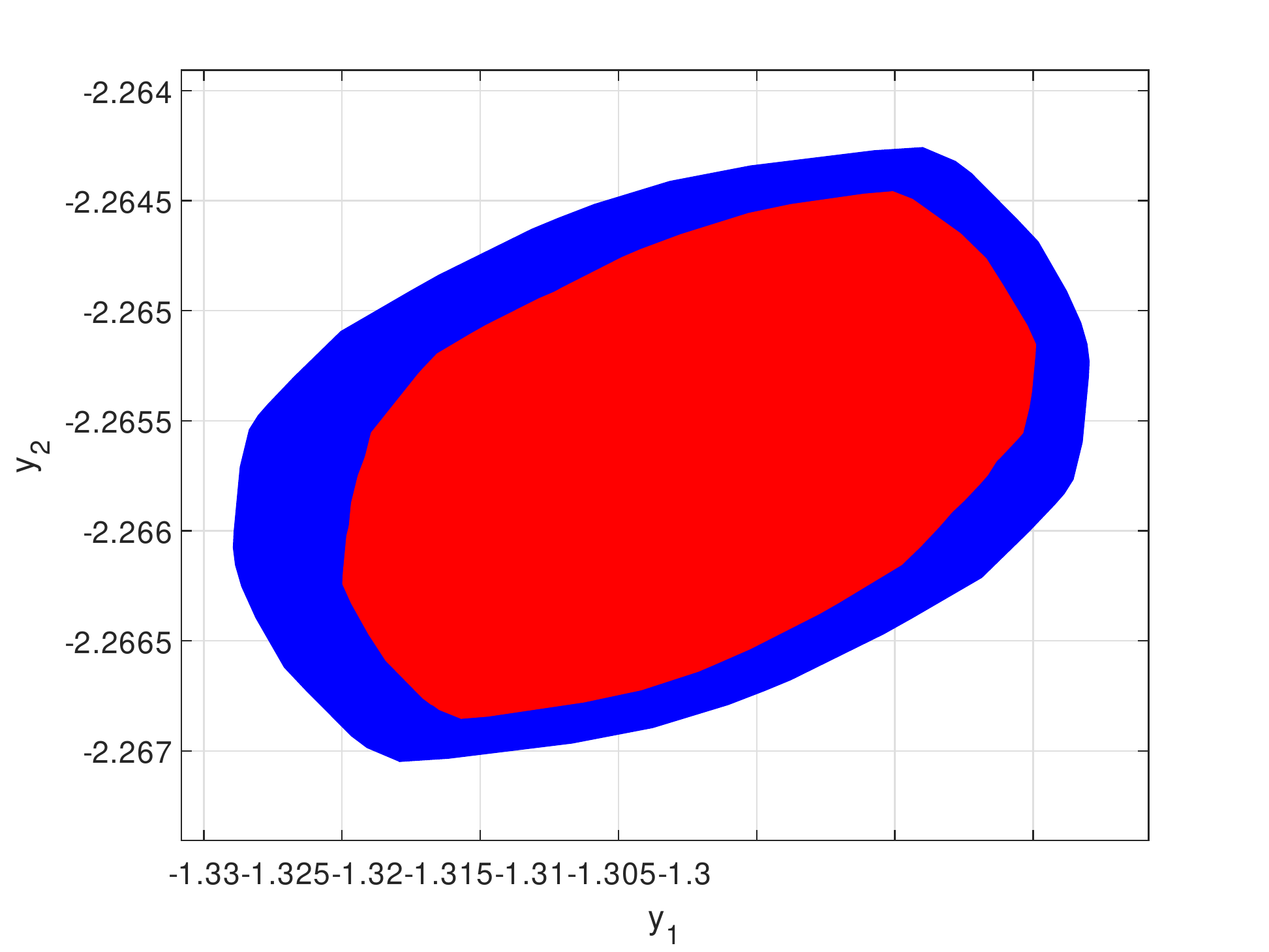}
%\caption{fig2}
\label{S2-deepz0.2}
\end{minipage}%
}%
\subfigure[$\epsilon=0.5$,  \textcolor{blue}{DeepZ} Vs. \textcolor{red}{DeepZ-O}]{
\begin{minipage}[t]{0.32\linewidth}
\centering
\includegraphics[width=1.6in]{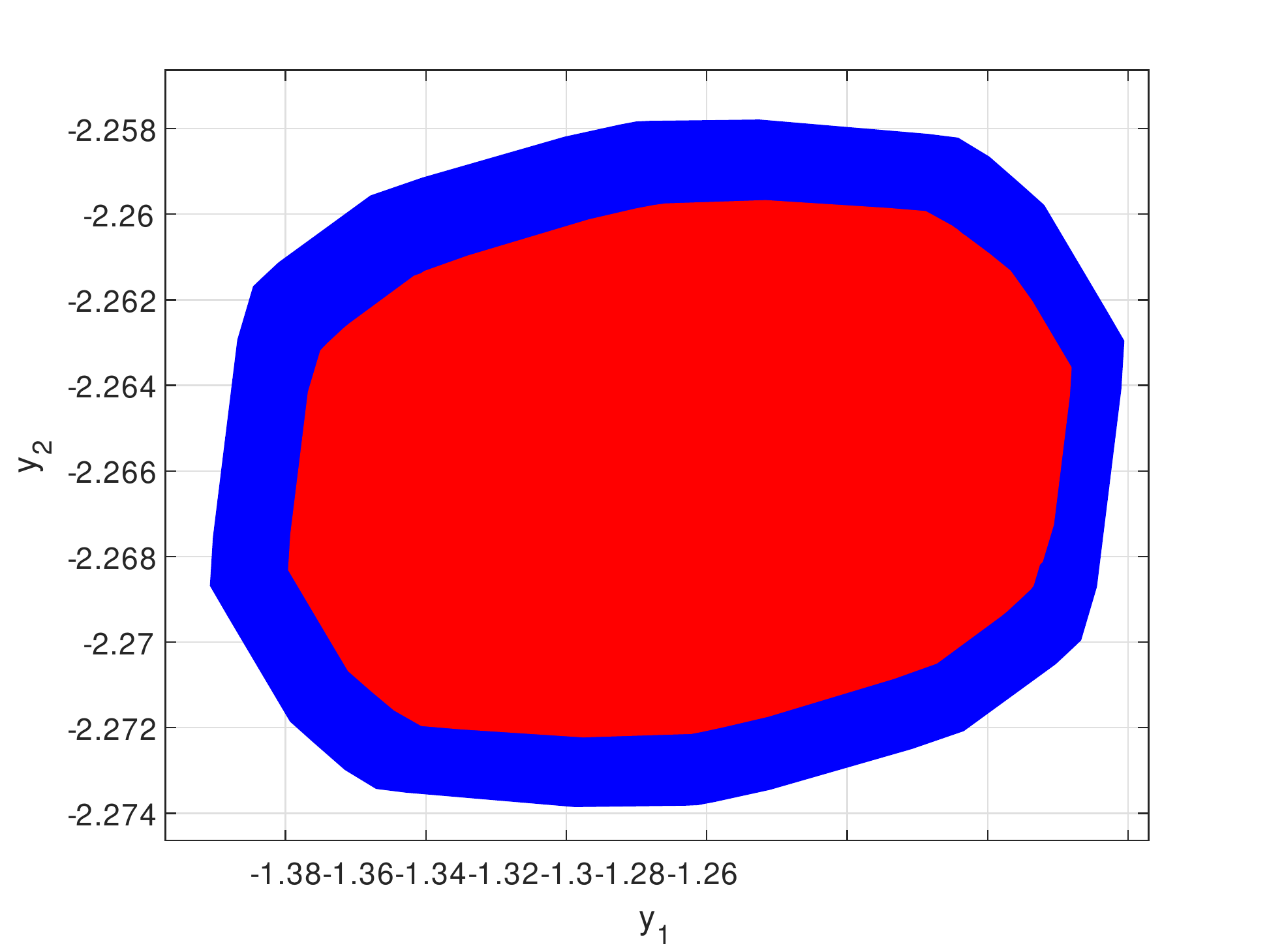}
%\caption{fig2}
\label{S2-deepz0.5}
\end{minipage}
}%%
\\
\subfigure[$\epsilon=0.1$, \textcolor{blue}{TM} Vs. \textcolor{red}{TM-O}]{
\begin{minipage}[t]{0.32\linewidth}
\centering
\includegraphics[width=1.6in]{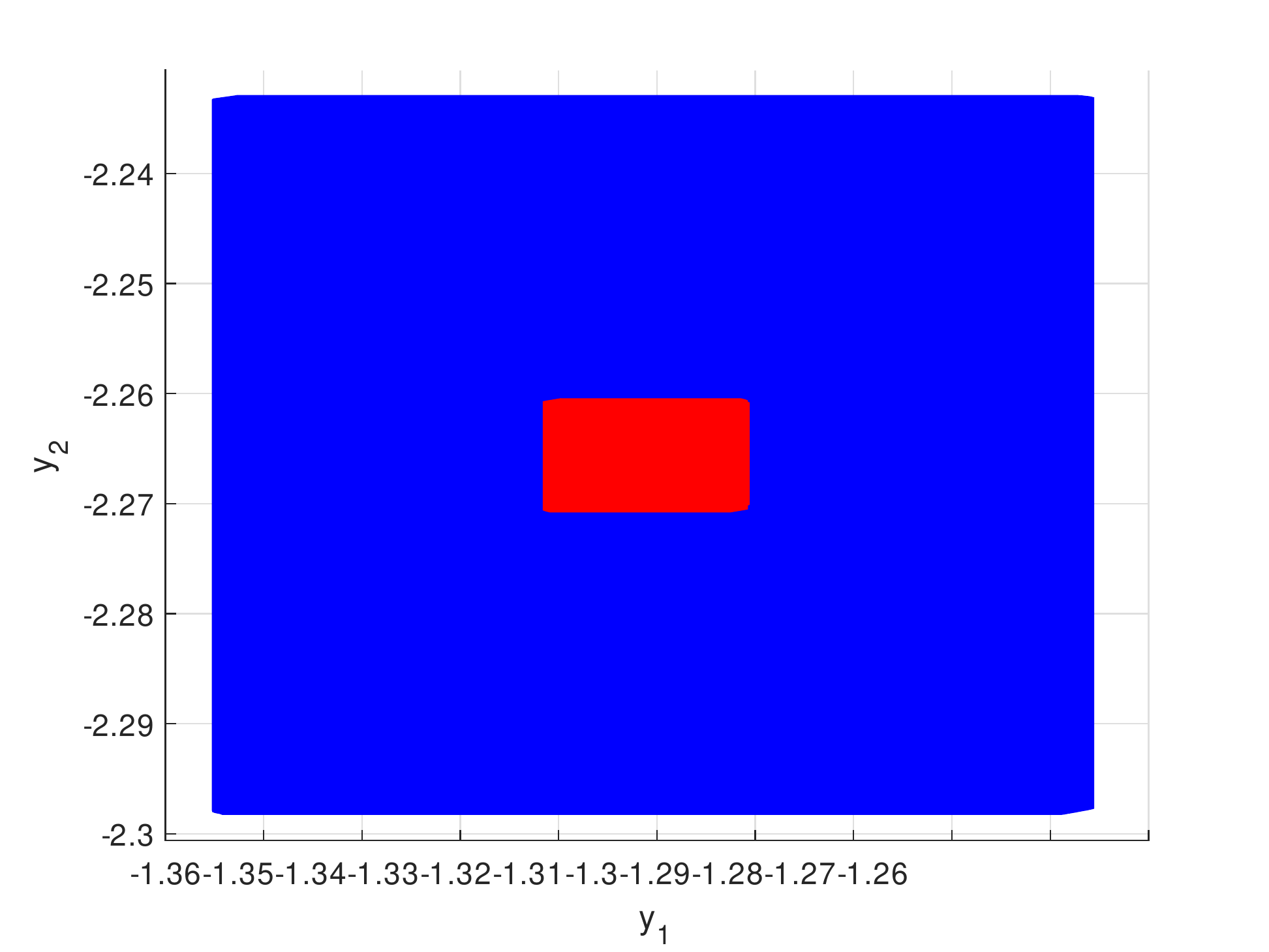}
%\caption{fig2}
\label{S2-tm0.1}
\end{minipage}%
}%
\subfigure[$\epsilon=0.2$, \textcolor{blue}{TM} Vs. \textcolor{red}{TM-O}]{
\begin{minipage}[t]{0.32\linewidth}
\centering
\includegraphics[width=1.6in]{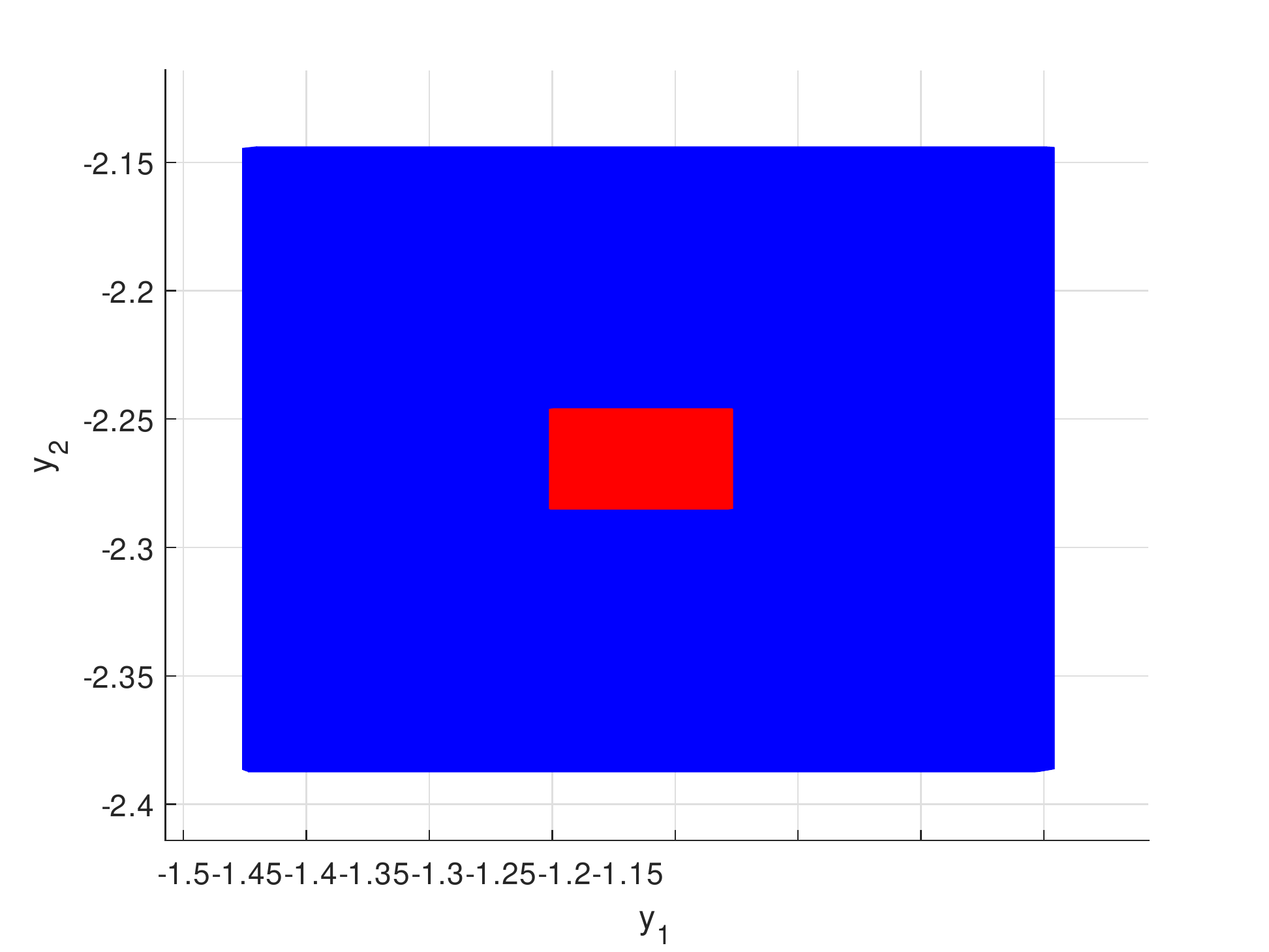}
%\caption{fig2}
\label{S2-tm0.2}
\end{minipage}%
}%
\subfigure[$\epsilon=0.5$, \textcolor{blue}{TM} Vs. \textcolor{red}{TM-O}]{
\begin{minipage}[t]{0.32\linewidth}
\centering
\includegraphics[width=1.6in]{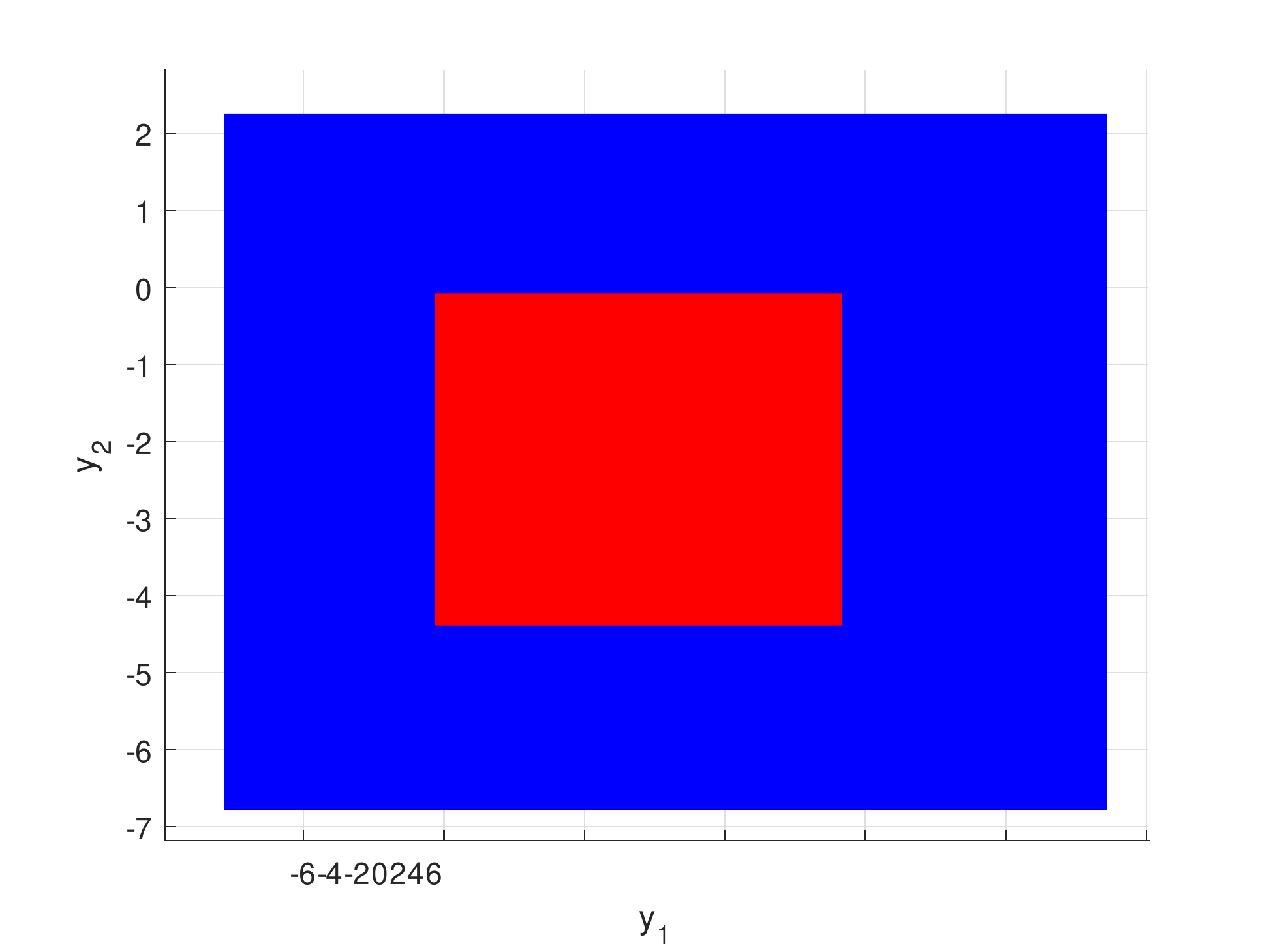}
%\caption{fig2}
\label{S2-tm0.5}
\end{minipage}
}%%
\centering
\caption{Reachable sets on NN $\bm{N_7}$.}
\label{block S2}
\end{figure}

\begin{table}[b]
  \centering
  \caption{Reachable set comparisons on NNs versus DeepZ.}
   \setlength{\tabcolsep}{2mm}{
    \begin{tabular}{c|c|c|c|c|c|c|c|c}
    \Xhline{1.2pt}
    & \multicolumn{2}{c|}{$\bm{N_8}$} & \multicolumn{2}{c|}{$\bm{N_9}$} & \multicolumn{2}{c|}{$\bm{N_{10}}$} & \multicolumn{2}{c}{$\bm{N_{11}}$} \\
    \cline{2-3}  \cline{4-5}  \cline{6-7}   \cline{8-9} 
    \textbf{Perturbation}& $\epsilon =0.003$ & $\epsilon =0.004$  & $\epsilon =0.003$ & $\epsilon =0.004$ & $\epsilon =0.001$ & $\epsilon =0.002$ & $\epsilon =0.001$ & $\epsilon =0.002$ \\
    \hline
    \textbf{Min} & 0.9238 &0.8580 &0.9141&0.7569 &0.9740 &0.9569 & 0.9659&0.9409 \\
    \textbf{Max} &0.9373& 0.8724 &0.9345 &0.7849 &0.9805 &0.9615 &0.9758 & 0.9481  \\
    \textbf{Mean} &0.9308 &0.8651 &0.9247 &0.7672 &0.9777 &0.9590 &0.9709 &0.9444 \\
    \Xhline{1.2pt}
    \end{tabular}}%
  \label{tab:deeepz}%
\end{table}%

The comparisons versus Versig 2.0 on $\bm{N_6}$ and $\bm{N_7}$  are displayed in Table \ref{tab:tm}. Likewise, it can be seen that our proposed set-boundary based reachability analysis can reduce the entire set based over-approximated reachable sets by $20\%-80\%$, compared with Verisig 2.0.  In Table \ref{tab:tm}, we also list the polynomial order and running time time utilized for Taylor models, it can be seen that our proposed set-boundary based  reachability analysis  generally requires lower polynomial order and less time cost, appealingly, obtaining much tighter over-approximation of reachable sets.

\begin{table}[t]
  \centering
  \caption{Reachable set comparisons versus Versig 2.0.}
  \setlength{\tabcolsep}{2mm}{
    \begin{tabular}{c|c|c|c|c|c|c|c|c}
     \Xhline{1.2pt}
    \multicolumn{1}{c|}{\multirow{2}[4]{*}{\textbf{NN}}} & \multirow{2}[4]{*}{\textbf{Perturbation}} & \multicolumn{2}{c|}{\textbf{Entire Set}} & \multicolumn{2}{c|}{\textbf{Boundary}} & \multirow{2}[4]{*}{\textbf{Min}} & \multirow{2}[4]{*}{\textbf{Max}} & \multirow{2}[4]{*}{\textbf{Mean}} \bigstrut\\
\cline{3-6}          &       & \textbf{Order} & \textbf{Time}  & \textbf{Order} & \textbf{Time}  &       &       &  \bigstrut\\
    \hline
    \multirow{3}[6]{*}{$\bm{N_6}$} & $\epsilon = 0.1$      &3     & 0.2585$\pm$0.0011    &2    & 0.3121$\pm$0.0030      & 0.6501     &  0.8755     & 0.7628 \bigstrut\\
\cline{2-9}          & $\epsilon = 0.2$ &3 &0.4957$\pm$0.0100 &2 &0.3143$\pm$0.0036 & 0.6491 &0.8255 &0.7373 \bigstrut\\
\cline{2-9}          & $\epsilon = 0.5$      & 3      & 1.3670$\pm$0.0375      &  1     &0.1352$\pm$0.0023      & 0.1874      &  0.4079     &0.2976  \bigstrut\\
    \hline
    \multirow{3}[6]{*}{$\bm{N_7}$} & $\epsilon = 0.1$ &2 &0.3253$\pm$0.0013 &1 &0.3348$\pm$0.0048 &0.1566 &0.2298 &0.1932 \bigstrut\\
\cline{2-9}          &  $\epsilon = 0.2$     &2      & 0.6661$\pm$0.0019      &1      &0.3430$\pm$0.0119      & 0.1598      & 0.2220      & 0.1909 \bigstrut\\
\cline{2-9}          & $\epsilon = 0.5$ &2 & 3.8716$\pm$0.7725 &2 &8.1132$\pm$0.8472 &0.4608 &0.4758 &0.4683 \bigstrut\\
    \Xhline{1.2pt}
    \end{tabular}}
  \label{tab:tm}%
\end{table}%

%, without considering the homeomorphism subset, the result of DeepZ becomes more compact and the verification conclusion turns from ``\textbf{Unknown}" to be ``\textbf{Safe}".

\section{Conclusion}
\label{Sec:conl}
In this paper we proposed a set-boundary reachability method to verify safety property of neural networks. Different from existing works on developing computational techniques for output reachable sets estimation of NNs,  the set-boundary reachability method analyzed the reachability from the topology point of view. Based on the homeomorphisms and open maps, this analysis took a careful inspection on what happens at the boundaries of input sets, and uncovered that the homeomorphism and open map properties facilitate the reduction of computational burdens on the safety verification of NNs. Several examples and the comparison with existing methods demonstrated the performance of our method. The experimental results showed that our method is indeed able to promote the computation efficiency of existing verification methods on certain NNs, while obtaining more exact reachable sets.

%\textcolor{red}{}

There is a lot of work remaining to be done in order to render the proposed approach more practical. The proposed interval estimation of Jacobian matrices is coarse, which affects the determination of a homeomorphism and, thus the extraction of the small subset for reachability computations. In the future, we will develop more efficient and accurate methods (including the use of parallel computing techniques) for calculating Jacobian matrices. Also, extending these nice topological properties to other NN-related fields seems to be promising. Furthermore, integrating the proposed set-boundary reachability analysis with the various state-of-the-art verification tools and evaluating it on more practical NNs are also  promising future directions.

%Different from homeomorphisms, open maps, mapping open sets to open sets \cite{mendelson1990introduction}, can also ensure that the output reachable set's boundary corresponds to the input's boundary. Moreover, the open mapping condition is weaker than the one for a homeomorphism. Consequently, in future work we would exploit the open mapping property to facilitate reachability computations for safety verification.

%ensuring the exact computation from input boundary to output boundary, the boundary computation of open mapping seems to be redundant at some cases, for a boundary point may be mapped to an interior point with an open mapping. Whereas, the computation based on open mapping still is sound and the redundancy can be ignored compared with the computation on the entire set. Last but not least, the open mapping condition is much easier to hold on in general neural networks.

\begin{acks}
This work is supported by the National Key R\&D Program of China No. 2022YFA1005101, the  National Natural Science Foundation of China under Grant No. 61836005, No. 61872371 and No. 62032024, and the CAS Pioneer Hundred Talents Program.
\end{acks}

\bibliographystyle{ACM-Reference-Format}
\bibliography{reference}

\end{document}